\documentclass[twoside]{article}
\usepackage[preprint]{jmlr2e}
\usepackage{epsfig}
\usepackage{amssymb,amsmath}
\usepackage{natbib}
\usepackage{graphicx}
\usepackage{subcaption}
\usepackage{graphicx} 

\usepackage{bbm}
\usepackage{stmaryrd}
\usepackage[table,svgnames]{xcolor} 
\usepackage{hhline}

\newcommand{\R}{\mathbbm{R}}
\newcommand{\Z}{\mathbbm{Z}}
\newcommand{\C}{\mathbbm{C}}
\newcommand{\N}{\mathbbm{N}}

\newcommand{\cN}{\mathcal{N}}
\newcommand{\cU}{\mathcal{U}}
\newcommand{\cZ}{\mathcal{Z}}

\newcommand{\cL}{\mathcal{L}}
\newcommand{\cO}{\mathcal{O}}

\newcommand{\cV}{\mathcal{V}}
\newcommand{\cT}{\mathfrak{T}}

\newcommand{\sA}{\mathsf{A}}
\newcommand{\sAp}{\mathsf{A}_{\mathsf{P}}}
\newcommand{\sC}{\mathsf{C}}
\newcommand{\sCp}{\mathsf{C}_{\mathsf{P}}}
\newcommand{\sQ}{\mathsf{Q}}
\newcommand{\sK}{\mathsf{K}}
\newcommand{\sT}{\mathsf{T}}
\newcommand{\sF}{\mathsf{F}}
\newcommand{\sE}{\mathsf{E}}
\newcommand{\sS}{\mathsf{S}}
\newcommand{\sG}{\mathsf{G}}
\newcommand{\su}{\mathsf{u}}
\newcommand{\sv}{\mathsf{v}}
\newcommand{\dK}{d_\textrm{\scalebox{.7}{$K$}}}

\newcommand{\sV}{\mathsf{V}}
\newcommand{\sW}{\mathsf{W}}
\newcommand{\cF}{\mathcal{F}}

\newcommand{\bbE}{\mathbb{E}}

\newcommand{\placeholder}{\mathord{\color{black!33}\bullet}}
\newcommand{\range}[2]{\llbracket #1, #2 \rrbracket}

\usepackage{amsmath}  

\usepackage{algpseudocode}
\usepackage[]{algorithm}

\usepackage{amsmath}
\usepackage{amssymb}
\usepackage{amsfonts}
\usepackage{dsfont}
\usepackage{mathtools}

\usepackage{amssymb}

\usepackage{microtype}

\usepackage[table,svgnames]{xcolor}         

\usepackage{soul}

\definecolor{darkred}{rgb}{.7,0,0}
\definecolor{grey}{rgb}{.7,.6,.5}

\definecolor{darkgreen}{rgb}{0.05, 0.5, 0.06}

\definecolor{darkred}{rgb}{0.5, 0.05, 0.06}

\usepackage[normalem]{ulem}
\newcommand\mlremove{\bgroup\markoverwith
{\textcolor{red}{\rule[.5ex]{2pt}{1pt}}}\ULon}
\newcommand\asremove{\bgroup\markoverwith
{\textcolor{cyan}{\rule[.5ex]{2pt}{1pt}}}\ULon}
\newcommand\agreeremove{\bgroup\markoverwith
{\textcolor{blue}{\rule[.5ex]{2pt}{1pt}}}\ULon}

\usepackage{amsmath,amsbsy,amsgen,amscd,amsfonts} 
\usepackage{mathrsfs} 
\usepackage{bbold}
\usepackage{csquotes}

\providecommand{\mathbbm}{\mathbb} 

\usepackage{nicefrac}       

\usepackage{cleveref}

\newcommand{\E}{\mathbb{E}}

\newcommand{\unif}{\mathsf{unif}}

\newcommand\restr[2]{{
  \left.\kern-\nulldelimiterspace 
  #1 
  \littletaller 
  \right|_{#2} 
  }}

\newcommand{\littletaller}{\mathchoice{\vphantom{\big|}}{}{}{}}
\usepackage{lastpage}
\jmlrheading{26}{2025}{1-\pageref{LastPage}}{6/24; Revised
10/25}{12/25}{24-0879}{Edoardo Calvello, Nikola B. Kovachki, Matthew E. Levine, Andrew M. Stuart}

\ShortHeadings{Continuum Attention for Neural Operators}{Calvello, Kovachki, Levine, Stuart}
\firstpageno{1}

\title{Continuum Attention for Neural Operators}
\author{\name Edoardo Calvello \email e.calvello@caltech.edu \\
       \addr 
       California Institute of Technology
       \AND
       \name Nikola B. Kovachki \email nkovachki@nvidia.com \\
       \addr NVIDIA
       \AND
       \name Matthew E. Levine \email matt@basis.ai \\
       \addr The Broad Institute of MIT and Harvard \\
       \addr Basis Research Institute
       \AND
       \name Andrew M. Stuart \email astuart@caltech.edu \\
       \addr
       California Institute of Technology}

\begin{document}

\editor{Lorenzo Rosasco}

\maketitle

\begin{abstract}
    Transformers, and the attention mechanism in particular, have become ubiquitous in machine learning. Their success in modeling nonlocal, long-range correlations has led to their widespread adoption in natural language processing, computer vision, and time series problems. Neural operators, which map spaces of functions into spaces of functions, are necessarily both nonlinear and nonlocal if they are universal;
    it is thus natural to ask whether the attention mechanism can be used in the design of neural operators. Motivated by this, we study transformers in the function space setting. We formulate attention as a map between infinite dimensional function spaces and prove that the attention mechanism as implemented in practice is a Monte Carlo
    or finite difference approximation of this operator. The function space formulation allows for the design of transformer neural operators, a class of architectures designed to learn mappings between function spaces. In this paper, we state and prove the first universal approximation result for transformer neural operators, using only a slight modification of the architecture implemented in practice. The prohibitive cost of applying the attention operator to functions defined on multi-dimensional domains leads to the need for more efficient attention-based architectures. For this reason we also introduce a function space generalization of the patching strategy from computer vision, and introduce a class of associated neural operators. Numerical results,
    on an array of operator learning problems, demonstrate the promise of our approaches to function space formulations of attention and their use in neural operators.
\end{abstract}


\section{Introduction}
\label{sec:intro}

Attention was introduced in the context of neural machine translation in \citet{bahdanau2014neural} to effectively model correlations in sequential data. Many attention mechanisms used in conjunction with recurrent networks were subsequently developed, for example long short-term memory \citep{hochreiter1997long} and gated recurrent units \citep{chung2014empirical}, which were then employed for a wide range of applications, see for example the review \citet{chung2014empirical}. The seminal work \citet{vaswani2017attention} introduced the now ubiquitous transformer architecture, which relies solely on the attention mechanism and not on the recurrence of the network in order to effectively model nonlocal, long-range correlations. In its full generality, the attention function described in \citet{vaswani2017attention} defines an operation between sequences $\{u_i\}_{i=1}^N\subset  \R^{d_u},\{v_i\}_{i=1}^M\subset  \R^{d_v}$ of lengths $N,M\in\N$ respectively, represented as tensors $u\in \R^{N\times d_u}, v\in \R^{M\times d_v}$, where learnable weights $Q\in \R^{d_K\times d_u},K\in \R^{d_K\times d_v},V\in \R^{d_V\times d_v}$ parametrize 
\begin{equation}
\label{eq:attention_vaswani}
{\tt attention}(u,v):= {\tt softmax}(uQ^TKv^T)vV^T.
\end{equation}
In this context, we recall that the function ${\tt softmax}: \R^{N\times M}\to\R^{N\times M}$ is defined component-wise via its action on the input as
\begin{equation}
\label{eq:softmax}
\bigl({\tt softmax}(w)\bigr)_{ij} = \frac{\exp(w_{ij})}{\sum_{j=1}^M\exp(w_{ij})},
\end{equation}
for any $w\in\R^{N\times M}$. Therefore, in this form the attention function defines a mapping ${\tt attention}: \R^{N \times d_u}\times \R^{M \times d_v} \to \R^{N \times d_V}$.\footnote{We note that in most implementations, the input of the softmax function is rescaled by a factor of $1/\sqrt{d_K}$. In our context, continuum attention, a different scaling will be used.} Alternatively, we can view the output of the attention function as a sequence $\{{\tt attention}(u,v)_{i}\}_{i=1}^N\subset \R^{d_V}$.

Viewed more generally, attention defines a mapping from two sequences $\{u_i\}_{i=1}^N$ and $\{v_i\}_{i=1}^M$ into a third sequence $\{{\tt attention}(u,v)_{i}\}_{i=1}^N$; all three sequences take values in $\R^{r}$ for various values of $r$. In this paper, we generalize this in two ways: (a) we replace index $i$ by index $x\in D$, where $D$ is a bounded open subset in $\R^{d}$; (b) we retain a discrete index $i$ but allow the sequences to take values in Banach spaces of $\R^r$-valued functions defined over $D$, where $D$ is a bounded open subset in $\R^{d}$. We find the first generalization (a) most useful in settings where $d=1$,
that is for time series. However, when discretizing the index set $D\subset \R^d$ with $N\coloneqq n^d$ points, the quadratic scaling of the attention mechanism in $N$ can be prohibitive. This motivates generalization (b) where we retain the discrete index but allow sequences to be function valued, corresponding to patching of functions defined over $D\subset \R^d$, with $d \ge 2.$

Our motivation for these generalizations is operator learning and, in particular, design of discretization invariant neural operator architectures. Such architectures disentangle model parameters from any particular discretization of the input functional data; learning is focused on the intrinsic 
continuum problem and not any particular discretization of it. This allows models to obtain consistent errors when trained with different resolution data, as well as enabling training, testing, and inference at different resolutions. We consider the general setting where $D\subset \R^d$ is a bounded open set and $\cU=\cU(D;\R^{d_u})$ and $\cZ = \cZ(D;\R^{d_z})$ are separable Banach spaces of functions.
To train neural operators we consider the following standard data scenario. The set $\{u^{(j)}, z^{(j)} \}_{j=1}^J$ of input-output observation pairs is such that $\bigl(u^{(j)},z^{(j)}\bigr)\sim\nu$ is an i.i.d. sequence distributed according to the measure $\nu$, which is supported on $\cU\times \cZ $. We let $\mu$ and $\pi$ denote the marginals of $\nu$ on $\cU$ and $\cZ$ respectively, so that $u^{(j)}\sim \mu$ and $z^{(j)}\sim \pi$. In what follows, we consider $\sG^\dagger:\cU\to\cZ$ to be an arbitrary map so that 
\begin{equation}
\label{eq:data_model}
z^{(j)} = \sG^\dagger\bigl(u^{(j)}\bigr),
\end{equation}
noting that such maps will typically be both nonlocal (with respect to domain $D$) and nonlinear. We
assume for simplicity that the data generation mechanism is compatible with \eqref{eq:data_model} so that $\pi = \sG^\dagger_\sharp \mu$, i.e., the data generation mechanism is not corrupted by noise.  Neural operators construct an approximation of $\sG^\dagger$ picked from the parametric class
$$\sG:\cU\times \Theta \to \cZ;$$
equivalently we may write $\sG_\theta:\cU\to\cZ,$ for $\theta\in\Theta$. The set $\Theta$ is assumed to be a finite dimensional parameter space from which a $\theta^*\in\Theta$ is selected so that $\sG(\cdot,\theta^*) = \sG_{\theta^*}\approx \sG^\dagger$.  We then define a cost functional $c:\cZ\times\cZ\to \R$ and define $\theta^*$ by
\begin{equation}
\label{eq:data_assimilation_model}
\theta^*={\arg \min}_{\theta\in\Theta}\E_{u\sim\mu}\Bigl[ c\bigl(\sG(u,\theta), \sG^\dagger(u) \bigr)\Bigr].
\end{equation}
In practice, we will only have access to the values of each $u^{(j)}$ and $z^{(j)}$ at a set of discretization points of the domain $D$, which we denote as $\{x_i\}_{i=1}^N$. 
We seek discretization invariant approaches to operator learning
which allow for sharing of learned parameters $\theta^*$ between different discretizations.

The choice of an architecture defining the approximation family $\sG_\theta$ is sometimes
known as surrogate modeling. We use continuum attention concepts to make this choice.
The resulting transformer based neural operator methodology that we develop here can be used to learn solution operators to parametric PDEs, to solve PDE inverse problems, and may also be applied to solve a class of data assimilation problems. In the context of parametric PDEs we will explore how this algorithmic framework may be employed
in operator learning for Darcy flow and 2d Navier-Stokes problems. As a particular case of operator learning, we will also explore a class of data assimilation problems that involve recovering unobserved trajectories of possibly chaotic dynamical systems. Namely, we illustrate examples involving the Lorenz-63 dynamical system and a controlled ODE.


\subsection{Literature Review}

The attention mechanism, as introduced in \citet{vaswani2017attention} and defined in \eqref{eq:attention_vaswani}, has shown widespread success for a variety of tasks.
In particular, given the suitability of attention and transformers for applications involving sequential data, there has been a large body of work extending the methodology from natural language processing, the setting in which it was originally introduced, to the more general context of time series. Namely, as surveyed in \citet{wen2023transformers}, transformers and variants thereof have been applied to time series forecasting, classification and anomaly detection tasks. This generalization has involved modifications and innovations relating to the positional encoding and the design of the attention function itself. 
We refer the reader to \citet{wen2023transformers} for a detailed review of the large body of work and architectural variants of transformers that have been developed for time series applications.
\par
\textbf{Attention in Vision} The quadratic scaling in the input of the attention mechanism has made transformers a computationally prohibitive architecture for applications involving long sequences. Nonetheless, through techniques such as patching, the attention mechanism and transformers have also enjoyed success in computer vision, where high-resolution image data presents an increased computational cost. In fact, in vision transformers (ViT), introduced in \citet{dosovitskyi2021image}, the input image is subdivided into patches; attention is then applied across the shorter sequence of patches, thus decreasing the overall cost. 
\par
\textbf{Neural Operators} Given the success of the methodology in computer vision, there has been growing interest in extending the transformer methodology to the operator learning context. This operator setting concerns learning mappings between infinite dimensional spaces of functions. In this case, the domains on which the functions are defined (spatial and temporal) may inherently be a continuum. Neural operators \citep{kovachki2023neural}, generalizations of neural networks that learn operators mapping between infinite-dimensional spaces, have been developed for this task. As they parametrize operators acting on spaces of functions, the learned parameters of neural operators may be applied to any discretization of the input function; this property of neural operators is referred to as ``discretization invariance''. Discretization invariance is desirable both at a theoretical and practical level. Indeed, for a neural network to approximate an operator acting on functions defined on a continuum, its parameters cannot depend on the discretizations of the input functions that are used for learning the parameters themselves. On the other hand, discretization invariance allows to scalably train models on coarser discretizations of functions and to deploy or finetune trained models at different resolutions for downstream tasks. A recent example of the success of this approach is the GenCast medium-range weather forecasting model \citep{Price2025}, which because of discretization invariance properties is pretrained on low resolutions and finetuned at the desired predictive resolution.
\par
\textbf{Transformers for Operator Learning} In the following paragraphs, we discuss several approaches in the literature employing transformers for operator learning. We first discuss some approaches leading to neural operators, followed by a discussion of architectures that are not discretization invariant.
\par 
\par \uline{\textit{Neural Operators.}} An integral kernel interpretation of the attention mechanism appearing for example in \citet{tsai2019transformer, martins2020sparse,moreno2022kernel,wright2021transformers,cao2021choose,kovachki2023neural} inspired a range of attention-based architectures for operator learning. For example, the work of \citet{cao2021choose} leverages insight from the integral kernel formulation to design novel attention mechanisms that do not employ the softmax function. The author designs attention mechanisms using integral kernels which may be approximated by a Fourier-type projection, leading to quadratic complexity, and by a Galerkin projection, leading to linear complexity. While the resulting attention maps are discretization invariant, the full architecture proposed uses convolutional neural networks to downsample the input to a feature map of sufficiently small resolution for the method to be scalable. Hence the full architecture is not a neural operator. Furthermore, the lack of the softmax nonlinearity leads to a mechanism lacking the probabilistic description outlined in this paper, thus differing from the original definition of attention in \citet{vaswani2017attention}. In \cite{li2024transformer} the authors take the Fourier attention approach as in \cite{cao2021choose} to design a neural operator for large eddy simulation, while in \cite{luo2024hierarchical} the authors take the Galerkin attention approach from \cite{cao2021choose} to design a novel neural operator architecture. In \citet{li2023transformer} the authors present ``OFormer'' a neural operator with an encoder-decoder structure that employs cross-attention to obtain the latent embeddings and then stacked self-attention blocks; the model uses the attention from \citet{cao2021choose} and random Fourier features for query coordinate encodings while using a recurrent multi-layer perceptron for unrolling in time-dependent problems as opposed to masked attention in \citet{vaswani2017attention}. In the recent work \citet{rahman2024pretraining}, the authors propose a neural operator architecture that employs an attention mechanism that acts across channels of the function in the latent embedding, and do so in the context of pre-training a foundation model for PDE solution operators. We demonstrate how such an architecture can be derived from our formulation in Remark~\ref{remark:codomain_attention}. As in our paper, various attention-related concepts are
defined in a function-space setting, but our formulations are more general;
this issue is discussed in detail at the relevant point in the paper. We note that the architectures introduced in this paper rely solely on the attention mechanism defined in the continuum, hence they are neural operators; furthermore, because of the simple design, with a slight modification to our architectures we may prove the first universal approximation theorem for transformer operators.
\par
\uline{\textit{Non discretization invariant architectures.}} In \citet{guibas2022efficient} the authors propose a modification of the Fourier neural operator that involves mixing of tokens in Fourier space. This is achieved by first applying the FFT to the sequence of tokens and then applying a MLP to each token, before inverting back to the original space. This approach differs fundamentally from the original definition of attention in \citet{vaswani2017attention} and lacks the probabilistic description outlined in this paper. Furthermore, because of learnable positional embeddings and patching via convolution, the AFNO is not discretization invariant. Various other approaches have been developed. The work \citet{hao2023gnot} introduces ``GNOT'', an architecture employing a normalized attention layer that handles multiple inputs. This design, along with a learned gating mechanism for input coordinates, allows the model to be applied to irregular meshes and multi-scale problems; however, this architecture employs an attention mechanism that does not involve rescaling by the grid, hence not directly allowing for discretization invariance. In \citet{ovadia2023vito} the authors propose to use an architecture composed of a ViT in the latent space of a U-Net \citep{ronneberger2015unet}. The proposed model takes as input the input field evaluated on a grid with associated coordinate values and outputs an approximation of the solution to PDE inverse problems. While it is shown that the architecture achieves state-of-the-art accuracy, the parameters in the scheme are not invariant to the input function resolution. In \citet{wang2025continuous} the authors design ``Continuous ViT'', an architecture consisting of a ViT encoder and a cross-attention decoder. The method handles input functions defined on space and time and the method is shown to present evidence of discretization invariance, with test discretizations close to training discretizations. However, while the architecture is discretization invariant in time, it is not by design in space. This is because the tokenization procedure employed by standard ViTs is not discretization invariant; we elaborate on this point in Remark \ref{rm:patch_vit}.

\par
 
\textbf{Transformers for Continuum Data} The use of attention-based transformer models in the context of dynamical systems and data assimilation problems, where spatial and temporal domains of the systems are continuous, is at its infancy. For a comprehensive review on the application of machine learning to data assimilation we refer the reader to \citet{cheng2023machine}. In this context, the efficacy of transformers has been explored for forecasting in numerical weather prediction. In \citet{chattopadhyay2022towards}, a transformer module is integrated in the latent space of a U-Net \citep{ronneberger2015unet} architecture in order to leverage the permutation equivariance property of the transformer to improve physical consistency and forecast accuracy. The work in \citet{bi2023accurate} achieves state-of-the-art performance among end-to-end learning methods for numerical weather prediction by applying a ViT-like architecture and employing the shifted window methodology from \citet{liu2021swin}. These approaches, however, are not discretization invariant by design. 
\par

\textbf{Mathematics of Transformers} Mathematical foundations for the methodology are only now starting to emerge, however most analysis has been preformed in the finite-dimensional setting. In the recent work \citet{geshkovski2023emergence}, the authors study the attention dynamics by viewing tokens as particles in an interacting particle system. The formulation of the continuous-time limit of the dynamics, where the limit is taken in ``layer time'' allows to investigate the emergence of clusters among tokens. In \citet{Yun2020Are}, the authors prove that transformers are universal approximators of continuous permutation equivariant sequence-to-sequence functions with compact support. More broadly, developing a firm theoretical framework for the subject is necessary to understand the empirical performance of the methodology and to inform the design of novel attention-based architectures.
\par
\textbf{This Work} This paper bridges the gap between the recent theoretical developments for the attention mechanism and practical application in operator learning. In fact, starting from a description of the attention mechanism as a map on spaces of functions, we build discretization invariant transformer blocks that may be directly deployed for operator learning tasks. A numerical investigation shows that the ensuing transformer neural operators are competitive with state-of-the-art architectures for a range of operator learning tasks. The framework developed may be used more generally to design larger, more complex architectures for which the universal approximation theorem proved in this paper may potentially be generalized.

\subsection{Contributions and Outline}

Motivated by the success of transformers for sequential data, we set out to formulate the attention mechanism and transformer architectures as mappings between infinite dimensional spaces of functions. In this work we introduce a continuum perspective on the transformer methodology. The resulting theoretical framework enables the design of attention-based neural operator architectures, which find a natural application to operator learning problems. The strength of the approach of devising implementable discrete architectures from the continuum perspective stems from the discretization invariance property that the schemes inherit. Indeed, the resulting neural operator architectures are designed as mappings between infinite dimensional functions spaces that are invariant to the particular discretization, or resolution, at which the input functions are defined. Notably, this allows for zero-shot generalization to different function discretizations. With these insights in view, our main contributions can be summarized as follows:
\begin{enumerate}
    \item
    \label{cont:1}
    Formulation of attention as a mapping between function spaces.
    \item 
    \label{cont:2}
    Approximation theorem that quantifies the error between application of the attention operator to a continuous function and the result of applying its finite dimensional analogue to a discretization of the same function. The relevant results may be found in \Cref{thm:self_attention_limit,thm:cross_attention_limit}.
    \item 
    \label{cont:3}
    A formulation of the transformer architecture as a neural operator, hence as a mapping between function spaces; the resulting scheme is resolution invariant and allows for zero-shot generalization to irregular, non-uniform meshes.
    \item 
    \label{cont:4}
    A first universal approximation theorem for transformer neural operators. The relevant results may be found in \Cref{thm:ua1,thm:ua2}.
    \item 
    \label{cont:5}
    Formulation of patch-attention as a mapping between spaces of functions acting on patch indices. This continuum perspective leads to the design of efficient attention-based neural operator architectures.
    \item 
    \label{cont:6} 
    Numerical evidence of the competitiveness of transformer-based neural
    operators with state-of-the-art operator learning architectures. 
\end{enumerate}

After introducing, in Subsection \ref{subsec:Notation}, the notation that will be used throughout,  Section \ref{sec:attention} is focused on contributions \ref{cont:1} and \ref{cont:2}, formulating attention as acting on functions defined on a continuous domain.  In Section \ref{sec:patch_attention} we formulate the attention mechanism as acting on a sequence of ``patches'' of a function, addressing
contribution \ref{cont:5}.
In Section \ref{sec:transformers_operator} we make use of the operator frameworks developed in Sections \ref{sec:attention} and \ref{sec:patch_attention} to formulate transformer neural operator architectures acting on function space. Thus we address  \ref{cont:3}, whilst the last two subsections concern the algorithmic aspect of extending to patching and address contribution \ref{cont:5}.
In Section \ref{sec:universal_approx} we state and prove a first universal approximation theorem for a transformer-based neural operator, hence addressing contribution \ref{cont:4}. 
Finally, in Section \ref{sec:numerics} we explore applications of the transformer neural operator architectures to a variety of operator learning problems, contribution \ref{cont:6}.


\subsection{Notation}

\label{subsec:Notation}

Throughout we denote the positive integers and non-negative integers
respectively by $\N=\{1,2,\cdots\}$ and $\Z^+=\{0,1,2,\cdots\},$
and the notation $\R=(-\infty,\infty)$ and $\R^+=[0,\infty)$ for the
reals and the non-negative reals. We let $\range{1}{n} \subset\N$ denote the set $\{1,\ldots,n\}$. For a set $D$, we denote by $\Bar{D}$ the closure of the set, i.e., the union of the set itself and the set of all its limit points. We let
$\langle \cdot, \cdot \rangle, |\cdot|$ denote the
Euclidean inner-product and norm, noting that
$|v|^2=\langle v,v \rangle.$ We may also use $|\cdot|$ to denote the cardinality of a set, but the distinction will be clear from context. We write $\langle \cdot, \cdot \rangle_{\R^d}, |\cdot|_{\R^d}$ if we want to make explicit the dimension of the space on which the Euclidean inner product is computed. We will also denote by $|\cdot|_1$ the norm on the vector space $\ell^1$.

We denote by $C(D;\R^d)$ the infinite dimensional Banach space of continuous functions mapping the set $D$ to the $d$-dimensional vector space $\R^d$. The space is endowed with the supremum norm. Similarly, we denote by $L^2(D;\R^d)$ the infinite dimensional space of square integrable functions mapping $D$ to $\R^d$. The space is endowed with the $L^2$ inner product, which induces the norm on the space. We will sometimes use the shorthand notation $C(D)$ and $L^2(D)$ to denote continuous functions and square integrable functions defined on the domain $D$, respectively, when the image space is irrelevant. For integers $s\geq0$ we denote by $C^s$ the space of continuously differentiable functions up to order $s$; furthermore, for $p\in [1,\infty)$, we denote by $W^{s,p}$ the Sobolev space of functions possessing weak derivatives up to order $s$ of finite $L^p$ norm. We use the notation $H^s$ to denote $W^{s,2}$. Throughout, general vector spaces will be written using a calligraphic font $\cU$. We denote by $\cL(\cU,\cV)$ the infinite-dimensional space of linear operators mapping the vector space $\cU$ to the vector space $\cV$. Throughout, we use different fonts to highlight the difference between operators acting on finite dimensional spaces, such as $Q$, and operators acting on infinite dimensional spaces, such as $\sQ$. We use $\cF$ to denote the Fourier transform, while $\cF^{-1}$ will denote the inverse Fourier transform.

We use $\bbE$ to denote expectation under
the prevailing probability measure; if we wish to make clear
that measure $\pi$ is the prevailing probability measure then we
write $\bbE_\pi$. We use $\mathcal{N}(m,C)$ to denote a Gaussian with mean $m$ and covariance $C$; we also use the same notation in the infinite dimensional context do denote a Gaussian measure with covariance operator $C$; whether we are in the finite dimensional setting or the infinite dimensional one will be clear from context. We also use the notation $\unif(D)$ to denote a uniform distribution defined over the domain $D$. 


\section{Continuum Attention}
\label{sec:attention}

In this section we formulate a continuum analogue of the attention mechanism described in \citet{vaswani2017attention}. We note that in the transformer architecture of \citet{vaswani2017attention}, a distinction is made between the self-attention and cross-attention functions. The former, a particular instance of \eqref{eq:attention_vaswani} where $u=v$, can be used to produce a representation of the input that captures intra-sequence correlations. On the other hand, cross-attention is used to output a representation of the input $u$ that models cross-correlations between both inputs $u$ and $v$. Indeed, in \citet{vaswani2017attention} cross-attention is employed in the decoder component of the transformer architecture so that the input to the decoder attends to the output of the encoder. For natural language applications in particular, the form of the cross-attention function has allowed for sequences of arbitrary length $N$ to be outputted from a transformer architecture while attending to sequences of different lengths, e.g. as outputs of length $M$ from a self-attention encoder block. 

For pedagogical purposes we subdivide the presentation between the self-attention operator and the cross-attention operator, which are treated in Subsection \ref{subsec:self-attention} and Subsection \ref{subsec:cross-attention} respectively. In Subsection \ref{subsubsec:self-attention_discrete} we define self-attention for discrete sequences indexed on finite bounded sets. Here, we outline the interpretation of self-attention as an expectation under a family of carefully defined discrete probability measures, each acting on the space of sequence indices. We show that under this formulation we recover the standard definition of self-attention from \citet{vaswani2017attention}. We proceed in Subsection \ref{subsubsec:self-attention_cont} by formulating self-attention as an operator mapping between spaces of functions defined on bounded uncountable sets. This interpretation is natural, as functions may be thought of as ``sequences'' indexed over an uncountable domain. As in the discrete case, we formulate the self-attention operator as an expectation under a carefully defined probability measure acting on the domain of the input. We obtain an approximation result given by Theorem \ref{thm:self_attention_limit} that shows that for continuous functions the self-attention mechanism as implemented in practice may be viewed as a Monte Carlo approximation of the self-attention operator described. We mimic these constructions in Subsections \ref{subsubsec:cross-attention_discrete} and \ref{subsubsec:cross-attention_cont} for the discrete and continuous formulations of cross-attention, respectively, and obtain an analogous approximation result given by Theorem \ref{thm:cross_attention_limit}.
 As in \eqref{eq:attention_vaswani}, throughout this section $Q\in \R^{d_K\times d_u},K\in \R^{d_K\times d_v},V\in \R^{d_V\times d_v}$ are learnable weights that parametrize the attention operators. When we consider self-attention we will have $d_v=d_u.$

\subsection{Self-Attention}
\label{subsec:self-attention}

In the following discussion we define the self-attention operator in the discrete setting and in the continuum setting, in Subsections \ref{subsubsec:self-attention_discrete} and \ref{subsubsec:self-attention_cont}, respectively. 

\subsubsection{Sequences over \texorpdfstring{$D^N \subset \Z$}{}}
\label{subsubsec:self-attention_discrete}

We begin by considering the finite, bounded set $D^N= \{1, \cdots, N\}.$ We let $u: D^N \to \R^{d_u}$ be a sequence and $j \in D^N$ an index so that $u(j)\in \R^{d_u}$. 
We will now focus on formulating the attention mechanism in a manner that allows generalization
to more general domains $D^N;$ at the end of this subsection we connect our formulation with
the more standard one used in the literature, which is specific to the current choice of
domain $D^N= \{1, \cdots, N\}.$ We define the following discrete probability measure.

\begin{definition}
\label{d:1}
We define a probability measure on $D^N$, parameterized by $(u,j)$, via the probability mass function $p(\cdot;u,j): D^N \to [0,1]$ defined by
\[p(k; u, j) = \frac{\mathrm{exp} \Big ( \big\langle Q u(j), K u(k) \big\rangle_{\R^{\dK}} \Big )}{\sum_{\ell \in D^N} \mathrm{exp} \Big ( \big\langle Q u(j), K u(\ell) \big\rangle_{\R^{\dK}} \Big )},\]
for any $k \in D^N$.
\end{definition}

We may now define the self-attention operator as an expectation over this measure $p$ of the linear transformation $V\in \R^{d_V\times d_u}$ applied to $u$.

\begin{definition}
\label{d:2}
We define the self-attention operator $\sA^N$ as a mapping from a $\R^{d_u}$-valued sequence over $D^N$, $u$, into a $\R^{d_V}$-valued sequence over $D^N$, $\sA^N(u)$, that takes the form
\[\sA^N(u)(j):=  \E_{k \sim p(k;u,j)} [Vu(k)],\]
for any $j\in D^N$.
\end{definition}

To connect these definitions with the standard formulation of transformers, 
note that sequence $u$ can be reformulated as the matrix $u \in \R^{N \times d_u}.$ Then
\[
\{uQ^TKu^T\}_{jk}=\big\langle Q u_j, K u_k \big\rangle_{\R^{\dK}};
\]
that is, the $\{j,k\}-$entry of $uQ^TKu^T \in \R^{N \times N}$ is the right-hand side of the preceding identity, where $u_\ell$ denotes the $\ell$'th row of $u\in\R^{N\times d_u}$.
It is then apparent that applying the {\tt softmax} function along rows of $uQ^TKu^T$ delivers the vector
of probabilities defined in Definition \ref{d:1}, indexed by $k;$ note that $j$ is a parameter in this vector
of probabilities, indicating the row of the matrix $uQ^TKu^T$ along which the {\tt softmax} operation is applied.
Finally the
attention operation from Definition \ref{d:2} can be rewritten to act between matrices as
\[ {\sA^N(u)} = {\tt attention}(u,u):= {\tt softmax}(uQ^TKu^T)uV^T,\]
so that ${\tt attention}(u,u)\in\R^{N \times d_V}$.

\begin{remark}[Sequences over $D^N \subset \Z^d$]
\label{remark:self-attention_discrete_multidim}

Once the attention mechanism is formulated as in Definitions \ref{d:1}, \ref{d:2}, it is straightforward 
to extend to (generalized) sequences indexed over bounded subsets of $\Z^{d};$ note, for example, that pixellated
images are naturally indexed over bounded subsets of $\Z^2.$
Let $D^N \subset \Z^{d}$ possess cardinality $|D^N|=N.$ Let $u: D^N \to \R^{d_u}$ be a sequence and $j \in D^N$ an index.
Then the definition of probability on $D^N$, indexed by $(u,j)$, and the resulting definition of
self-attention operator $\sA$, is exactly as in Definitions \ref{d:1}, \ref{d:2}. This shows the power
of non-standard notation we have employed here.
\end{remark}

\subsubsection{Sequences over \texorpdfstring{$D \subset \R^d$}{}}
\label{subsubsec:self-attention_cont}

We note that the preceding two formulations of the attention mechanism view it as a transformation defined
as an expectation applied to the input sequence. The expectation is with respect to a probability
measure supported on an index of the input sequence. Furthermore, the expectation itself depends on the
input sequence to which it is applied, and is parameterized by the input sequence index, resulting in a new 
output sequence of the same length as the input sequence. The formulation using an expectation results from
the softmax part of the definition of the attention mechanism; the query and key linear operations
define the expectation, and the value linear operation lifts the output sequence to take values in a 
(possibly) different Euclidean space than the input sequence. 
These components may be readily extended to work with (generalized) sequences 
defined over open subsets of $\R^d$. (These would often be referred to as functions, but we call them
generalized sequences to emphasize similarity with the discrete case.)

Let $D \subseteq \R^d$ be an open set. Let $u: D \to \R^{d_u}$ be a sequence (function) and $x \in D$ an index.

\begin{definition}
\label{d:3}
We define a probability measure on $D$, parameterized by $(u,x)$, via the probability density function
$p(\cdot;u,x): D \to \R^+$ defined by
\[p(y; u, x) = \frac{\mathrm{exp} \Big ( \big\langle Q u(x), K u(y) \big\rangle_{\R^{\dK}} \Big )}{\int_D \mathrm{exp} \Big ( \big\langle Q u(x), K u(s) \big\rangle_{\R^{\dK}} \Big ) \: \mathrm{d}s},\]
for any $y \in D$.
\end{definition}

\begin{definition}
\label{d:4}
We define the self-attention operator $\sA$ as a mapping from a $\R^{d_u}$-valued sequence (function) over $D$, $u$, into a $\R^{d_V}$-valued sequence (function) over $D$, $\sA(u)$, as follows:
\[\sA(u)(x) = \E_{y \sim p(y;u,x)} [V u(y)],\]
for any $x\in D$.
\end{definition}

A connection between the discrete and continuum formulations of self-attention is captured in the following theorem.

\begin{theorem}
\label{thm:self_attention_limit}
The self-attention operator $\sA$ may be viewed as a mapping \(\sA : L^{\infty} (D;\R^{d_u}) \to L^{\infty} (D;\R^{d_V})\)
and thus as a mapping
\(\sA : C (\Bar{D};\R^{d_u}) \to C (\Bar{D};\R^{d_V})\). Furthermore, for any compact set \(B \subset C(\Bar{D};\R^{d_u})\),
\[ \lim_{N \to \infty} \sup_{u \in B} \E \left \| \sA(u) - \frac{\sum_{j=1}^N \mathrm{exp} \big ( \langle Qu(\cdot), Ku(y_j)  \rangle \big ) V u(y_j)}{\sum_{\ell=1}^N \mathrm{exp} \big ( \langle Qu(\cdot), Ku(y_\ell)  \rangle \big )} \right \|_{C(\Bar{D};\R^{d_V})}  = 0,\]
with the expectation taken over i.i.d. sequences \(\{y_j\}_{j=1}^N \sim \unif(\Bar{D})\).
\end{theorem}

\begin{proof}
    The proof of this result is developed in Appendix \ref{Appendix:self}.
\end{proof}

\begin{remark}
This result connects our continuum formulation of self-attention with the standard definition, using
Monte Carlo approximation of the relevant integrals; a similar result could be proved using finite
difference approximation on, for example, a uniform grid. We note that it is not possible to obtain a uniform convergence rate for all $u\in B$ in the norm of $C(\Bar{D};\mathbb{R}^{d_V})$ as the Riemann integral may exhibit arbitrarily slow convergence depending on the regularity of the integrand.
\end{remark}

\subsection{Cross-Attention}
\label{subsec:cross-attention}

In the following discussion we define the cross-attention operator in the discrete setting and in the continuum setting, in Subsections \ref{subsubsec:cross-attention_discrete} and \ref{subsubsec:cross-attention_cont}, respectively. 


\subsubsection{Sequences over \texorpdfstring{$D^N \subset \Z^d,E^M\subset\Z^e$}{} }
\label{subsubsec:cross-attention_discrete}

We begin by letting $D^N \subseteq \Z^d$ and $E^M \subseteq \Z^e$ be finite sets of points with cardinalities $N$ and $M$
respectively. Let $u: D^N \to \R^{d_u}$ be a sequence (function),
$v: E^M \to \R^{d_v}$ another sequence (function), and $j \in D^N$, $k \in \E^M$ indices so that $u(j)\in\R^{d_u}$ and $v(k)\in \R^{d_v}$. We define the following discrete probability measure.
\begin{definition}
\label{d:5}
We define a probability measure on $D^N$, parameterized by $(u,v,j)$, via the probability mass function
$q(\cdot;u,v,j): D^N \to [0,1]$ defined by
\[q(k; u,v, j) = \frac{\mathrm{exp} \Big ( \big\langle Q u(j), K v(k) \big\rangle_{\R^{\dK}} \Big )}{\sum_{\ell \in E^M} \mathrm{exp} \Big ( \big\langle Q u(j), K v(\ell)\big\rangle_{\R^{\dK}} \Big )},\]
for any $k \in E^M$.
\end{definition}
We may now define the cross-attention operator as an expectation over this measure $q$ of the linear transformation $V\in \R^{d_V\times d_v}$ applied to $v$.
\begin{definition}
\label{d:6}
We define the cross-attention operator $\sC^N$ as a mapping from a $\R^{d_u}$-valued sequence (function) over $D^N$, $u$, 
and a $\R^{d_v}$-valued sequence (function) over $E^M$, $v$, into a $\R^{d_V}$-valued sequence (function) over $D^N$, $\sC^N(u,v)$, as follows:
\[\sC^N(u,v)(j) = \E_{k \sim q(k;u,v,j)} [V v(k)],\]
for any $j\in D^N$.
\end{definition}

\subsubsection{Sequences over \texorpdfstring{$D \subset \R^d, E\subset \R^e$}{}}
\label{subsubsec:cross-attention_cont}

The framework described in Subsection \ref{subsubsec:cross-attention_discrete} may be readily extended to work with generalized sequences (functions) defined over open subsets of $\R^d$. Indeed, let $D \subseteq \R^d$ and $E \subseteq \R^e$ be open sets. Let $u: D \to \R^{d_u}$ be a sequence (function),
$v: E \to \R^{d_v}$ another sequence (function), and $x \in D$, $y \in E$ indices. We define the following probability measure over $D$ before providing the definition of the cross-attention operator.

\begin{definition}
\label{d:7}
We define a probability measure on $D$, parameterized by $(u,v,x)$, via the probability density function
$q(\cdot;u,v,x): D \to \R^+$ defined by
\[q(y; u,v, x) = \frac{\mathrm{exp} \Big ( \big\langle Q u(x), K v(y) \big\rangle_{\R^{\dK}} \Big )}{\int_{E} \mathrm{exp} \Big ( \big\langle Q u(x), K v(s) \big\rangle_{\R^{\dK}} \Big ) \: \mathrm{d}s},\]
for any $y \in {E}$.
\end{definition}

\begin{definition}
\label{d:8}
We define the cross-attention operator $\sC$ as a mapping from a $\R^{d_u}$-valued sequence (function) over $D$, $u$, 
and a $\R^{d_v}$-valued sequence (function) over $E$, $v$, into a $\R^{d_V}$-valued sequence (function) over $D$, $\sC(u,v)$, as follows:
\[\sC(u,v)(x) = \E_{y \sim q(y;u,v,x)} [V v(y)],\]
for any $x\in D$.
\end{definition}

A connection between the discrete and continuum formulations of cross-attention is captured in the following theorem.

\begin{theorem}
\label{thm:cross_attention_limit}
    The cross-attention operator $\sC$ may be viewed as a mapping $\sC:L^{\infty} (D;\R^{d_u})\times L^{\infty}(E;\R^{d_v}) \to L^{\infty} (D;\R^{d_V})$, and so as a mapping $\sC: C(\Bar{D};\R^{d_u})\times C(\Bar{E};\R^{d_v})\to C(\Bar{D};\R^{d_V})$. Furthermore, for any compact set $B\subset C(\Bar{D};\R^{d_u})\times C(\Bar{E};\R^{d_v})$

    \[ \lim_{N \to \infty} \sup_{(u,v) \in B} \E \left \| \sC(u,v) - \frac{\sum_{j=1}^N \mathrm{exp} \big ( \langle Qu(\cdot), Kv(y_j)  \rangle \big ) V v(y_j)}{\sum_{\ell=1}^N \mathrm{exp} \big ( \langle Qu(\cdot), Kv(y_\ell)  \rangle \big )} \right \|_{C(\Bar{D};\R^{d_V})}  = 0,\]
    with the expectation taken over i.i.d. sequences \(\{y_j\}_{j=1}^N \sim \unif(\Bar{E})\).

\end{theorem}
\begin{proof}
    The proof of this result is developed in Appendix \ref{Appendix:cross}, which contains straightforward generalizations of the techniques in Appendix \ref{Appendix:self}.
\end{proof}

\section{Continuum Patched Attention}
\label{sec:patch_attention}

Patches of a function are defined as the restriction of the function itself to elements of a partition of the domain. We may then define patch-attention as a mapping between spaces of functions defined on patch indices.  
Interest in applying the attention mechanism for computer vision led to the development of methodologies to overcome its prohibitive computational cost. Patching, as employed in vision transformers (ViT), involves subdividing the data space domain into $P\in \mathbb{N}$ ``patches'' and applying attention to a sequence of flattened patches. This step drastically reduces the complexity of the attention mechanism, which is quadratic in the input sequence length. An analogue of the patching strategy in vision transformers can be considered in the function space setting. Indeed, in image space, attention can be applied across subsets of pixels (patches); in function space, attention can be applied across the function defined on elements of a partition of the domain. Such a generalization to functions defined on the continuum allows for the development of architectures that are mesh-invariant. We describe the patched-attention methodology in the continuum framework developed thus far. 

Throughout this section, we let $D\subset\R^d$ be a bounded open set and let $D\coloneqq D_1 \cup \cdots \cup D_P$ be a uniform partition of the space $D$ so that $D_j\cong D'$ {are congruent} for all $j\in\range{1}{P}$ for some $D'\subset D$, where $P\in\mathbb{N}$ represents the number of patches. We note that we require a uniform partition, namely a partition of subsets of equivalent finite volume, in order to define operators that can be applied to functions defined on each of these subsets; however, for practical application we may have a different number of discretization points within each subset. 
We consider a function defined on the full domain $D$, denoted as $u\in \cU\bigl( D; \R^{d_u} \bigr)$. We define a mapping 
\begin{equation}
\widetilde{u}:\range{1}{P} \to \cU(D' ,\mathbb{R}^{d_u}),
\end{equation}
that defines the patched version of the function $u\in \cU\bigl( D; \R^{d_u} \bigr)$, so that
\begin{equation}
\widetilde{u}(p)= \restr{u}{D_p},
\end{equation}
for any $p\in\range{1}{P}$. We proceed to establish the definitions of the patched self-attention operator $\sAp$ in Subsection \ref{subsec:patch-self-attention} and for completeness the patched cross-attention operator $\sCp$ in Subsection \ref{subsec:patch-cross-attention}.

\subsection{Patched Self-Attention}
\label{subsec:patch-self-attention}

We begin with the definition of the patched self-attention operator $\sAp$. Indeed, we let the operators $\sQ,\sK,\sV$ be defined such that $\sQ:\cU(D' ,\mathbb{R}^{d_u})\to L^2(D' ,\mathbb{R}^{d_K}), \sK:\cU(D' ,\mathbb{R}^{d_u})\to L^2(D' ,\mathbb{R}^{d_K}), \sV: \cU(D' ,\mathbb{R}^{d_u})\to \cU(D' ,\mathbb{R}^{d_V}) $ and let $j \in \range{1}{P}$ be an index.

\begin{definition}
\label{def:patch_attention_p}
Define a probability measure on $\range{1}{P}$, parameterized by $(\widetilde{u},j)$, via the probability density function $p(\cdot;\widetilde{u},j): \range{1}{P} \to [0,1]$ defined by
\[p(k; \widetilde{u}, j) = \frac{\mathrm{exp} \Big ( \big\langle \sQ \widetilde{u}(j), \sK \widetilde{u}(k) \big\rangle_{L^2(D' ,\mathbb{R}^{d_K})} \Big )}{\sum_{\ell=1}^P \mathrm{exp} \Big ( \big\langle \sQ \widetilde{u}(j), \sK \widetilde{u}(\ell) \big\rangle_{L^2(D' ,\mathbb{R}^{d_K})} \Big )},\]
for any $k \in \range{1}{P}$.\footnote{We note that as $\sQ$ is linear, $\sQ u(j)$ is shorthand notation for $\sQ\bigl(u(j) \bigr)$, and similarly for $\sK$ and $\sV$.}
\end{definition}

\begin{definition}
\label{def:patch_attention}
The self-attention operator $\sAp$ maps the function taking a patch index to a $\R^{d_u}$-valued function over $D'$, $\widetilde{u}$, into a function taking a patch index to a $\R^{d_V}$-valued function over $D'$, $\sAp(\widetilde{u})$, as follows:
\[\sAp(\widetilde{u})(j) = \E_{k \sim p(k;\widetilde{u},j)} [\sV \widetilde{u}(k)],\]
for any $j\in\range{1}{P}$.
\end{definition}

\begin{remark}[Patched Attention in ViT]
\label{rm:patch_vit}
    An image may be viewed as a mapping $u: D \to \R^{d_u}$, where $D\subset \Z^2$, and where $d_u=1$ for grey-scale and $d_u=3$ for RGB-valued images.  In the context of vision transformers \citep{dosovitskyi2021image}, the input to the attention mechanism is a sequence of $P$ ``flattened'' patches of the input image; we next outline the details to this procedure. 
    
    Letting $N$ be the total number of discretization points in $D$ and $P$ the number of patches, in the discrete setting each function patch $\widetilde{u}(p):D'\to \R^{d_u}$ may be represented as an $\R^{N/P \times d_u}$ matrix. The flattening procedure of each patch in ViT involves reshaping this matrix into an $\R^{N/P \cdot d_u}$ vector. The full image is thus represented as a sequence of $P$ vectors of dimension $N/P \cdot d_u$. Each patch vector is then linearly lifted to an embedding space of dimension $d_{\textrm{model}}$. The embedded image, to which attention is applied, may hence be viewed as a matrix $\widetilde{u}\in\R^{P\times d_{\textrm{model}}}$. Therefore, ViT may be cast in the setting of the discrete interpretation of the self-attention operator outlined in Subsection \ref{subsubsec:self-attention_discrete}, where the domain on which attention is applied is $D^P$. A key observation to the ensuing discussion is that the application of a linear transformation to a ``flattened'' patch in ViT breaks the mesh-invariance of the architecture; however, this operation may be viewed as a nonlocal transformation on each patch. We leverage this insight in \Cref{sec:transformers_operator} to design discretization invariant transformer neural operators that employ patching. We further note that in practice, patching in ViT is often also achieved via strided convolutions, with stride equivalent to the kernel size; however, because the kernel size determines the number of parameters, such a procedure also breaks discretization invariance as it leads to architectures with parameters dependent on resolution of the input.

\end{remark}

\begin{remark}[Codomain Attention]
\label{remark:codomain_attention}
The work of \citet{rahman2024pretraining} employs an attention mechanism that acts across the channels of the function $u\in \cU\bigl(D;\R^{d_{\textrm{model}}} \bigr)$ in the latent space of a neural operator architecture. Indeed the ``codomain'' attention mechanism defined may be viewed as a particular variant of Definition \ref{def:patch_attention}, where $D'\cong D$, where $\sAp(\widetilde{u})(j)$ is defined for $j \in \range{1}{d_{\textrm{model}}}$ and where $\sQ, \sK, \sV$ are implemented as integral operators. In
particular the indexing is by channels and not by patch. For problems in computer vision, similar ideas have been explored with discrete versions of the $\sQ, \sK, \sV$ operators \citep{chen2017sca}.
\end{remark}

\subsection{Patched Cross-Attention}
\label{subsec:patch-cross-attention}

For completeness, we outline the definition of the patched cross-attention operator. We consider $E\subset \R^e$ to be a bounded open set and $E\coloneqq E_1\cup\ldots \cup E_O$ a uniform partition of $E$ so that $E_j\cong E' \cong D'$ for all $j\in\range{1}{O}$ for some $E'\subset E$ and $O\in \N$. We note that for the following construction we require $E'\cong D'$, but it may hold in general that $P\neq O$. As done for $\widetilde{u}$ we define an operator 
\begin{equation}
\widetilde{v}:\range{1}{O} \to \cU(E' ,\mathbb{R}^{d_v}),
\end{equation}
that defines the patched version of the function $v\in \cU\bigl( E; \R^{d_v} \bigr)$, so that
\begin{equation}
\widetilde{v}(o) = \restr{v}{E_o},
\end{equation}
for any $o\in\range{1}{O}$. Furthermore, we let the operators $\sQ,\sK,\sV$ be defined such that $\sQ:\cU(D' ,\mathbb{R}^{d_u})\to L^2(D' ,\mathbb{R}^{d_K}), \sK:\cU(E' ,\mathbb{R}^{d_v})\to L^2(E' ,\mathbb{R}^{d_K}), \sV: \cU(E' ,\mathbb{R}^{d_v})\to \cU(E' ,\mathbb{R}^{d_V}) $ and let $j \in \range{1}{P}$ be an index.

\begin{definition}
\label{d:9}
We define a probability measure on $D$, parameterized by $(\tilde{u},\tilde{v},j)$, via the probability density function $q(\cdot;\widetilde{u},\widetilde{v},j): \range{1}{P} \to [0,1]$ defined by
\[q(k; \widetilde{u},\widetilde{v}, j) = \frac{\mathrm{exp} \Big ( \big\langle \sQ \widetilde{u}(j), \sK \widetilde{v}(k) \big\rangle_{L^2(E' ,\mathbb{R}^{d_K})} \Big )}{\sum_{\ell=1}^O \mathrm{exp} \Big ( \big\langle \sQ \widetilde{u}(j), \sK \widetilde{v}(\ell) \big\rangle_{L^2(E' ,\mathbb{R}^{d_K})} \Big )},\]
for any $k \in \range{1}{{O}}$.
\end{definition}

\begin{definition}
\label{d:10}
The cross-attention operator $\sC$ maps the operator taking a patch index to a $\R^{d_u}$-valued function over $D'$, $\widetilde{u}$, 
and the operator taking a patch index to a $\R^{d_v}$-valued function over $E'$, $\widetilde{v}$, into an operator taking a patch index to a $\R^{d_V}$-valued function over $D'$, $\sCp(\widetilde{u},\widetilde{v})$, as follows:
\[\sCp(\widetilde{u},\widetilde{v})(j) = \E_{k \sim q(k;\widetilde{u},\widetilde{v},j)} [\sV \widetilde{v}(k)],\]
for any $j\in\range{1}{P}$.
\end{definition}

\section{Transformer Neural Operators}
\label{sec:transformers_operator}

In this section we turn our focus to describing how to devise transformer-based neural operators using the continuum attention operator $\sA$ as defined in the function space setting in Section \ref{sec:attention}, and the continuum patched-attention operator $\sAp$ acting on function space from Section \ref{sec:patch_attention}. 
We make use of the operator frameworks developed in Sections \ref{sec:attention} and \ref{sec:patch_attention} to formulate transformer neural operator architectures acting on function space. We start in Subsection \ref{subsec:setup}, setting-up the framework. In Subsection \ref{subsec:vanilla_transformer} we formulate an analogue of the transformer from \citet{vaswani2017attention} as an operator mapping an input function space $\cU\bigl(D;\R^{d_u}\bigr)$ to the solution function space $\cZ\bigl(D;\R^{d_z}\bigr)$. The resulting scheme, which we define as transformer neural operator (TNO) is invariant to the discretization of the input function and allows for zero-shot generalization to irregular, non-uniform meshes. However, this architecture exhibits quadratic complexity with respect to the number of discretization points of the input functions. Hence, motivated by the need to develop new efficient architectures to employ attention on longer sequences, or in our case higher resolution discretizations of functions, in Subsection \ref{subsec:patch_transformer} we describe a generalization of the ViT methodology to the function space setting, inspired by the work in \citet{dosovitskyi2021image}. In particular we generalize the procedure of lifting flattened patches to the embedding dimension, as done in ViT, to the function space setting. This is achieved by lifting patches of functions to an embedding dimension using a nonlocal linear operator, namely, an integral operator. Unlike ViT, the resulting neural operator architecture, which we call ViT neural operator (ViTNO), is discretization invariant and allows for zero-shot generalization to different resolutions. In Subsection \ref{subsec:fourier_patch_transformer} we describe a different patching-based transformer neural operator acting on function space, the Fourier attention neural operator (FANO). The architecture is discretization invariant and is based on $\sQ,\sK, \sV$, learnable parameters in the attention mechanism, being defined as integral operators acting themselves on function space. We demonstrate in \Cref{sec:numerics} that the approach involving parametrizing $\sQ,\sK, \sV$ as integral operators yields an architecture with greater parameter count for a similar computational complexity. We show that this architecture produces improved results in the context of a problem with smooth inputs.

\subsection{Set-Up}
\label{subsec:setup}
In the subsections that follow we describe transformer neural operators of the form $\sG:\cU(D;\R^{d_u})\times \Theta \to \cZ(D;\R^{d_z})$, which serve as approximations to an operator $\sG^\dagger:\cU(D;\R^{d_u}) \to \cZ(D;\R^{d_z})$. The neural operators we devise have the general form
\begin{equation}
\label{eq:transformer_operator_compact}
\sG(u,\theta) \coloneqq \Bigl(\sT_{\textrm{out}}\circ \sE_L \circ \sT_{\textrm{in}}\Bigr) (u,\theta),
\end{equation}
for any $u\in \cU(D;\R^{d_u})$ and $\theta \in \Theta$. In \eqref{eq:transformer_operator_compact}, the operators $\sT_{\textrm{in}}:\cU(D;\R^{d_u})\times \Theta \to \cV$ and $\sT_{\textrm{out}}:\cV\times \Theta \to \cZ(D;\R^{d_z})$ are defined for some appropriate embedding function space $\cV$.  We will make explicit the definitions of the general operators $\sT_{\textrm{in}}$, $\sT_{\textrm{out}}$ and the embedding function space $\cV$ in the context of each architecture. 
Throughout this section, we drop explicit dependencies on $\theta$ for brevity unless we wish to stress a parametric dependence; for example, we may abuse notation and write $\sT_\textrm{in}(u) := \sT_\textrm{in}(u; \ \theta)$.
The operator $\sE_L:\cV\times \Theta\to\cV$ defines the neural operator analogue of the transformer encoder block from \citet{vaswani2017attention}. Its action on the input $v^{(0)}\coloneqq \sT_{\textrm{in}}(u)\in\cV$ is summarized by the iteration
\begin{subequations}
\label{eq:recurrence_enc_def}
\begin{align}
    v^{(l-1)} &\mapsfrom {\sW_1}v^{(l-1)} + \sA_{\textrm{MultiHead}}\big(v^{(l-1)}\big), \\
    v^{(l-1)} &\mapsfrom \sF_{\textrm{LayerNorm}}\big(v^{(l-1)}\big), \\
    v^{(l-1)} &\mapsfrom {\sW_2}v^{(l-1)} + \sF_{\textrm{NN}}\big( v^{(l-1)}\big),\\
    v^{(l)} &\mapsfrom \sF_{\textrm{LayerNorm}}\big(v^{(l-1)}\big),
\end{align}
\end{subequations}
for $l=1, \dots, L$ layers, so that $\sE_L\bigl( v^{(0)} \bigr) = v^{(L)}\in  \cV$ \footnote{We note that we use $v$ to denote an arbitrary function in the space $\cV$. This is not to be confused with $v$ used for cross-attention in \Cref{sec:attention}. We also denote the output of the $\ell$'th encoder layer by $v^{(\ell)}$, not to be confused with the notation used for the data pairs.}. We note that for notational convenience, we have suppressed dependence of the operators on the layer $\ell$; indeed, each of the operators $\sA_{\textrm{MultiHead}}$, $\sW_1,\sW_2,$ $\sF_{\textrm{NN}}$ and $\sF_{\textrm{LayerNorm}}$, will have different learnable parametrizations for each encoder layer $\ell$. Before we delve into the specific neural operator architectures, we summarize the definitions of the elements of the transformer encoder described in \eqref{eq:recurrence_enc_def}, which are common to all the subsequent architectures. The multi-head self-attention operator $\sA_{\textrm{MultiHead}}: \cV \to \cV $ is defined as the composition of a linear transformation $W_{\textrm{MultiHead}}$ applied pointwise to its input and a concatenation of the outputs of $H\in\N$ self-attention operators so that
\begin{equation}
\label{eq:multihead1}
    \Bigl(\sA_{\textrm{MultiHead}}(v)\Bigr)(x) = W_{\textrm{MultiHead}} \Bigl(\sA_{\placeholder}(v;\theta_1)(x),\ldots, \sA_{\placeholder}(v;\theta_H)(x)  \Bigr),
\end{equation}
for any $v\in \cV$. 
Here the notation $\sA_{\placeholder}$ indicates the ambiguity in the definition of self-attention. We note that we will make explicit which definition of self-attention operator from the previous sections we employ for each architecture; we will also make clear the definitions of the parameters $\theta_h\in \Theta$ defining each attention operator. The operators $\sW_1,\sW_2:\cV\to\cV$ are pointwise linear operators.\footnote{We note that for all experiments in Section \ref{sec:numerics} we set $\sW_1,\sW_2$ to be the identity operators.} Next, the operator $\sF_{\textrm{LayerNorm}}:\cV \to \cV$ defines the layer normalization and the operator $\sF_\textrm{NN}:\cV \to \cV$ defines a feed-forward neural network layer. Both operators are applied pointwise to the input; further details are provided in the context of each architecture.

A key aspect of the ensuing discussion will be the distinction between local and nonlocal linear operators. We note that for an arbitrary function space $\cU\bigl(D; \R^r\bigr)$ where $D \subset \R^d$, the space of linear operators $\cL\Bigl(\cU\bigl(D; \R^r\bigr) ; \cU\bigl(D; \R^{r'}\bigr) \Bigr)$ includes both local and nonlocal linear operators. In the following, we will consider pointwise linear operators in $\cL\Bigl(\cU\bigl(D; \R^r\bigr) ; \cU\bigl(D; \R^{r'}\bigr) \Bigr)$ that admit a finite-dimensional representation $W\in \R^{r'\times r}$. On the other hand, we will also consider linear integral operators in $\cL\Bigl(\cU\bigl(D;\R^r) ; \cU\bigl(D; \R^{r'}\bigr) \Bigr)$, a class of nonlocal operators. Namely, we consider operators $\sW$ that are integral operators learned from data of the form $\sW:\Theta_{\sW}\to \cL\Bigl(\cU\bigl(D;\R^r) ; \cU\bigl(D; \R^{r'}\bigr) \Bigr)$.

\begin{definition}[Integral Operator $\sW$]
\label{def:int_op}
The integral operator $\sW:\Theta_{\sW}\to \cL\Bigl(\cU\bigl(D;\R^r) ; \cU\bigl(D; \R^{r'}\bigr) \Bigr)$ is defined as the mapping given by
\begin{equation}
\label{eq:conv_operator_final}
    \bigl(\sW(\theta){u} \bigr)(x) \coloneqq \int_{D} {\kappa_\theta}(x,y){u}(y) \mathrm{d}y, \qquad \forall u\in \cU\bigl(D;\R^r),\, \forall x\in D, 
\end{equation}
where ${\kappa}_\theta:D\to \R^{{r'}\times r}$ is a neural network parametrized by $\theta \in \Theta_{\sW}$.
\end{definition}
Assuming ${\kappa}_\theta(x,y) = {\kappa}_\theta(x-y)$, hence the kernel to be periodic, then \eqref{eq:conv_operator_final} may be viewed as a convolution operator, which makes it possible to use the Fast Fourier Transform (FFT) to compute \eqref{eq:conv_operator_final} and to parametrize $\sW$ in Fourier space. Following \citet{li2021fourier}, this insight leads to the following formulation of the operator given in Definition \ref{def:int_op}.

\begin{definition}[Fourier Integral Operator $\sW$]
\label{def:FourierIntegralOperator_final}
We define the Fourier integral operator as 
    \begin{equation}
        \bigl(\sW(\theta){u} \bigr)(x) = \cF^{-1}\Bigl(R_\sW(\theta) \cdot (\cF{u}) \Bigr)(x)\qquad \forall u\in \cU\bigl(D;\R^r),\,\forall x\in D.
    \end{equation}
    To link to the definition \eqref{eq:conv_operator_final} with translation-invariant kernel
we take $R_{\sW}(\theta)$ to be the Fourier transform of the periodic function ${\kappa_\theta}:D\to \R^{{r'}\times r}$ parametrized by $\theta \in \Theta_{\sW}$.
    \end{definition}
For uniform discretizations $n_1\times\cdots\times n_d=N$ of the space $D$, the Fourier transform can be replaced with the Fast Fourier Transform (FFT).
Indeed, for the Fourier transform we have 
$\cF u : \R^d \to \C^r$ and $R_\sW(\theta) : \R^d \to \C^{r' \times r}$.
On the other hand, for the FFT we have
$\cF u : \Z^d \to \C^r$ and $R_\sW(\theta) : \Z^d \to \C^{r' \times r}$. In practice, it is possible to select a finite-dimensional parametrization by choosing a finite set of wave numbers 
\[
\bigl|Z_{k_{\textrm{max}}} \bigr|\coloneqq\bigl|\bigl\{ 
k\in\Z^d: |k_j|\leq k_{\textrm{max},j},\, \text{for}\,j=1,\ldots,d
\bigr\} \bigr|.
\] We can therefore parametrize $R_{\sW}$ by a complex $\bigl((2k_{\text{max},1}+1)\times \ldots\times (2k_{\text{max},d}+1) \times r'\times r\bigr)$-tensor. We will use the notation $k_{\textrm{max}}=\bigl|Z_{k_{\textrm{max}}} \bigr|$ when discussing parameter scalings in the subsequent sections.

These integral operator definitions will be used to define architectures for the Vision Transformer Neural Operator in Subsection \ref{subsec:patch_transformer} and the Fourier Attention Neural Operator in Subsection \ref{subsec:fourier_patch_transformer}, but are not needed to define the basic Transformer Neural Operator in Subsection \ref{subsec:vanilla_transformer}. In the context of each architecture, we make explicit the parametrizations of the integral operators employed.

\subsection{Transformer Neural Operator}
\label{subsec:vanilla_transformer}

\begin{figure}[h!]
\centering
\includegraphics[width=\linewidth]{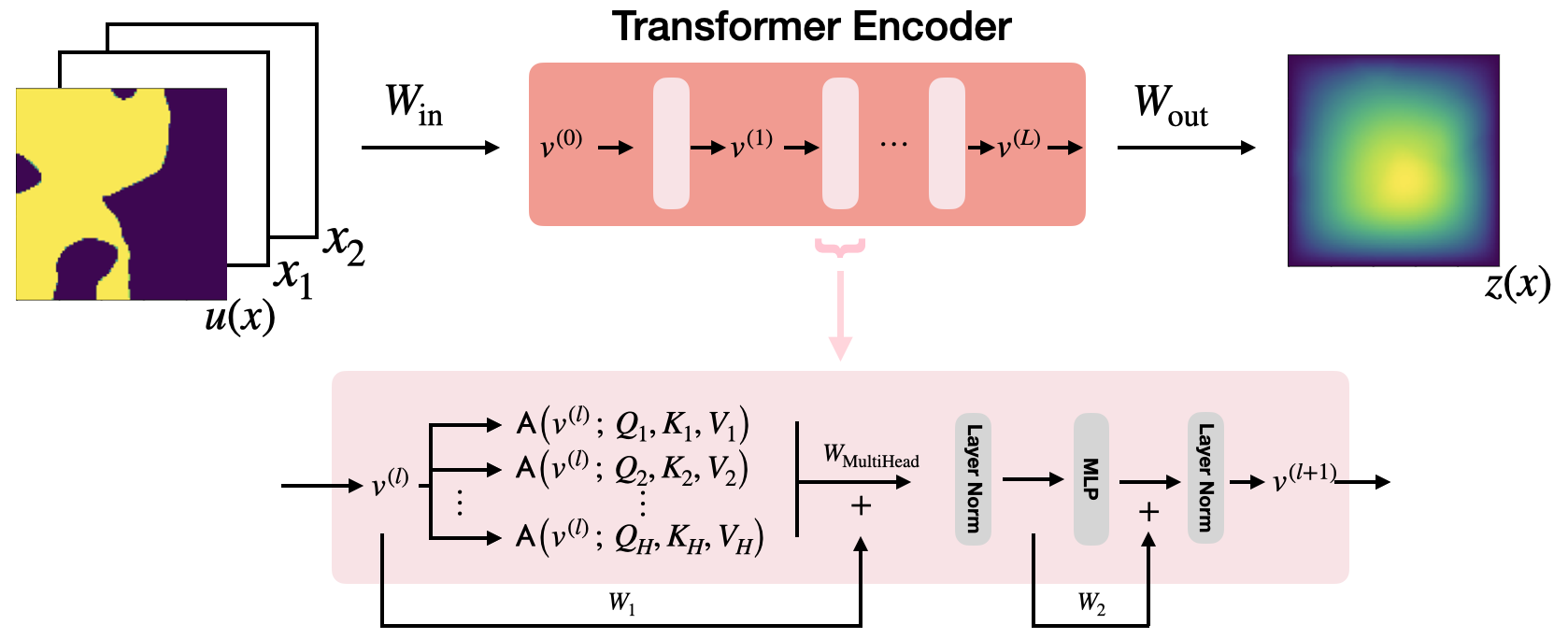}
\caption{Transformer Neural Operator.}
\label{fig:TNO}
\end{figure}

The transformer neural operator architecture follows the general structure of \Cref{eq:transformer_operator_compact,eq:recurrence_enc_def}; here, we make explicit the choices for $\sT_\textrm{in}, \sT_\textrm{out}, \sE_L$ that instantiate it. In \Cref{fig:TNO} we provide a schematic representation of the full transformer neural operator architecture.
We consider model inputs  $u \in \cU(D; \R^{d_u})$, model
embedding space 
$\cV := \cU(D; \R^{d_\textrm{model}})$,
and model output space
$\cZ(D; \R^{d_z})$,
with $D \subset \R^d$. The function $u\in \cU(D;\R^{d_u})$ is first concatenated with the identity defined on the domain $D$, so that the resulting function $u_{\textrm{in}}\in \cU(D;\R^{d_u+d})$ is defined by its action on $x\in D$ as
\begin{equation}
    \label{eq:concat}
    u_{\textrm{in}}(x) \coloneqq \bigl(u(x), x \bigr).
\end{equation}
In practice, this step defines a positional encoding that is appended to the input function. A linear transformation $W_\textrm{in}\in \R^{d_\textrm{model}\times (d_u+d)}$ is then applied pointwise to lift the function $u_{\textrm{in}}$ into the embedding space $\cV$. The input operator $\sT_\textrm{in} : \cU(D; \R^{d_u}) \to \cV$ is hence defined as
\begin{equation}
    \Bigl(\sT_\textrm{in}(u)\Bigr)(x) :=   W_\textrm{in}u_\textrm{in}(x)
\end{equation}
for $u_\textrm{in}\in \cU\bigl(D;\R^{d_u+d} \bigr)$ defined as in \eqref{eq:concat}, where $x \in D$, $u(x) \in \R^{d_u}$, and $W_\textrm{in} \in \R^{d_\textrm{model}\times (d_u+d)}$. Observe that $\sT_\textrm{in}$ acts as a local, pointwise linear operator on $u_\textrm{in}$.

Next, we define the details of $\sE_L: \cV \to \cV$ in \eqref{eq:recurrence_enc_def} by specifying a definition of $\sA_\textrm{MultiHead}$, $\sW_1,\sW_2$, $\sF_\textrm{LayerNorm}$, and $\sF_{\textrm{NN}}$ for the transformer neural operator setting. The operator $\sA_{\textrm{MultiHead}}: \cV \to \cV$ is the multi-head attention operator from \eqref{eq:multihead1} based on the continuum self-attention from Definition \ref{d:4}. Letting $H\in \N$ denote the number of attention heads, the multi-head attention operator is parametrized by the learnable linear transformations $Q_h\in \R^{d_K\times d_{\textrm{model}}},K_h\in \R^{d_K\times d_{\textrm{model}}},V_h\in \R^{d_K\times d_{\textrm{model}}}$, for $h=1,\ldots,H$, where for implementation purposes $d_K$ is chosen as $d_K\coloneqq d_{\textrm{model}}/H$.\footnote{We note that $d_{\textrm{model}}$ and $H$ should be chosen so that $d_{\textrm{model}}$ is a multiple of $H$.} 
The multi-head attention operator is thus defined by the application of a linear transformation $W_{\textrm{MultiHead}}\in \R^{d_{\textrm{model}}\times H\cdot d_K}$ to the concatenation of the outputs of $H\in\N$ self-attention operators. The operators $\sW_1,\sW_2:\cV\to\cV$ are pointwise linear operators and hence admit finite-dimensional representations $W_1,W_2\in \R^{d_{\textrm{model}}\times d_{\textrm{model}}}$, so that they define a map
\[
v(x)\mapsto W_1v(x),
\]
for each $x\in D$, and similarly for $W_2$. The operator $\sF_{\textrm{LayerNorm}}: \cV \to \cV$ is defined such that
\begin{equation}
\label{eq:LN}
\Bigl(\sF_{\textrm{LayerNorm}}({v};\gamma,\beta)(x)\Bigr)_k = \gamma_k\cdot\frac{\bigl(v(x)\bigr)_k - m\bigl(v(x)\bigr)}{\sqrt{\sigma^2\bigl(v(x)\bigr)+\epsilon}} +\beta_k,
\end{equation}
for $k=1,\dots,d_{\textrm{model}}$, any $v \in \cV$, and any $x \in D \subset \R^{d}$, where the notation $(\,\cdot\,)_k$ is used to denote the $k$'th entry of the vector. In equation \eqref{eq:LN}, $\epsilon\in\R^+$ is a fixed parameter, $\gamma_k,\beta_k\in\R$ are learnable parameters and $m,\sigma$ are defined as
\begin{equation}
\label{eq:LN_meanvar}
\begin{aligned}
    m(y) &\coloneqq \frac1{d_{\textrm{model}}}\sum_{k=1}^{d_{\textrm{model}}}y_k,\\
    \sigma^2(y) &\coloneqq \frac1{d_{\textrm{model}}}\sum_{k=1}^{d_{\textrm{model}}}\bigl(y_k-m(y) \bigr)^2 ,
\end{aligned}
\end{equation}
for any $y\in \R^{d_{\textrm{model}}}.$ 
The operator $\sF_\textrm{NN}: \cV \to \cV$ is defined such that
\begin{equation}
\label{eq:NN}
    \sF_\textrm{NN}(v;W_3,W_4,b_1,b_2)(x) = W_3f\bigl(W_4v(x)+b_1\bigr) + b_2,
\end{equation}
for any $v \in \cV$ and $x\in D$, where $W_3,W_4\in\R^{d_\textrm{model}\times d_\textrm{model}}$ and $b_1,b_2 \in \R^{d_\textrm{model}}$ are learnable parameters and where $f$ is a nonlinear activation function. 

Finally, we define the output operator $\sT_\textrm{out} : \cV \to \cZ(D; \R^{d_z})$ as a local linear operator applied pointwise, so that 
\begin{equation}
    \Bigl(\sT_\textrm{out}({v})\Bigr)(x) = W_\textrm{out}v(x),
\end{equation}
for $x \in D$, $v \in \cV$ and $W_\textrm{out} \in \R^{d_z\times d_\textrm{model}}$. 

The composition of the operators $\sT_\textrm{in}, \sE_L, \sT_\textrm{out}$ completes the definition of the transformer neural operator. In Appendix \ref{subsubsec:vanilla_transformer_implement} we outline how this neural operator architecture is implemented in the finite-dimensional setting, where the input function $u\in \cU\bigl( D; \R^{d_u} \bigr)$ is defined on a finite set of discretization points $\{x_i\}_{i=1}^N\subset D$.

\subsection{Vision Transformer Neural Operator}
\label{subsec:patch_transformer}

We introduce an analogue of the vision transformer architecture as a neural operator using the patching in the continuum framework developed in Section \ref{sec:patch_attention}. We note that in the context of the ViT from \citet{dosovitskyi2021image}, performing patch ``flattening'' before a local lifting to the embedding dimension introduces a nonlocal transformation on the patches, as described in Remark \ref{rm:patch_vit}. This procedure as implemented in \citet{dosovitskyi2021image} violates the desirable (for neural operators) discretization invariance property. Hence, for a neural operator analogue, we substitute the lifting to embedding dimension with a nonlocal integral operator. The ViT neural operator (ViT NO) introduced is thus discretization invariant. We describe this methodology 
within the framework provided by 
making explicit choices for $\sT_\textrm{in}, \sT_\textrm{out}, \sE_L$ that instantiate
\Cref{eq:transformer_operator_compact,eq:recurrence_enc_def} appropriately. In \Cref{fig:vit} we provide a schematic for the architecture, with a graphical representation of the encoder layer in \Cref{fig:vit_enc}.

\begin{figure}[h!]
\centering
\includegraphics[width=\linewidth]{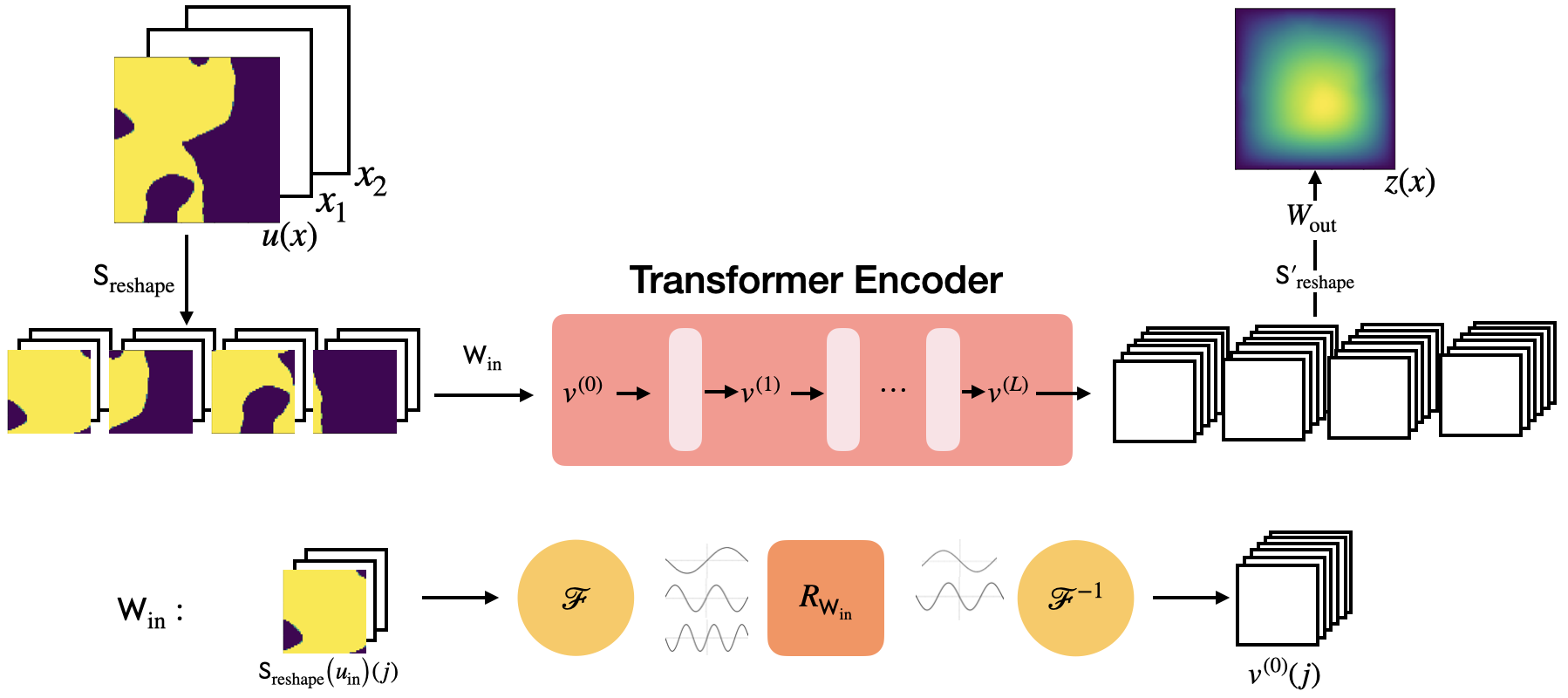}
\caption{Vision Transformer Neural Operator.}
\label{fig:vit}
\end{figure}

Let $D \subset \R^d$ be a bounded open set. Consider $u\in \cU(D;\R^{d_u})$ and let $D\coloneqq D_1 \cup \cdots \cup D_P$ be a uniform partition of the space $D$ so that $D_j\cong D'$ have equivalent finite volume for all $j\in\range{1}{P}$ for some $D'\subset D$, where $P\in\mathbb{N}$ represents the number of patches. 
We consider model 
inputs 
$u \in \cU(D; \R^{d_u})$, model
embedding space 
$\cV := \left\{v\mid v:\range{1}{P}\to  \cU\bigl(D';\R^{d_{\textrm{model}}}\bigr)\right\}$,
which is a space of functions acting on patch indices, and model output space
$\cZ(D; \R^{d_z})$,
with domain $D \subset \R^d$.

We begin by defining the input function $\sT_\textrm{in} : \cU(D; \R^{d_u}) \to \cV$.
The function $u\in \cU(D;\R^{d_u})$ is first concatenated with the identity defined on the domain $D$, so that the resulting function $u_{\textrm{in}}\in \cU(D;\R^{d_u+d})$ is defined as in \eqref{eq:concat}; recall that this defines a positional encoding. An operator $\sS_{\textrm{reshape}}$ is applied to the function ${u}_{\textrm{in}} \in \cU(D;\R^{d_u+d})$ that acts on the input so that the resulting function is of the form
\begin{equation}
    \label{eq:reshape}
    \sS_{\textrm{reshape}}({u}_{\textrm{in}}): \range{1}{P} \to \cU(D' ,\mathbb{R}^{d_u+d}).
\end{equation}
We apply a nonlocal linear operator $\sW_{\textrm{in}}\in \cL\Bigl(\cU(D';\R^{d_u+d})\,;\,\cU\bigl(D';\R^{d_{\textrm{model}}}\bigr)\Bigr)$ that acts on any $u:\range{1}{P}\to \cU(D';\R^{d_u+d})$ so that for any $j \in \range{1}{P}$
\begin{equation}
    \label{eq:lift_patch}
u(j) \mapsto \sW_{\textrm{in}}u(j).
\end{equation}
The above nonlocal linear operator is chosen as an integral operator as in Definition \ref{def:FourierIntegralOperator_final} and the subsequent discussion.
The composition of the operators defined in \Cref{eq:concat,eq:reshape,eq:lift_patch} defines our choice of $\sT_\textrm{in}$. Namely, we define $\sT_\textrm{in}: \cU(D; \R^{d_u}) \to \cV$ via
\begin{equation}
\label{eq:compt}
    \Bigl(\sT_\textrm{in}(u)\Bigr)(j) \coloneqq  \sW_{\textrm{in}} \Bigl(\sS_{\textrm{reshape}}\bigl({u}_{\textrm{in}} \bigr)(j)\Bigr),
\end{equation}
where ${u}_{\textrm{in}} \in \cU(D;\R^{d_u+d})$ is defined as in \eqref{eq:concat}, $j\in\range{1}{P}$ and $\sW_{\textrm{in}}$ is defined as in \eqref{eq:lift_patch}.

\begin{figure}[h!]
\centering
\includegraphics[width=0.7\linewidth]{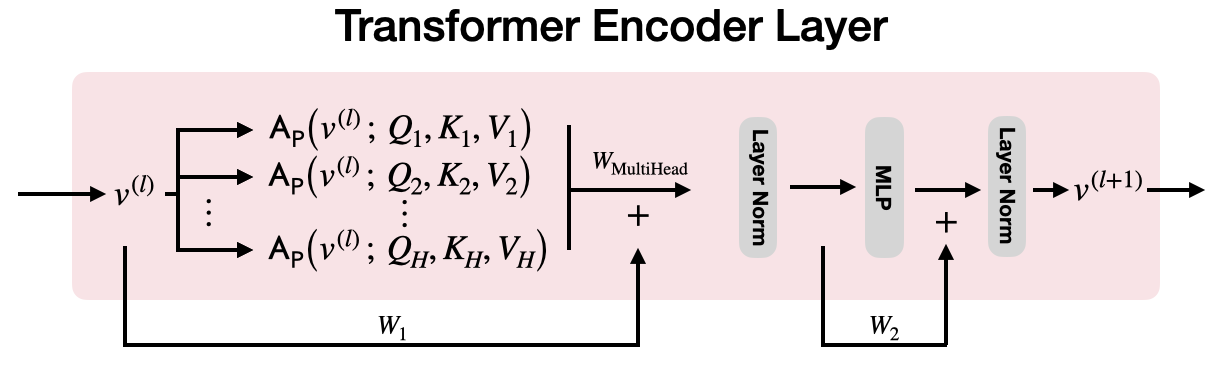}
\caption{Vision Transformer Neural Operator Encoder Layer.}
\label{fig:vit_enc}
\end{figure}

Next, we define the details of $\sE_L: \cV \to \cV$ in \Cref{eq:recurrence_enc_def} by specifying a definition of $\sA_\textrm{MultiHead}$, $\sF_\textrm{LayerNorm}$, $\sW_1,\sW_2$, and $\sF_{\textrm{NN}}$ in this patched setting.
In each layer the operator $\sA_{\textrm{MultiHead}}: \cV \to \cV $ is the multi-head attention operator from \eqref{eq:multihead1} based on the self-attention from Definition \ref{def:patch_attention}. Letting $H\in \N$ denote the number of attention heads, the multi-head attention operator is parametrized by learnable local (pointwise) linear operators $\sQ_h, \sK_h, \sV_h$ for $h=1,\ldots,H$. In this setting, these are chosen to be linear operators applied pointwise, they can be represented by finite dimensional linear transformations $Q_h\in \R^{d_K\times d_{\textrm{model}}},K_h\in \R^{d_K\times d_{\textrm{model}}},V_h\in \R^{d_K\times d_{\textrm{model}}}$, for $h=1,\ldots,H$. \footnote{We note that in the context of the Fourier attention neural operator, presented in the next Subsection \ref{subsec:fourier_patch_transformer}, we generalize this definition to $\sQ_h, \sK_h, \sV_h$ for $h=1,\ldots,H$ being nonlocal linear integral operators.} For implementation purposes $d_K$ is chosen as $d_K\coloneqq d_{\textrm{model}}/H$.
\footnote{We note that $d_{\textrm{model}}$ and $H$ should be chosen so that $d_{\textrm{model}}$ is a multiple of $H$.} 
The multi-head attention operator is thus defined by the application of a local linear transformation $W_{\textrm{MultiHead}}\in \R^{d_{\textrm{model}}\times H\cdot d_K}$ to the concatenation of the outputs of $H\in\N$ self-attention operators. The operators $\sW_1,\sW_2:\cV\to\cV$ are pointwise linear operators and hence admit finite-dimensional representations $W_1,W_2\in \R^{d_{\textrm{model}}\times d_{\textrm{model}}}$, so that they define a map
\begin{equation}
\label{eq:w1w2_ops}
v(x;p)\mapsto W_1v(x;p),
\end{equation}
for each $x\in D'$ and $p\in \range{1}{P}$, and similarly for $W_2$. The operator $\sF_{\textrm{LayerNorm}}: \cV \to \cV$ is defined such that
\begin{equation}
\label{eq:LN_patch}
\Bigl(\sF_{\textrm{LayerNorm}}(v;\gamma,\beta)(x;p)\Bigr)_k = \gamma_k\cdot\frac{\bigl(v(x;p)\bigr)_k - m\bigl(v(x;p)\bigr)}{\sqrt{\sigma^2\bigl(v(x;p)\bigr)+\epsilon}} +\beta_k,
\end{equation}
for $k=1,\dots,d_{\textrm{model}}$, any $p\in\range{1}{P}$ and any $x\in \R^{d}\subset D$, where we use the notation $(\,\cdot\,)_k$ to denote the $k$'th entry of the vector. In equation \eqref{eq:LN_patch}, $\epsilon\in\R^+$ is a fixed parameter, $\gamma_k,\beta_k\in\R$ are learnable parameters and $m,\sigma$ are defined as in \eqref{eq:LN_meanvar}. The operator $\sF_\textrm{NN}:\cV \to \cV$ is defined by
\begin{equation}
\label{eq:NN_patch_f}
    \sF_\textrm{NN}(v;W_3,W_4,b_1,b_2)(x;p) = W_3f\bigl(W_4v(x;p)+b_1\bigr) + b_2,
\end{equation}
for any $x\in D$ and any $p\in\range{1}{P}$, where $W_3,W_4\in\R^{d_\textrm{model}\times d_\textrm{model}}$ and $b_1,b_2 \in \R^{d_\textrm{model}}$ are learnable parameters and $f$ is a nonlinear activation function. 

Finally, we define the action of $\sT_\textrm{out} : \cV \to \cZ(D; \R^{d_z})$. This is given by first applying a reshaping operator $\sS'_{\textrm{reshape}}$ to the output of the encoder, where this map is defined as
\begin{equation}
    \label{eq:unshape}
{v}\mapsto\sS'_{\textrm{reshape}}\bigl({v}\bigr)\in \cU(D;\mathbb{R}^{d_{\textrm{model}}}).
\end{equation}
We note that the operator $\sS'_{\textrm{reshape}}$ may be viewed as an inverse of $\sS_{\textrm{reshape}}$. A pointwise linear transformation $W_{\textrm{out}}\in \R^{d_z \times d_{\textrm{model}}}$ defined by
\begin{equation}
\label{eq:Wout_nl}
    v(x)\mapsto W_{\textrm{out}}{v}(x) \in \cZ(D;\R^{d_z}),
\end{equation}
for any $x\in D$, is then applied. This procedure completes the definition of the operator $\sT_\textrm{out} : \cV \to \cZ(D; \R^{d_z})$, which is given by 
\begin{equation}
\label{eq:T_out_patch}
    \Bigl(\sT_\textrm{out}(v)\Bigr)(x) \coloneqq W_{\textrm{out}} \Bigl(\sS'_{\textrm{reshape}}\bigl(v \bigr)(x)\Bigr),
\end{equation}
for any $v\in \cV$ and $x\in D$.

\begin{remark}
\label{rm:smoothing}
    Discontinuities may arise when the patching architecture is used and may be visible in error plots and can also interfere with self-composition of maps learned as solution operators of time-dependent PDEs. Ameliorating such discontinuities can be achieved using a smoothing operator as a final layer of the network.
    To this end, assuming $\cZ\bigl(D;\R^{d_z}\bigr) \subset H^{s}\bigl(D;\R^{d_z}\bigr)$ for some $s\geq 0$, we define for $\epsilon>0$ and $\alpha>1$ the operator $(I- \epsilon\Delta)^{-\alpha} :\cZ\bigl(D;\R^{d_z}\bigr)\to H^{s+2\alpha}\bigl(D;\R^{d_z}\bigr)$, and define the last layer of the architecture as
\begin{equation}
\label{eq:Wout_smooth}
{z} \mapsto (I- \epsilon\Delta)^{-\alpha} (z),
\end{equation}
for $z \coloneqq \Bigl(\sT_{\rm{out}}\circ \sE_L \circ \sT_{\rm{in}}\Bigr) (u)$.
\end{remark}

The composition of the operators $\sT_\textrm{in}, \sE_L, \sT_\textrm{out}$ completes the definition of the ViT neural operator. In Appendix \ref{subsubsec:ViT_implementation} we outline how this neural operator architecture is implemented in the finite-dimensional setting, where the input function $u\in \cU\bigl( D; \R^{d_u} \bigr)$ is defined on a finite set of discretization points $\{x_i\}_{i=1}^N\subset D$.

\subsection{Fourier Attention Neural Operator}
\label{subsec:fourier_patch_transformer}

In this subsection, we introduce a different transformer neural operator based on patching, the Fourier attention neural operator (FANO). In this setting, to gain additional expressivity, we replace the local linear operators $\sQ,\sK,\sV$ in the attention mechanism with nonlocal integral operators. In \Cref{fig:fano} we provide a schematic for the Fourier attention neural operator architecture, with a graphical representation of the encoder layer in \Cref{fig:fano_enc}. We outline the details of the architecture in the following discussion.

We again let $D \subseteq \R^d$ be a bounded open set. Consider $u\in \cU(D;\R^{d_u})$ and let $D\coloneqq D_1 \cup \ldots \cup D_P$ be a uniform partition of the space $D$ so that $D_j\cong D'$ for all $j\in\range{1}{P}$ for some $D'\subset D$, where $P\in\mathbb{N}$ represents the number of patches. 
We consider model 
inputs 
$u \in \cU(D; \R^{d_u})$, model
embedding space 
$\cV := \left\{v\mid v:\range{1}{P}\to  \cU\bigl(D';\R^{d_{\textrm{model}}}\bigr)\right\}$
which is a space of functions acting on patch indices
,
and model output space
$\cZ(D; \R^{d_z})$,
with domain $D \subset \R^d$.

\begin{figure}[h!]
\centering
\includegraphics[width=\linewidth]{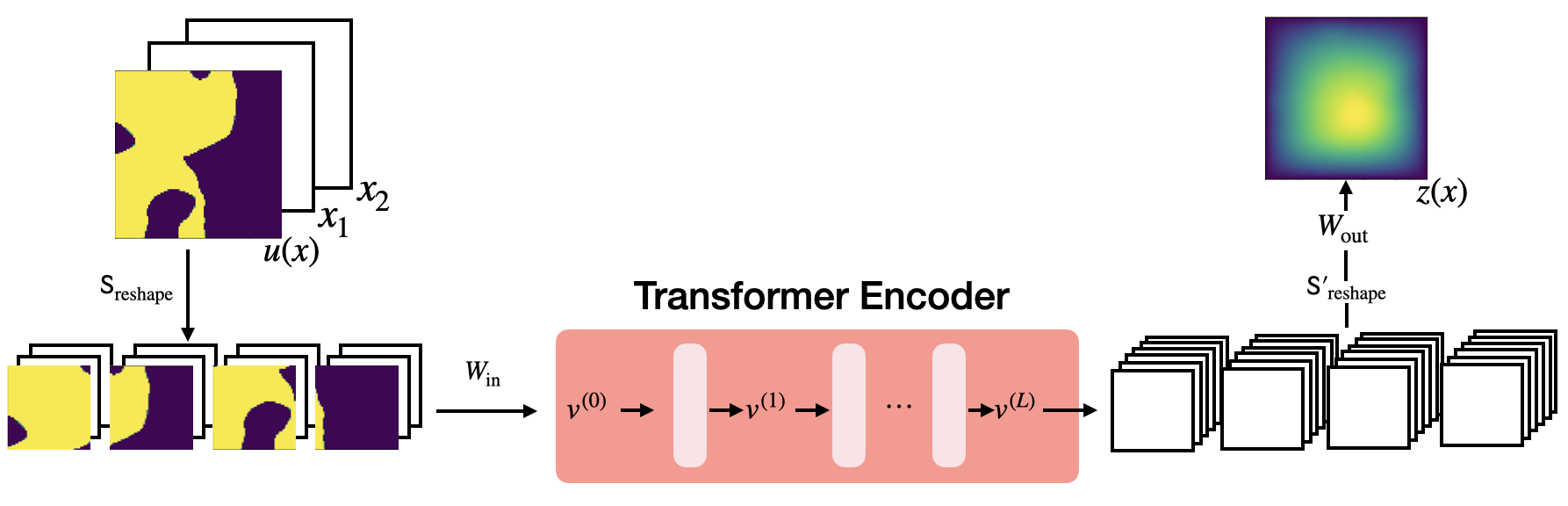}
\caption{Fourier Attention Neural Operator.}
\label{fig:fano}
\end{figure}

We define $\sT_\textrm{in}$ similarly to \eqref{eq:compt} in the previous section, via a composition of \Cref{eq:concat,eq:reshape,eq:lift_patch}; however here the linear lifting operator is chosen to be a pointwise linear transformation $W_{\textrm{in}}\in \R^{d_{\textrm{model}}\times (d_u+d)}$. Namely, we define $\sT_\textrm{in}: \cU(D; \R^{d_u}) \to \cV$ via
\begin{equation}
    \Bigl(\sT_\textrm{in}(u)\Bigr)\bigl(j\bigr)(x) \coloneqq  W_{\textrm{in}} \Bigl(\sS_{\textrm{reshape}}\bigl({u}_{\textrm{in}} \bigr)(j)\Bigr)(x),
\end{equation}
for any $j\in\range{1}{P}$ and any $x\in D$, where ${u}_{\textrm{in}} \in \cU(D;\R^{d_u+d})$ is defined as in \eqref{eq:concat}, $\sS_{\textrm{reshape}}$ as in $\eqref{eq:reshape}$ and the nonlocal operation
\eqref{eq:lift_patch} is replaced by a local pointwise one.

Next, we define the particulars of $\sE_L: \cV \to \cV$ in \Cref{eq:recurrence_enc_def}. The key distinction between the Fourier attention neural operator that we define here and the ViT neural operator from the previous subsection \ref{subsec:patch_transformer} lies in the nonlocality of operations performed within $\sA_\textrm{MultiHead}$. Here, in each layer the operator $\sA_{\textrm{MultiHead}}: \cV \to \cV $ is the multi-head attention operator from \eqref{eq:multihead1} based on the self-attention operator from Definition \ref{def:patch_attention}. Letting $H\in \N$ denote the number of attention heads, the multi-head attention operator is parametrized by learnable nonlocal linear integral operators $\sQ_h,\sK_h \in \cL\Bigl(\cU(D ',\mathbb{R}^{d_\textrm{model}}); L^2(D' ,\mathbb{R}^{d_K})\Bigr)$ and $ \sV_h \in \cL\Bigl(\cU(D ',\mathbb{R}^{d_\textrm{model}}); \cU(D' ,\mathbb{R}^{d_K})\Bigr)$, for $h=1,\ldots,H$, where for implementation purposes $d_K$ is chosen as $d_K\coloneqq d_{\textrm{model}}/H$. The specific formulation of these linear integral operators may be found in Definition \ref{def:FourierIntegralOperator_final} and the subsequent discussion. The multi-head attention operator is thus defined by the application of a pointwise linear transformation $W_{\textrm{MultiHead}}\in \R^{d_{\textrm{model}}\times H\cdot d_K}$ to a concatenation of the outputs of $H\in\N$ self-attention operators. The operators $\sW_1,\sW_2$ are defined as in \eqref{eq:w1w2_ops}. The layer normalization operator $\sF_{\textrm{LayerNorm}}:\cV\to \cV $ is defined as in \eqref{eq:LN_patch} so that it is applied to every point in every patch. Furthermore, the operator $\sF_{\textrm{NN}}:\cV\to \cV$ is defined as in \eqref{eq:NN_patch_f}. 

We define $\sT_\textrm{out}$ identically to \Cref{eq:T_out_patch}. We refer to Remark \ref{rm:smoothing} for an additional operation that can be applied to ameliorate the effect of patch discontinuities. Finally, the composition of the operators $\sT_\textrm{in}, \sE_L, \sT_\textrm{out}$ completes the definition of the Fourier attention neural operator. In Appendix \ref{subsubsec:patch_transformer_implement} we outline how this neural operator architecture is implemented in the finite-dimensional setting, where the input function $u\in \cU\bigl( D; \R^{d_u} \bigr)$ is defined on a finite set of discretization points $\{x_i\}_{i=1}^N\subset D$.

\begin{figure}[h!]
\centering
\includegraphics[width=0.7\linewidth]{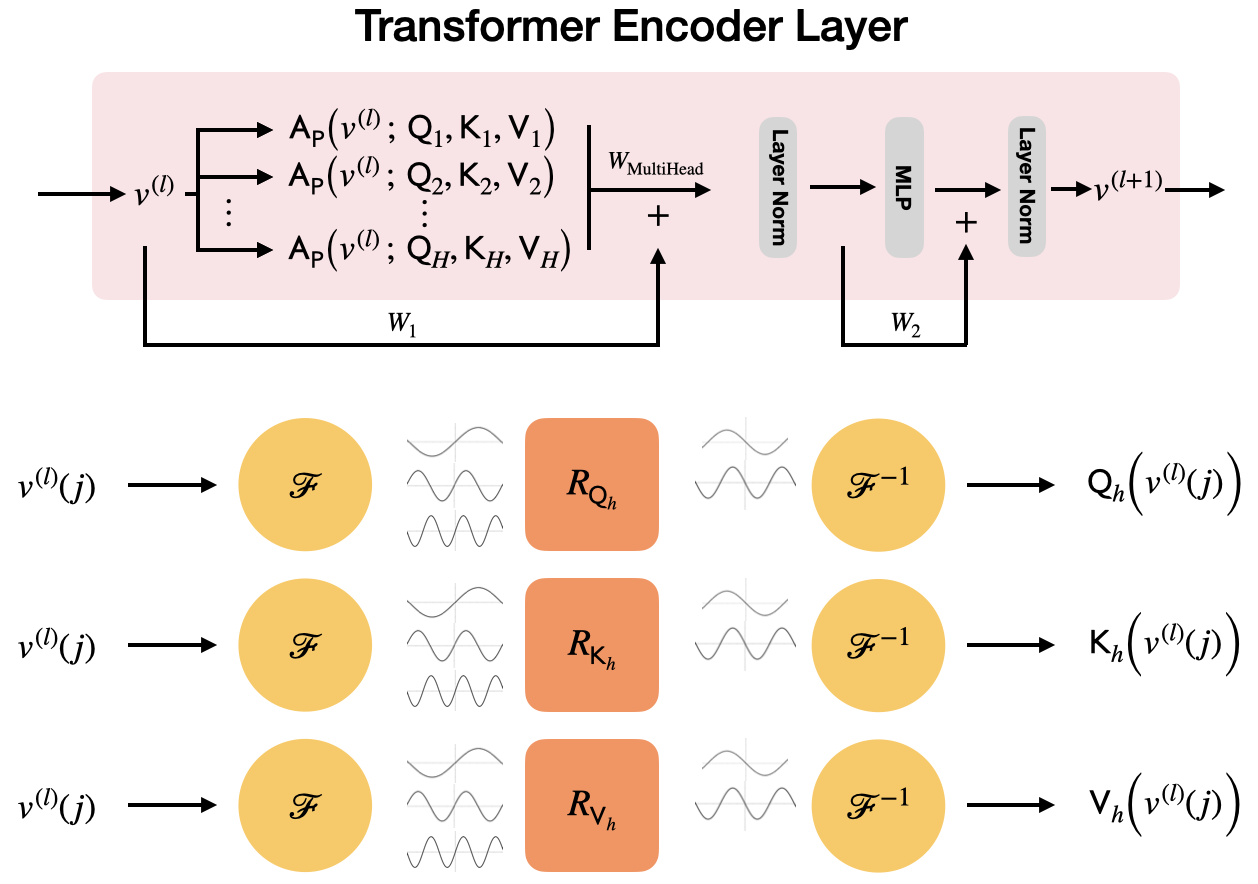}
\caption{Encoder Layer of the Fourier Attention Neural Operator.}
\label{fig:fano_enc}
\end{figure}

\section{Universal Approximation by Transformer Neural Operators}
\label{sec:universal_approx}

Universal approximation is a minimal requirement for neural operators. When the operator to be
approximated maps between spaces of functions defined over Euclidean domains, universal approximation
necessarily requires both nonlocality and nonlinearity. In the recent paper \citet{lanthaler2023nonlocal},
a minimal architecture possessing these two properties is exhibited. In this section, rather than
trying to prove universal approximation for the various discretization-invariant neural operators
introduced in the paper, we construct a simple canonical setting that exhibits the universality of the attention mechanism on function space, allowing direct exploitation of the ideas in \citet{lanthaler2023nonlocal}. We build on the notation used in Subsection \ref{subsec:setup}. In particular
$\sA$ denotes the self-attention operator from Definition \ref{d:4} and $f$ an activation function. Throughout this section, we consider activation functions $f \in C^\infty(\R)$ which are non-polynomial and Lipschitz continuous.

To be specific, we consider neural operators $\sG:\cU\bigl(D;\R^{d_u} \bigr)\times\Theta \to \cZ\bigl( D;\R^{d_z}\bigr)$ of the form
\begin{equation}
\label{eq:transformer_operator_approx_theory}
\sG(u;\theta) \coloneqq \Bigl(\sT_{\textrm{out}}\circ \sE \circ \sT_{\textrm{in}}\Bigr) (u;\theta),
\end{equation}
where $\sT_{\textrm{in}}: \cU\bigl(D;\R^{d_u} \bigr) \to \cV\bigl(D;\R^{d_\textrm{model}} \bigr)$ and $\sT_{\textrm{out}}: \cV\bigl(D;\R^{d_\textrm{model}}\bigr)\to\cZ\bigl(D;\R^{d_z} \bigr)$ are defined by neural networks of the form 
\begin{align}
    \Bigl(\sT_{\textrm{in}}(u)\Bigr)(x) &= R_2f\Bigl(R_1\bigl(u(x),x\bigr)+b_R \Bigr)+b_R',\\
    \Bigl(\sT_{\textrm{out}}(v)\Bigr)(x) &= P_2f\Bigl(P_1\bigl(v(x),x\bigr)+b_P \Bigr)+b_P',
\end{align}
where $\bigl(u(x),x\bigr)\in \R^{d_u+d}$ and $\bigl(v(x),x\bigr)\in \R^{d_{\textrm{model}}+d}$, where $R_1,R_2,P_1,P_2$ are learned linear transformations of appropriate dimensions and $b_R,b_R',b_P,b_P'$ are learned vectors. 
We define the operator $\sE: \cV\bigl(D;\R^{d_\textrm{model}} \bigr) \to \cV\bigl(D;\R^{d_\textrm{model}} \bigr)$ as a variant of the layer defined by the iteration step in \Cref{eq:recurrence_enc_def}, given by the two-step map acting on its inputs $v\in\cV$ as
\begin{subequations}
\label{eq:recurrence_enc_func_ua}
\begin{align}
    {v}(x) &\mapsfrom W_1{v}(x) + \sA\big({v}; Q,K,V\big)(x), \\
    v(x) &\mapsfrom W_2{v}(x) + W_3f\bigl(W_4{v}(x)+b_1\bigr) + b_2,
\end{align}
\end{subequations}
for any $x\in D$. Note that the neural operator thus defined does not include layer normalizations;
it is thus a variant of the transformer neural operator from Subsection \ref{subsec:vanilla_transformer}. 
We may now apply the result of \citet{lanthaler2023nonlocal} to show two universal approximation theorems for the resulting transformer neural operator.

\begin{theorem}
\label{thm:ua1}
Let $D\subset \R^d$ be a bounded domain with Lipschitz boundary, and fix integers $s,s'\geq 0$. If $\sG^\dagger:C^s\bigl(\Bar{D};\R^r\bigr) \to C^{s'}\bigl(\Bar{D};\R^{r'}\bigr)$
is a continuous operator and $K\subset C^s\bigl(\Bar{D};\R^r\bigr)$ a compact set, then for any $\epsilon>0$, there exists a transformer neural operator 
$\sG(\cdot;\theta):K\subset C^s\bigl(\Bar{D};\R^r\bigr) \to C^{s'}\bigl(\Bar{D};\R^{r'}\bigr)$ so that
\begin{equation}
\label{eq:ua_condition}
\sup_{u\in K} \left\|\sG^\dagger(u) - \sG(u;\theta)   \right\|_{C^{s'}} \leq \epsilon.
\end{equation}
\end{theorem}
\begin{proof}
We begin by noting that for $Q,K=0$ and $V = I$, the self-attention mapping reduces to 
\begin{equation}
\sA\bigl(v;Q,K,V \bigr)(\cdot) = \mathbb{E}_{y\sim p(y:v,x)}[V v(\cdot)] = \frac{1}{|D|}\int v(x)~\mathrm{d} x.
\end{equation}
For weights $W_2=0$, $W_3=W_4=I$ and $b_2 = 0$, the transformer encoder neural operator layer \eqref{eq:recurrence_enc_func_ua} reduces to the mapping 
\begin{equation}
\label{eq:nonlocal_NO}
    v(\cdot) \mapsto f\left(W_1v(\cdot) + b_1 + \frac{1}{|D|}\int v(x)~\mathrm{d} x \right).
\end{equation}
The existence of $W_1,R_1,R_2,P_1,P_2,b_1,b_R,b_R',b_P,b_P'$ so that $\sG(\cdot\,;\theta)$ satisfies \eqref{eq:ua_condition} then follows from \citet[Theorem 2.1]{lanthaler2023nonlocal}, which also involves the application of the universality result for two-layer neural networks of \citet[Theorem 4.1]{Pinkus_1999}.
\end{proof}

The analysis in \citet{lanthaler2023nonlocal} allows the derivation of an analogous universal approximation theorem for functions belonging to Sobolev spaces. This generalization is the content of
the next theorem. The proof follows easily by applying the same argument as in the proof of Theorem \ref{thm:ua1} and the result of \citet[Theorem 2.2]{lanthaler2023nonlocal}.
\begin{theorem}
\label{thm:ua2}
Let $D\subset \R^d$ be a bounded domain with Lipschitz boundary and fix integers $s,s'\geq 0$, $p,p'\in[1,\infty)$. If $\sG^\dagger:W^{s,p}\bigl(D;\R^r\bigr) \to W^{s',p'}\bigl(D;\R^{r'}\bigr)$
is a continuous operator and $K\subset W^{s,p}\bigl(D;\R^r\bigr)$ a compact set of bounded functions so that $\sup_{u\in K}\|u\|_{L^\infty}<\infty$, then for any $\epsilon>0$, there exists a transformer neural operator 
$\sG(\cdot;\theta):K\subset W^{s,p}\bigl(D;\R^r\bigr) \to W^{s',p'}\bigl(D;\R^{r'}\bigr)$ so that
\begin{equation}
\sup_{u\in K} \left\|\sG^\dagger(u) - \sG(u;\theta)   \right\|_{W^{s',p'}} \leq \epsilon.
\end{equation}
\end{theorem}


\section{Numerical Experiments}
\label{sec:numerics}

In this section we illustrate, through numerical experiments, the capability of the transformer neural operator architectures described in Section \ref{sec:transformers_operator}. Throughout, we consider the supervised learning problem described by the data model in \eqref{eq:data_model} and take the viewpoint of surrogate modeling to construct approximations of the operators $\sG^\dagger$. In Subsection \ref{subsec:1d_experiments} we consider problems given by dynamical systems and ordinary differential equations, namely operators acting on functions defined on a one-dimensional time domain. In this context, we explore the use of the transformer neural operator for operator learning problems given by the Lorenz `63 dynamical system and a controlled ODE. On the other hand, in Subsection \ref{subsec:2d_experiments} we consider problems given by partial differential equations equations, namely operators acting on functions defined on a two-dimensional spatial domain. In this context we explore the use of the transformer neural operator and of the more efficient ViT neural operator and Fourier attention neural operator architectures for operator learning problems given by the Darcy flow and Navier-Stokes equations. The code repository for the experiments in this paper may be found at \href{https://github.com/EdoardoCalvello/TransformerNeuralOperators}{https://github.com/EdoardoCalvello/TransformerNeuralOperators}.

\begin{remark}[Relative Loss]
For each neural operator architecture we optimize for the parameters $\theta \in\Theta$ according to the relative loss, i.e.
\begin{equation}
    \inf_{\theta\in\Theta}\, \frac1J\sum_{j=1}^J
\left(\frac{\|\sG\bigl(u^{(j)};\theta\bigr)-\sG^\dagger\bigl(u^{(j)}\bigr)\|_{\cZ}}{\|\sG^\dagger\bigl(u^{(j)}\bigr)\|_{\cZ}} \right),
\end{equation} 
for data points $\{u^{(j)}\}_{j=1}^J$.
\end{remark}

\subsection{1D Time Domain}
\label{subsec:1d_experiments}
In this subsection we consider problems where attention is applied to functions in the one-dimensional time domain. The first setting we consider is that of the Lorenz `63 dynamical system \citep{lorenz1963deterministic}, a simplified model for atmospheric convection. The dynamics are given by 
\begin{equation}\label{eq:l63}
\begin{alignedat}{2}
   {dx} &= \sigma (y - x){dt}, &\quad& \\
    {dy} &= \bigl(x (\rho - z) - y\bigr){dt}, && \\
    {dz} &= \bigl(xy - \beta z \bigr){dt},&&
\end{alignedat}
\tag{Lorenz `63}
\end{equation}
with parameters set to $\sigma=10$, $\rho=28$ and $\beta=8/3$. We consider learning the operator $\sG^\dagger:C\bigl([0,T];\R^{d_x} \bigr)\times \R^{d_y}\times \R^{d_z}\to C\bigl([0,T];\R^{d_y} \bigr)\times C\bigl([0,T];\R^{d_z} \bigr)$ given by 
\begin{equation}
\label{eq:xtoyz}
\sG^\dagger:\{x(t),y(0),z(0)\}_{t\in{[0,T]}} \mapsto \{y(t),z(t)\}_{t\in{[0,T]}}.
\end{equation}
In this case, $\sG^\dagger$ is known to exist and given by the solution operator for the system of linear ODEs arising from the second and third components of \eqref{eq:l63}, given known first component. We note that including the initial condition of the unobserved trajectories in the model introduces a discrepancy in dimension within the input function to the model. Namely, denoting by $u$ the input to the model, $u(0)\in \R^{3}$ while $u(t)\in \R$ for $t\neq 0$. We overcome this issue by learning separate embedding linear transformations $W_{\textrm{in}_0}$ and $W_{\textrm{in}}$, for the initial condition and the rest of the trajectory, respectively. In Figure~\ref{fig:xtoyz} we display the results obtained by applying a transformer neural operator model on a test data set of $1000$ samples with the same discretization ($\Delta t=0.01$ and $T=2$) as the training set of $8000$ samples on the operator learning problem described by \eqref{eq:xtoyz}. The operator is successfully learned, and displays median relative $L^2$ error of less than $1\%$, with worst case rising to just below $10\%.$

\begin{figure}[h!]
\centering
\includegraphics[width=\linewidth]{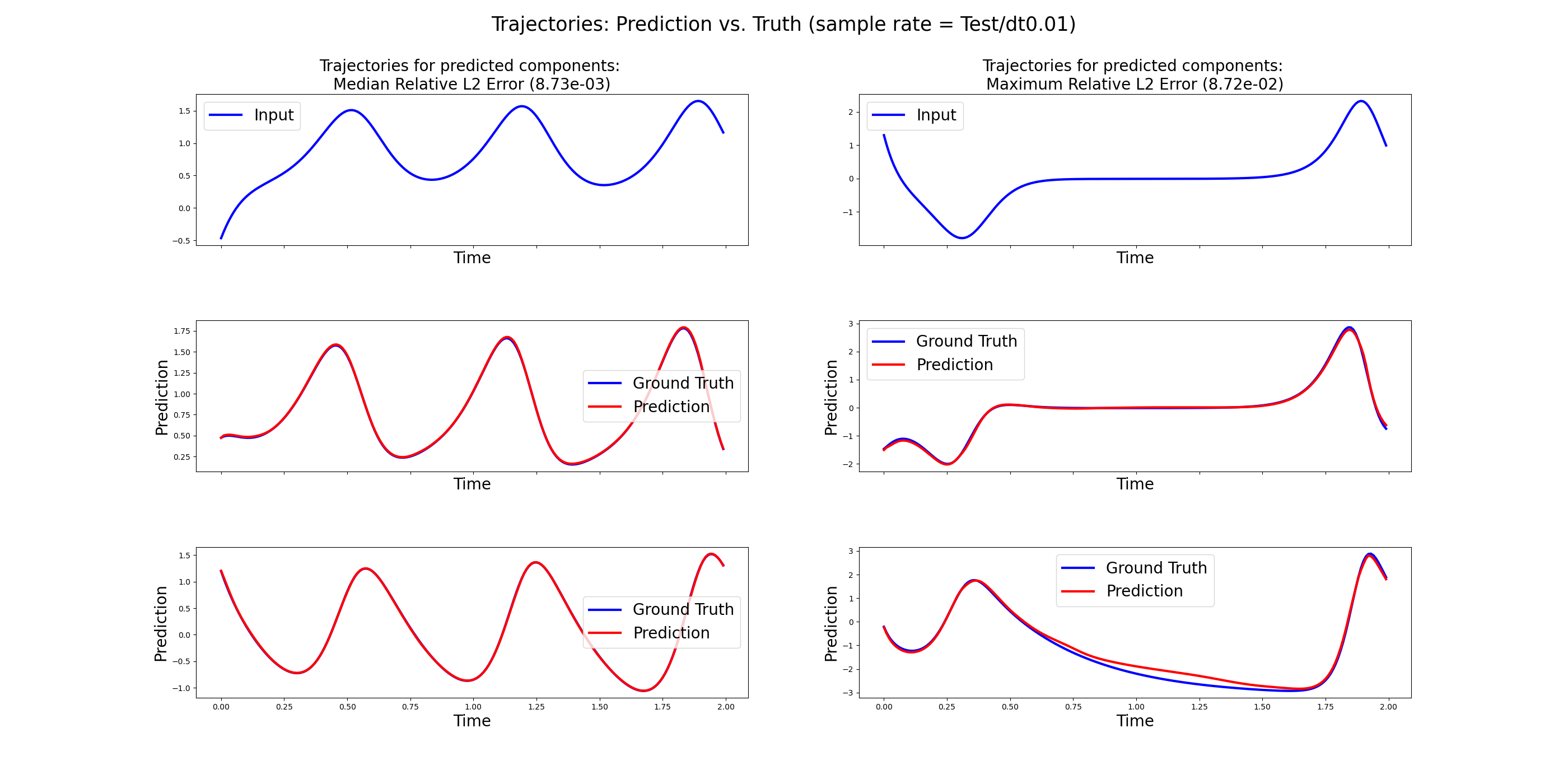}
\caption{The panel displays the performance
of the transformer neural operator when applied to the Lorenz `63 operator learning problem of recovering both unobserved $y$ and $z$ trajectories. In each column, the top plot shows the input $x$ trajectory data while the second and third row display the predicted against true $y$ and $z$ trajectories, respectively; the left column concerns the testing example achieving a median relative $L^2$ error, and the right columns shows the test example achieving the worst error.}
\label{fig:xtoyz}
\end{figure}

We also experiment with a setting for which $\sG^\dagger$ is not well-defined: we study the
recovery of unobserved trajectory $\{y(t)\}_{t\in{[0,T]}}$ from the observed trajectory 
$\{x(t)\}_{t\in{[0,T]}}$ so that
$\sG^\dagger:C\bigl([0,T];\R^{d_x} \bigr)\to C\bigl([0,T];\R^{d_y} \bigr)$  and
\[
\sG^\dagger: \{x(t)\}_{t\in{[0,T]}}\mapsto \{y(t)\}_{t\in{[0,T]}}.
\]
This operator would only be well-defined if $(y(0),z(0))$ were included as inputs. We
nonetheless expect accurate recovery of the hidden trajectory by the property of synchronization \citep{pecora1990synchronization,hayden2011discrete}.

We investigate an additional operator learning problem given by the setting of the controlled ODE
\begin{equation}
\label{eq:cde}
\begin{alignedat}{2}
    {dz} &= \sin(z) {du}, &\quad& z(0) = z_0, \\
    u(t) &= \sum_{j=1}^J\xi_j\sin(\pi \eta_j t),&\quad& \xi_j\sim\cN(0,1)\,\text{i.i.d. }, \eta_j\sim \unif([0,J])\,\text{i.i.d. }.  &&
\end{alignedat}
\tag{CDE}
\end{equation}
In this case we aim to approximate the operator $\sG^\dagger: C\bigl([0,T];\R^{d_u} \bigr) \to C\bigl([0,T];\R^{d_z} \bigr)$ defined by 
\[
\sG:\{u(t);\xi,\eta\}_{t\in{[0,T]}}\mapsto \{z(t);\xi,\eta\}_{t\in{[0,T]}}.
\]
For this problem we make the choice $J=10$.

The experimental settings we consider for the continuous time operator learning
problem may be viewed through the lens of data assimilation as smoothing problems \citep{Sanz-Alonso_Stuart_Taeb_2023}. Indeed, the mappings to be learned represent the recovery of an unobserved trajectory $\{z(t)\}_{t\in{[0,T]}}$ conditioned on the observation of a whole trajectory $\{u(t)\}_{t\in{[0,T]}}$.

\begin{figure}[h!]
\centering
\includegraphics[width=0.8\linewidth]{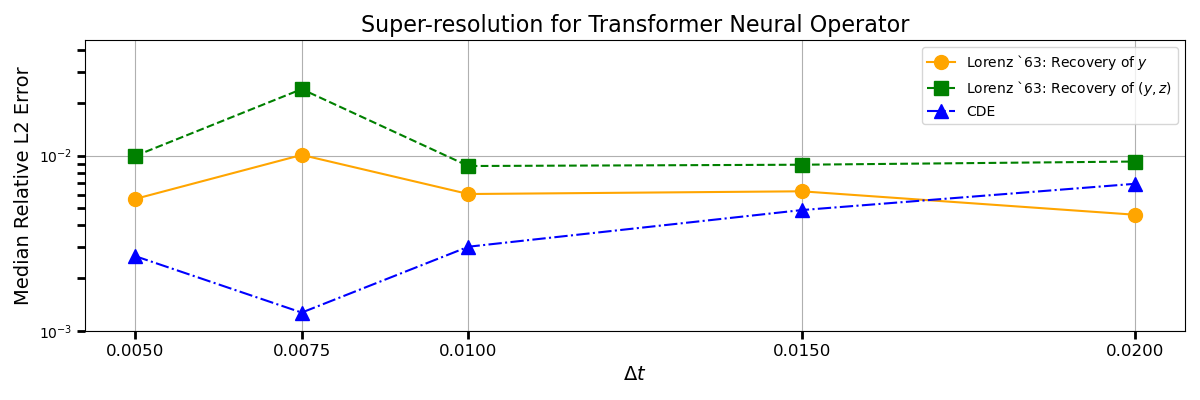}
\caption{The panel displays the performance in relative $L^2$ errors (in log-scale) of the transformer neural operator when applied to the Lorenz `63 and controlled ODE operator learning test sets of different resolutions at inference time (without retraining). All models were parametrized by $d_{\textrm{model}}=128$ and $6$ encoder layers, and were trained with $8000$ trajectories with $T=2$, sampled according to $\Delta t = 0.01$.}
\label{fig:superresolution_1d}
\end{figure}

In Figure \ref{fig:superresolution_1d} we demonstrate the mesh-invariance property of the transformer neural operator in the context of the operator learning problems arising from equations \eqref{eq:l63} and \eqref{eq:cde}. 
A key motivation for considering the continuum formulation of attention and the related transformer neural operator is the property of zero-shot generalization to different non-uniform meshes. To demonstrate
this, in Figure \ref{fig:irregular} we consider the operator learning problem of mapping the observed $x$ trajectory of \eqref{eq:l63} to the unobserved $y$ trajectory. We display the median relative $L^2$ error and worst error samples from a test set consisting of trajectories $\{x(t)\}_{t\in\cT_\textrm{test}}$, where the test time index set $\cT_\textrm{test}$ is defined as
\begin{equation}
\label{eq:irregular_set}
\cT_\textrm{test}\coloneqq \{n\Delta t \}_{n=0}^{N/2}\cup\{2n\Delta t\}_{n=N/2+1}^{3N/4},
\end{equation}
for $\Delta t\coloneqq T/N$, where the model deployed is trained on trajectories indexed at a set $\cT_\textrm{train}$ defined by 
\begin{equation}
\label{eq:regular_set}
\cT_\textrm{train} \coloneqq \{n\Delta t \}_{n=0}^{N}.
\end{equation}
In other words, the second half of the testing domain is up-sampled by a factor of 2 compared to the training discretization. As seen qualitatively in Figure \ref{fig:irregular} testing on the non-uniform discretization 
based on a model trained on the uniform grid is successful. The use of different grids in test and train data does
lead to a larger error than the one incurred when they match (Figure \ref{fig:superresolution_1d}),  but the errors remain small. It is our continuum formulation of attention from \Cref{subsubsec:vanilla_transformer_implement} that enables this direct generalization; indeed, as displayed in \Cref{fig:irregular_t}, the transformer from \cite{vaswani2017attention} when applied to the same setting yields larger errors than the transformer neural operator, due to the TNO's reweighting based on the discretization of the grid.
 
\begin{figure}[h!]
\centering
\includegraphics[width=0.8\linewidth]{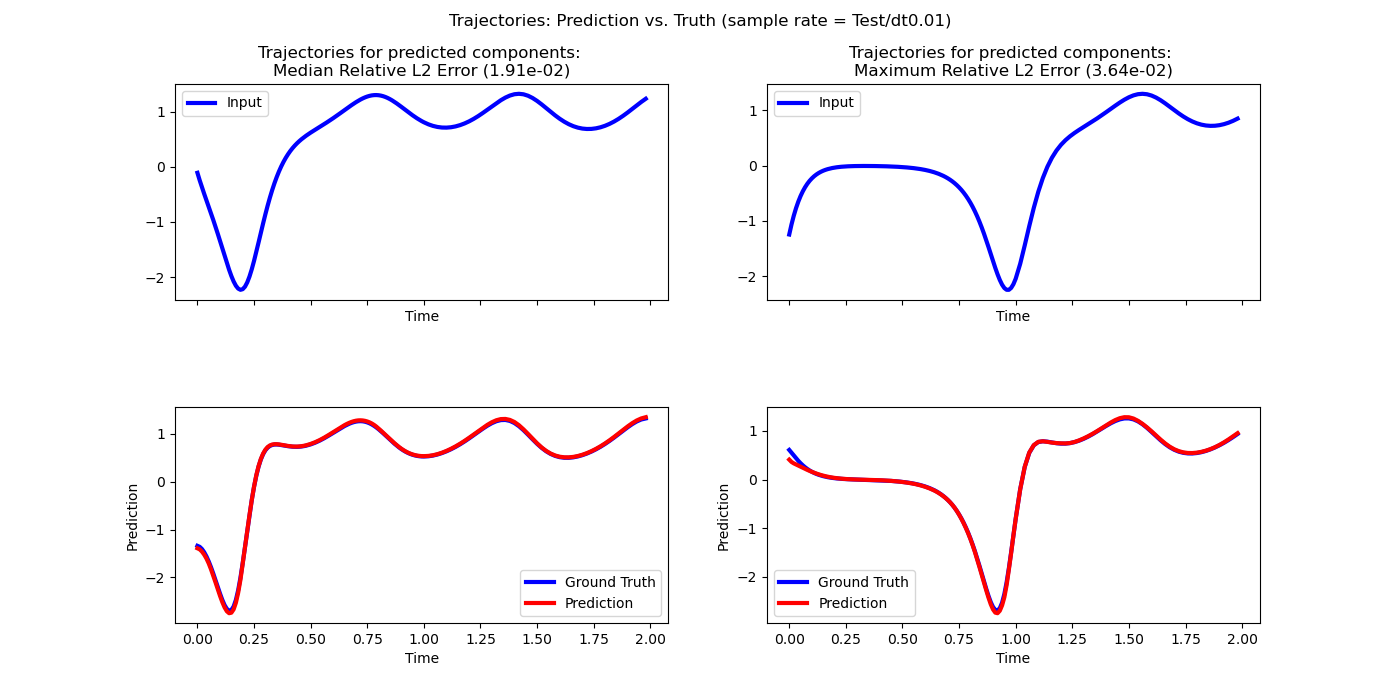}
\caption{The panel displays the performance in relative $L^2$ errors of the transformer neural operator when applied to the Lorenz `63 operator learning problem of recovering the unobserved $y$ trajectory. The results displayed concern a model trained on trajectories indexed on the index set \eqref{eq:regular_set} deployed on a test of trajectories indexed on \eqref{eq:irregular_set}. In each column, the top plot shows the input data and the bottom compares the prediction and ground truth; the left column shows the testing example achieving a median error, and the right columns shows the test example achieving the worst error.}
\label{fig:irregular}
\end{figure}

\begin{figure}[h!]
\centering
\includegraphics[width=0.8\linewidth]{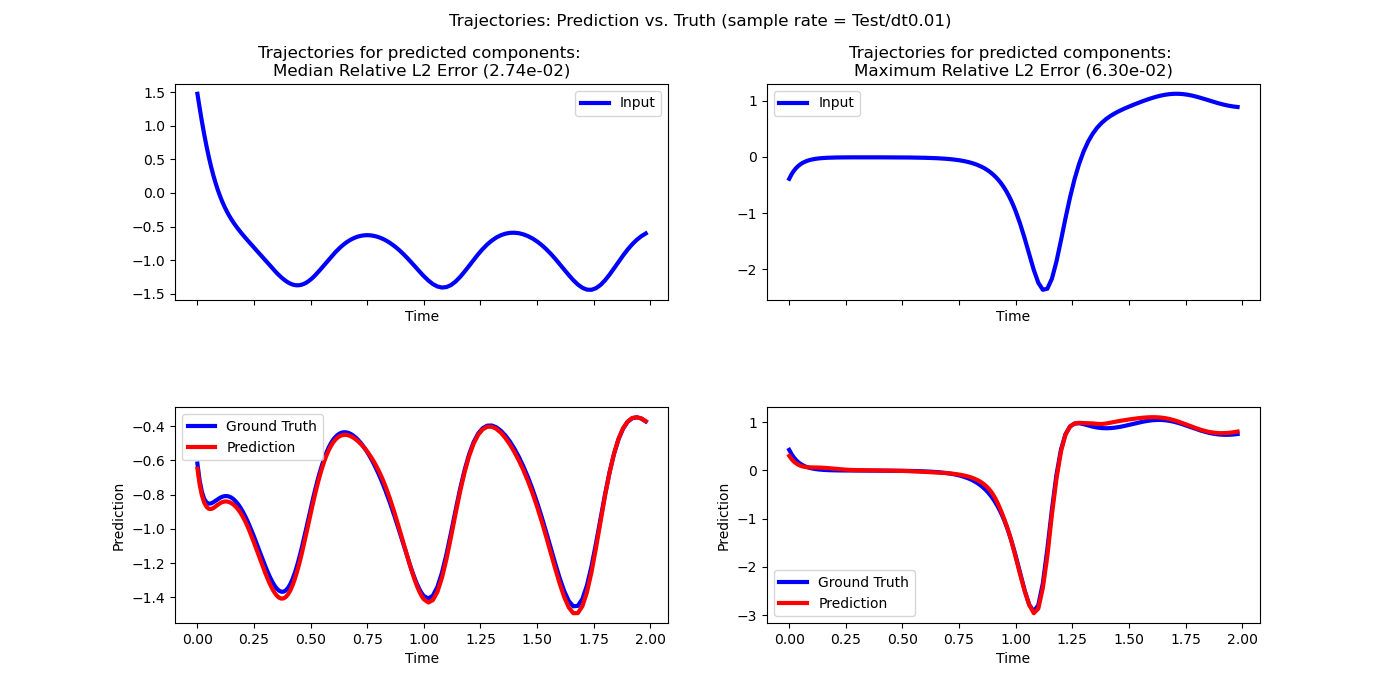}
\caption{The panel displays the performance in relative $L^2$ errors of the transformer from \cite{vaswani2017attention} when applied to the Lorenz `63 operator learning problem of recovering the unobserved $y$ trajectory. The results displayed concern a model trained on trajectories indexed on the index set \eqref{eq:regular_set} deployed on a test of trajectories indexed on \eqref{eq:irregular_set}. In each column, the top plot shows the input data and the bottom compares the prediction and ground truth; the left column shows the testing example achieving a median error, and the right columns shows the test example achieving the worst error.}
\label{fig:irregular_t}
\end{figure}

\subsection{2D Spatial Domain}
\label{subsec:2d_experiments}


In this subsection we study two different 2d operator learning problems -- flow in a porous medium, 
governed by the Darcy model, and incompressible viscous fluid mechanics governed by the Navier--Stokes
equation, in particular, a Kolmogorov flow. We apply the three architectures developed in
Sections \ref{sec:attention} and \ref{sec:patch_attention}, concentrating on the patched architectures. 
In the first Subsection \ref{subsec:2dr}
we report the results obtained by applying the neural operators described to an array of operator learning problems, where the domain of the functions is two-dimensional i.e. $D\subset \R^2$. 
In Subsection \ref{subsubsec:darcy} we describe the settings of the two Darcy flow operator learning problems considered, along with the parametrization and training details for the various architectures. Similarly in Subsection \ref{subsubsec:NS}, we describe the Kolmogorov flow operator learning problem and relevant implementation details. Subsections \ref{subsubsec:darcy} and \ref{subsubsec:NS} also contain further
numerical results illustrating the behavior of the proposed transformer neural operators,
and discretization invariance in particular.

\subsubsection{Results}
\label{subsec:2dr}

The results presented concern the two variants of Darcy flow, with lognormal and piecewise constant
inputs, and the Kolmogorov flow. The operator learning architectures we deploy include the three
transformer-based methods introduced in this paper, and the implementations of 
the Fourier neural operator from the ``NeuralOperator'' library \citep{li2021fourier,kovachki2023neural}, the Galerkin transformer from the ``galerkin-transformer'' library \citep{cao2021choose}, and the AFNO \citep{guibas2022efficient} from the ``physics-nemo'' library \citep{physicsnemo2023}. We note that in order to obtain a numerical comparison of the attention mechanisms themselves, in the context of the Galerkin transformer architecture we only use the Galerkin transformer encoder, with lifting and projections achieved by linear layers as in TNO. This allows for a more direct comparison between the Fourier attention from \citet{cao2021choose} and our continuum attention, without potential expressivity gain from additional architectural components. In the following discussion, we do not include experiments performed using  the Codomain-attention neural operator \citep{rahman2024pretraining} in the high-resolution settings: the memory usage with the implementation from the ``NeuralOperator'' library did not allow for comparably sized models, in terms of parameters; the performance of the models we trained was hence incomparably inferior to the others.

Our first results compares the transformer neural operator from Subsection \ref{subsec:vanilla_transformer}
with FNO, AFNO and the Galerkin transformer (employing Fourier attention) on the lognormal Darcy Flow problem in a low $64\times 64$ resolution setting. Table \ref{tab:2d_median_errors_low} provides a brief summary of the test relative $L^2$ errors obtained. 
It is notable that the transformer architecture obtains the best performance with an order of magnitude fewer parameters. However, since higher resolution
renders the architecture from Subsection \ref{subsec:vanilla_transformer} impractical in dimensions
$d \ge 2$ the remainder of our experiments use the patched ViT and Fourier attention neural operators from 
Subsections \ref{subsec:patch_transformer} and  \ref{subsec:fourier_patch_transformer} respectively.

\begin{table}[htbp]
    \centering
    \begin{tabular}{|l|c|c|}
        \hline
        \cellcolor{gray!30}\textbf{Architecture} & \cellcolor{gray!30}\textbf{Number of Parameters}&\cellcolor{gray!30}\textbf{Darcy Flow}  \\
        \rowcolor{gray!30} \textbf{} &  &
        \textbf{Lognormal Input} \\
        \hhline{|===|}
        \textbf{Transformer NO} & $5.98 \cdot 10^5$ & $1.19\cdot 10^{-2}$\\
        \hline
        \textbf{FNO}& $5.70\cdot 10^6$ &$1.96\cdot 10^{-2}$\\
        \hline
        \textbf{AFNO}& $1.10\cdot 10^6$ &$2.57\cdot 10^{-2}$\\
        \hline
        \textbf{Galerkin Transformer}& $6.04\cdot 10^5$ &$7.30\cdot 10^{-2}$\\
        \hline
    \end{tabular}
    \caption{Median relative $L^2$ error for each architecture applied to a low resolution experimental setting.}
    \label{tab:2d_median_errors_low}
\end{table}

We now consider training two patched architectures on the Darcy Flow problem with piecewise
constant inputs and the Kolmogorov Flow problem; in both cases 
we use a $416\times 416$ resolution setting.  We compare cost versus accuracy for the different neural operators implemented, defining cost in terms of both number of parameters and number of FLOPS. In Table \ref{tab:evaluation_cost} we present the FLOPS for each method; details of the derivation are provided in Appendix \ref{Appendix:complexity}. We note that the parameter scaling for the transformer neural operator is given by $\cO\bigl(d_{\textrm{model}}^2 \bigr)$, while the scaling for the FNO and the Fourier attention neural operator is $\cO\bigl(d_{\textrm{model}}^2 \cdot k_{\textrm{max}}\bigr)$, where $k_{\textrm{max}}$ is the total number of Fourier modes used. The parameter scaling for the ViT neural operator is given by $\cO\bigl(d_{\textrm{model}} \cdot k_{\textrm{max}}+d_{\textrm{model}}^2\bigr)$. We note that the parameter scaling for the AFNO architecture is given by $\cO\bigl(Nd_{\textrm{model}}+d_{\textrm{model}}^2 \bigr)$, where the $Nd_{\textrm{model}}$ appears because of the learnable positional encoding. Further details are provided in Appendix \ref{Appendix:complexity}.

\begin{table}[htbp]
    \centering
    \begin{tabular}{|l|c|}
        \hline
        \cellcolor{gray!30}\textbf{Architecture} & \cellcolor{gray!30}\textbf{Scaling} \\
        \hhline{|==|}
        \textbf{FNO} &  $\cO\bigl(d_{\textrm{model}}N\log(\sqrt{N})+Nd^2_{\textrm{model}} \bigr)$\\
        \hline
        \textbf{AFNO} &  $\cO\bigl(d_{\textrm{model}}N\log(\sqrt{N})+Nd^2_{\textrm{model}}/b \bigr)$\\
        \hline
        \textbf{Transformer NO}& $\cO\bigl(d_{\textrm{model}}N^2+Nd^2_{\textrm{model}}\bigr)$\\
        \hline
        \textbf{ViT NO}&
        $
    \cO\Bigl(d_{\textrm{model}}N\bigl(P+\log(\sqrt{N/P})\bigr)+N\log(\sqrt{N}) +Nd^2_{\textrm{model}}\Bigr)
        $\\
        \hline
        \textbf{FANO}& $
        
        \cO\Bigl(d_{\textrm{model}}N\bigl(P+\log(\sqrt{N/P})\bigr)
        +N\log(N) +Nd^2_{\textrm{model}}\Bigr)
        $\\
        \hline
    \end{tabular}
    \caption{Evaluation cost (measured in FLOPS) for the three proposed transformer neural operators compared to the Fourier neural operator \citep{li2021fourier} and adaptive Fourier neural operator \citep{guibas2022efficient}. We include the scaling of the AFNO as implemented for our use case, i.e. not employing patching and not employing mode truncation; we further note that in the context of this architecture $b$ is a chosen integer hyperparameter. We give further details on the derivation of these scalings in Appendix \ref{Appendix:complexity}.}
    \label{tab:evaluation_cost}
\end{table}

Figures \ref{fig:parameter_accuracy} and \ref{fig:evaluation_cost_accuracy} display the cost-accuracy trade-off. Table \ref{tab:2d_median_errors_high} provides the test relative $L^2$ errors values, depicted in the cost-accuracy analysis figures, obtained with the models with the highest parameter count. 
Figure \ref{fig:parameter_accuracy} demonstrates the parameter-efficiency of the ViT neural operator 
in comparison with the FNO, AFNO and Fourier attention neural operator. Once cost is defined
in terms of FLOPS it is apparent that both patched transformer architectures outperform the FNO;
furthermore in the Kolmogorov flow problem, the FNO exhibits accuracy-saturation, caused by the finite data set, whereas the patched
architectures are less affected, in the ranges we test here. This suggests that the patched transformer
architectures use the information content in the data more efficiently. The AFNO is comparable in terms of parameter count to FANO but more costly than ViTNO. On the other hand, even with no patching, the AFNO is cheaper in terms of evaluation complexity. Both ViTNO and FANO outperform the AFNO in terms of parameter-accuracy tradeoff in the context of the Darcy flow problem. For the same problem ViTNO and FANO bring an improvement in accuracy, albeit at a higher computational complexity. In the context of the Kolmogorov flow problem, ViTNO and FANO are competitive with AFNO in parameter-accuracy tradeoff. However for this problem, the AFNO can achieve higher accuracy with a lower complexity. It is important to note that because of the learnable positional embedding, and potential strided convolutions (if using patching), the AFNO is not a neural operator. This means that applying this architecture to different resolutions would require retraining, thus significantly increasing computational complexity. This is a fundamental advantage of the other architectures as their parameters may be deployed across different resolutions for training and inference. We recall the example of GenCast \citep{Price2025}, where a medium-range weather forecasting model is trained using a neural operator at low resolutions, and then finetuned at higher resolutions, thus drastically reducing the cost for such a large scale application.

\begin{figure}[h!]
\centering
\includegraphics[width=\linewidth]{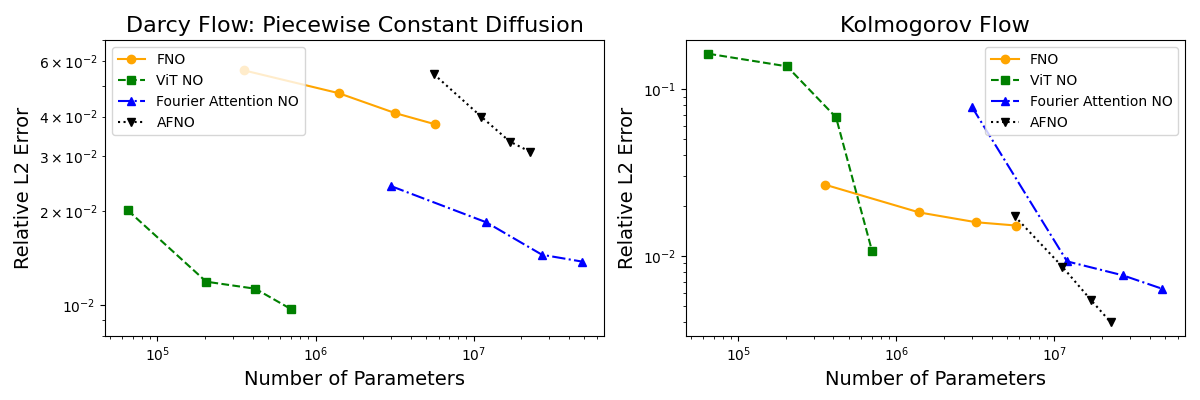}
\caption{The two panels display the test errors (in log-scale) against number of parameters for the architectures with the channel widths of 32, 64, 96 and 128.}
\label{fig:parameter_accuracy}
\end{figure}

\begin{figure}[h!]
\centering
\includegraphics[width=\linewidth]{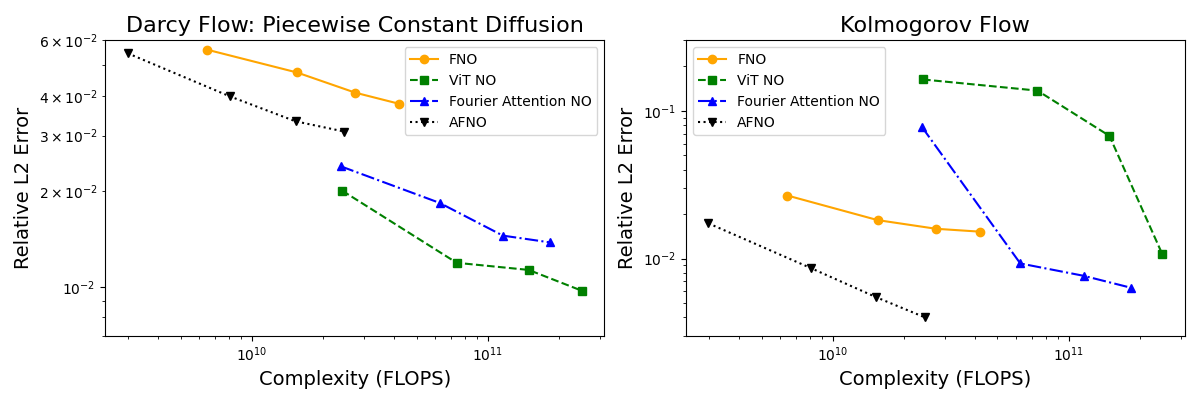}
\caption{The two panels display the test errors (in log-scale) against evaluation cost in FLOPS for the architectures with the channel widths of 32, 64, 96 and 128.}
\label{fig:evaluation_cost_accuracy}
\end{figure}

\begin{table}[htbp]
    \centering
    \begin{tabular}{|l|c|c|c|}
        \hline
        \cellcolor{gray!30}\textbf{Architecture} & \cellcolor{gray!30}\textbf{Darcy Flow} & \cellcolor{gray!30}\textbf{Darcy Flow} & \cellcolor{gray!30}\textbf{Kolmogorov Flow} \\
        \rowcolor{gray!30} \textbf{} & \textbf{Lognormal Input} & \textbf{Piecewise Constant Input} & \\
        \hhline{|====|}
        \textbf{ViT NO} & $7.68\cdot 10^{-3}$ &$9.72\cdot 10^{-3}$& $1.07\cdot 10^{-2}$\\
        \hline
        \textbf{FANO} & $8.39\cdot 10^{-3}$ &$1.38\cdot 10^{-2}$ & $6.34\cdot 10^{-3}$\\
        \hline
        \textbf{FNO} &$2.38\cdot 10^{-2}$ & $3.78\cdot 10^{-2}$ &$1.52\cdot 10^{-2}$ \\
        \hline
        \textbf{AFNO} &$2.66\cdot 10^{-2}$ & $3.09\cdot 10^{-2}$ &$3.99\cdot 10^{-3}$ \\
        \hline
    \end{tabular}
    \caption{Median relative $L^2$ error for each architecture applied to the different high resolution experimental settings.}
    \label{tab:2d_median_errors_high}
\end{table}

\subsubsection{Darcy Flow}
\label{subsubsec:darcy}
We consider the linear, second-order elliptic PDE defined on the unit square
\begin{equation}
\label{eq:Darcy}
\begin{alignedat}{2}
    -\nabla\cdot\bigl(a(x)\nabla u(x) \bigr) &= f(x), &\quad& x\in (0,1)^2,\\
    u(x) &= 0, &\quad& x\in\partial(0,1)^2,
\end{alignedat}
\tag{Darcy Flow}
\end{equation}
which is the steady-state of the 2d Darcy flow equation. The equation is equipped with Dirichlet boundary conditions and is well-posed when the diffusion coefficient $a\in L^\infty \bigl((0,1)^2;\R_+\bigr)$ and the forcing function is $f\in L^2 \bigl( (0,1)^2;\R \bigr)$. In this context, we will consider learning the nonlinear operator $\sG^\dagger: L^\infty \bigl((0,1)^2;\R_+\bigr)\to H^1_0 \bigl((0,1)^2;\R\bigr)$ defined as
\[
\sG^\dagger: a \mapsto u.
\]
We consider two different inputs. In both experimental settings the forcing function is chosen as $f\equiv 1$. 

In the first the data points $\{a^{(j)}\}_{j=1}^J$ are sampled from the probability measure 
\[
\mu_1\coloneqq ({T_1})_{\sharp}\cN\bigl(0,C \bigr),
\]
where $\cN\bigl(0,C \bigr)$ is a Gaussian measure with covariance operator $C$ defined as
\[
C = 12^2\bigl(-\Delta +6^2I \bigr)^{-2},
\]
where the Laplacian is equipped with zero Neumann boundary conditions and viewed as acting  between
spaces of spatially mean-zero functions, and $T_1(\cdot)=\exp(\cdot)$ is the exponential function. We hence refer to the data inputs $a^{(j)}$ as being \textit{lognormal}. In the second  
setting $\{a^{(j)}\}_{j=1}^J$ are sampled from the probability measure
\[
\mu_2\coloneqq ({T_2})_{\sharp}\cN\bigl(0,C \bigr),
\]
where the covariance operator is defined as before and
\[
T_2(x) = \begin{cases}
    12, \quad x\geq 0, \\
    3, \quad x<0.
\end{cases}
\]
We hence refer to the data inputs $a^{(j)}$ as being \textit{piecewise constant}.

\begin{figure}[h!]
\centering
\includegraphics[width=\linewidth]{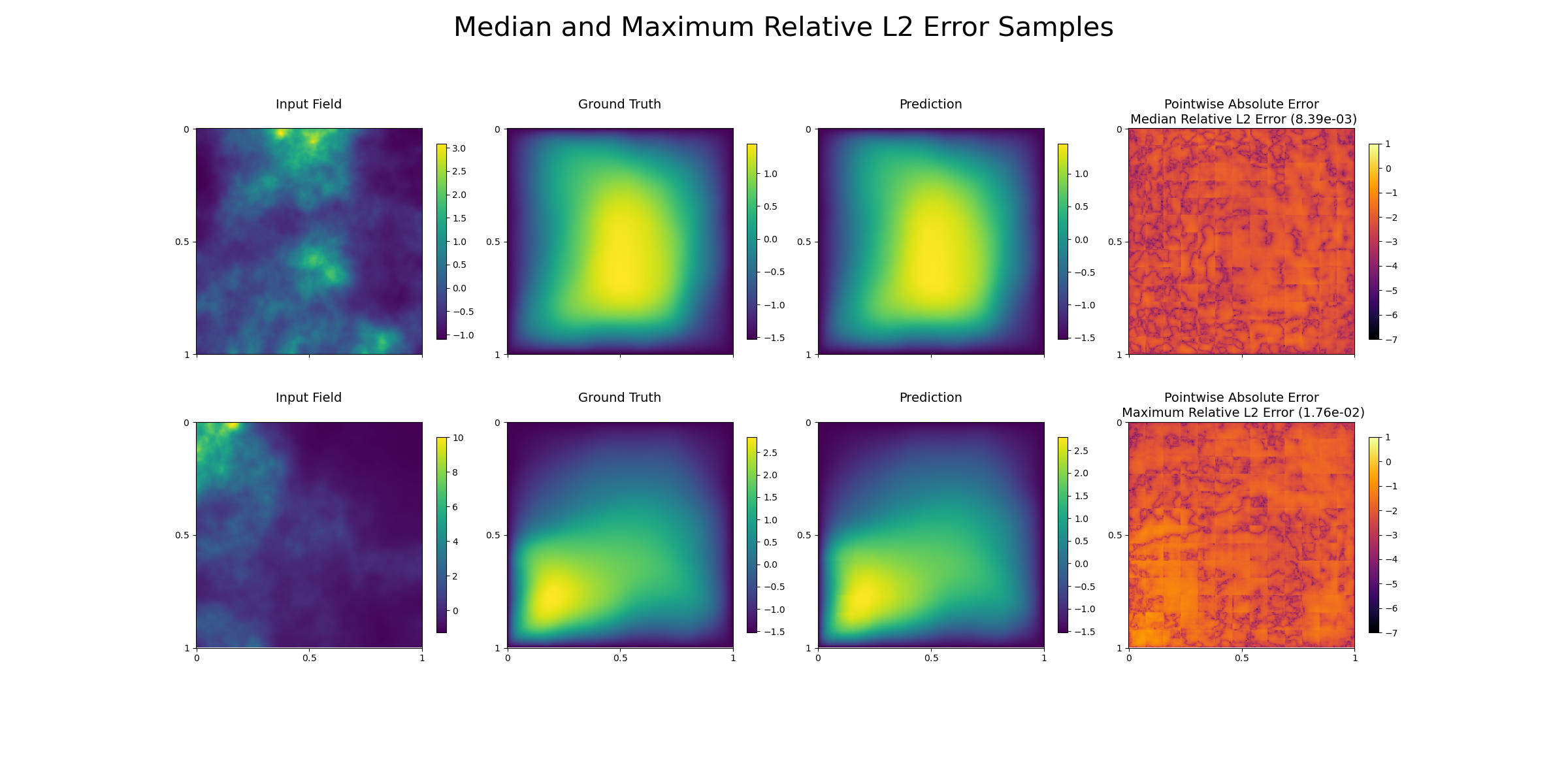}
\caption{The panel displays the result of the application of the Fourier attention neural operator on the Darcy flow experiment with lognormal diffusion for the median and maximum relative $L^2$ error samples. The first row shows the sample from the test set of the same resolution as the training set yielding the median relative $L^2$ error. On the other hand, the second row displays the sample yielding the maximum relative $L^2$ error. The first column of the panel displays the input diffusion coefficient $a(x)$ of the relevant sample. The second column shows the true solution $u(x;a)$, while the third column displays the the predicted solution. The last column displays the pointwise absolute error (in log-scale) between the prediction and truth.}
\label{fig:FANO_lognormal_samples}
\end{figure}

In the context of the lognormal experimental setting, we investigate a scenario with low resolution training samples, for which pointwise attention and hence the transformer neural operator
from Subsection \ref{subsec:vanilla_transformer} is suitable, and a high resolution scenario where patching becomes necessary for application of the patched-based architectures of 
Subsections \ref{subsec:patch_transformer} and \ref{subsec:fourier_patch_transformer}. 

For the low resolution setting, we train the FNO, AFNO, Galerkin transformer and TNO architectures using $3600$ independent samples of resolution $64\times 64$, we use $200$ samples for validation and $200$ samples for testing. The Fourier neural operator is parametrized by $12$ Fourier modes in each dimension and channel width of $128$; the model is trained using a batch size of $8$ and learning rate of $10^{-4}$. On the other hand, the transformer neural operator is parametrized by $128$ channels in the encoder and is trained using a batch size of $2$ and learning rate of $10^{-3}$. The AFNO is instantiated with $4$ layers of $128$ hidden channels, and does not employ patching; it is trained using a learning rate of $10^{-3}$ and batch size of $1$. The Galerkin transformer uses the Fourier type attention from \cite{cao2021choose}, with quadratic complexity, making it prohibitive for higher resolutions. It uses $4$ layers of $128$ hidden channels and is trained using a learning rate of $10^{-3}$ and batch size of $2$. On this problem we train all neural operators using the relative $H^1$  loss.

For the high resolution setting, we train the AFNO, FNO, ViT neural operator and Fourier attention neural operator architectures using $3600$ independent samples of resolution $416\times 416$; we use $200$ samples for validation and $200$ samples for testing. The Fourier neural operator is again parametrized by $12$ Fourier modes in each dimension and channel width of $128$; the model is trained using a batch size of $8$ and learning rate of $10^{-4}$. The AFNO is instantiated with $4$ layers of $128$ hidden channels, and does not employ patching; it is trained using a learning rate of $10^{-3}$ and batch size of $1$. Both the ViT neural operator and Fourier attention neural operator are parametrized by $128$ channels in the transformer encoder. The integral operator in the ViT NO is parametrized by $12$ Fourier modes in each dimension while the integral operators appearing in the Fourier attention neural operator are parametrized by $9$ Fourier modes in each dimension. Both architectures employ a final spectral convolution layer (see Remark \ref{rm:smoothing}) that is parametrized by $64$ Fourier modes in each dimension. Both of these neural operators are trained using a batch size of $1$ and learning rate of $10^{-3}$. On this problem we train all the neural operators with the relative $H^1$  loss. Figure \ref{fig:FANO_lognormal_samples} displays the high resolution test results obtained by applying the Fourier attention neural operator. 

\begin{figure}[h!]
\centering
\includegraphics[width=\linewidth]{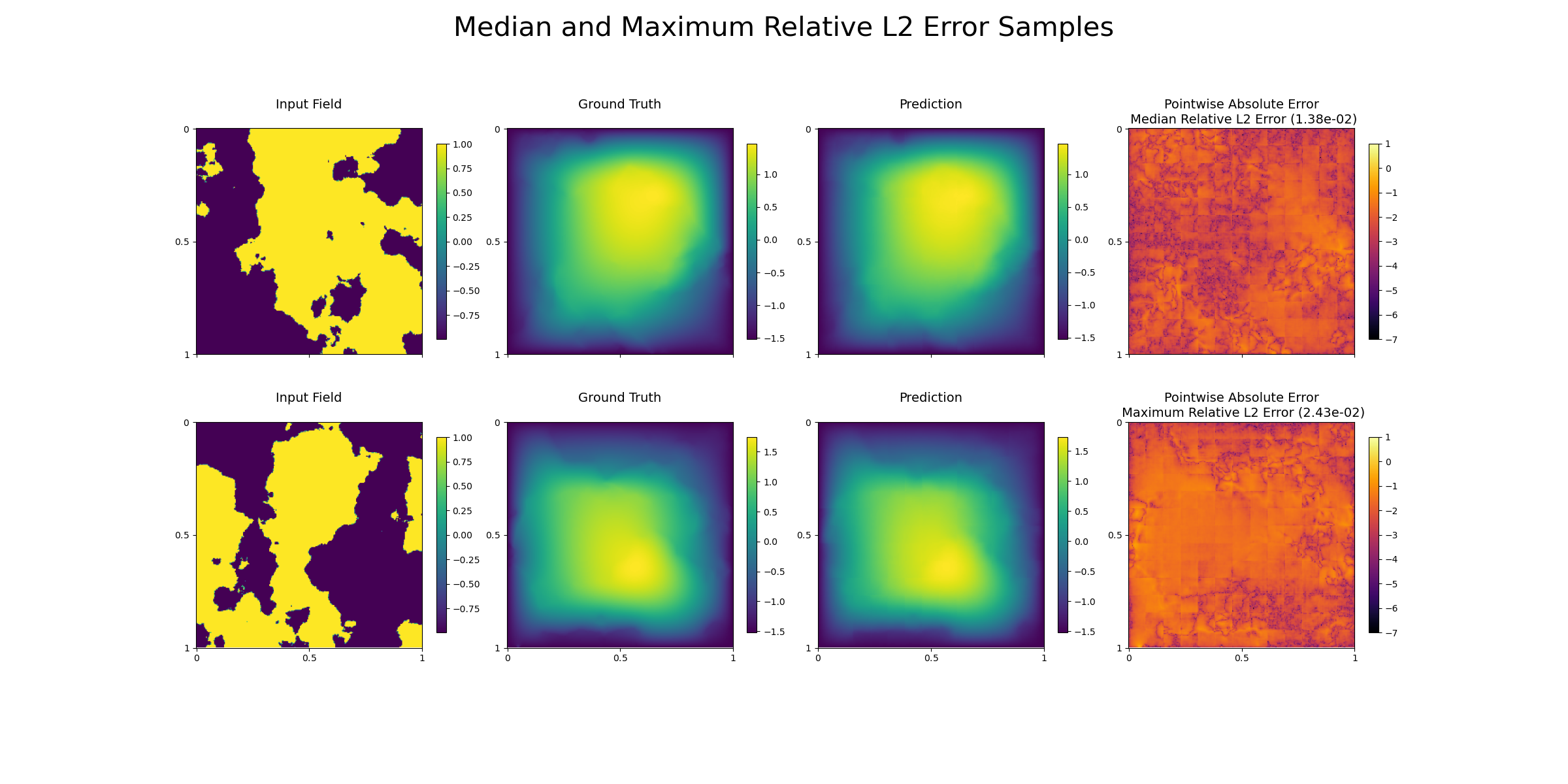}
\caption{The panel displays the result of the application of the Fourier attention neural operator on the Darcy flow experiment with piecewise constant diffusion for the median and maximum relative $L^2$ error samples. The first row displays the sample from the test set of the same resolution as the training set yielding the median relative $L^2$ error. On the other hand, the second row displays the sample yielding the maximum relative $L^2$ error. The first column of the panel displays the input diffusion coefficient $a(x)$ of the relevant sample. The second column shows the true solution $u(x;a)$, while the third column displays the the predicted solution. The last column displays the pointwise absolute error (in log-scale) between the prediction and truth.}
\label{fig:FANO_piecewise_samples}
\end{figure}

\begin{figure}[h!]
\centering
\includegraphics[width=\linewidth]{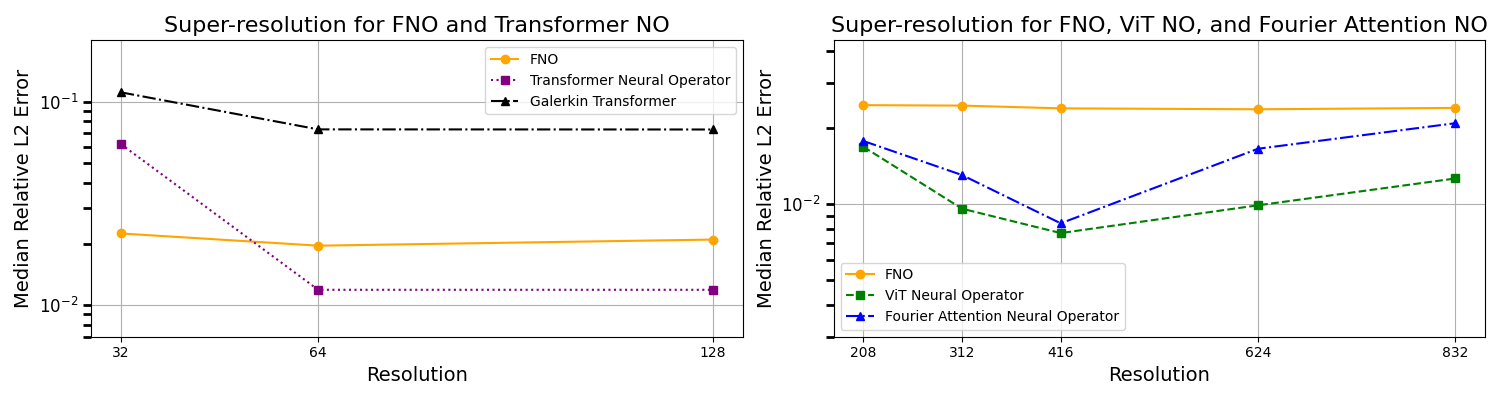}
\caption{The panel displays the performance in relative $L^2$ errors (in log-scale) of the various architectures when applied to Darcy flow with lognormal diffusion test sets of different resolutions at inference time (without retraining). The left panel concerns FNO, the transformer neural operator and the Galerkin transformer when trained on data of resolution $64\times 64$. On the other hand, the right panel concerns the FNO, ViT neural operator and Fourier attention neural operator architectures trained on data of resolution $416\times 416$.}
\label{fig:superresolution_lognormal}
\end{figure}

In the context of the piecewise constant experimental setting, we train the AFNO, FNO, ViT neural operator and Fourier attention neural operator architectures using $3600$ independent samples of resolution $416\times 416$, we use $200$ samples for validation and $200$ samples for testing. The Fourier neural operator is again parametrized by $12$ Fourier modes in each dimension and channel width of $128$; the model is trained using a batch size of $8$ and learning rate of $10^{-4}$. The AFNO is instantiated with $4$ layers of $128$ hidden channels, and does not employ patching; it is trained using a learning rate of $10^{-3}$ and batch size of $1$. Both the ViT neural operator and Fourier attention neural operator are parametrized by $128$ channels in the transformer encoder. The integral operator in the ViT NO is parametrized by $12$ Fourier modes in each dimension while the integral operators appearing in the Fourier attention neural operator are parametrized by $9$ Fourier modes in each dimension. Both architectures employ a final spectral convolution layer (see Remark \ref{rm:smoothing}) that is parametrized by $64$ Fourier modes in each dimension. Both of these neural operators are trained using a batch size of $1$ and learning rate of $10^{-3}$. On this problem we train all the neural operators with the relative $H^1$ loss. Figure \ref{fig:FANO_piecewise_samples} displays the high resolution test results obtained by applying the Fourier attention neural operator. 

\begin{figure}[h!]
\centering
\includegraphics[width=0.8\linewidth]{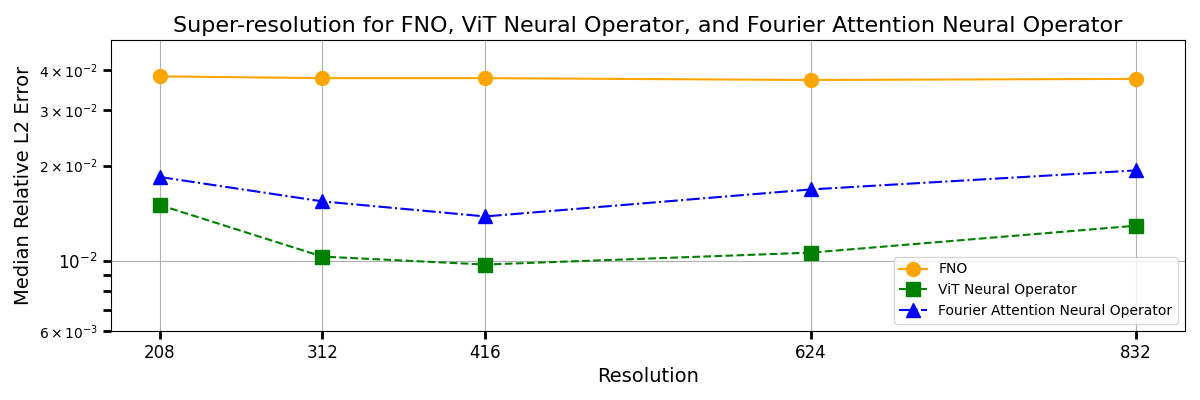}
\caption{The panel displays the performance in relative $L^2$ errors (in log-scale) of the FNO, ViT neural operator and Fourier attention neural operator when applied to Darcy flow with piecewise constant diffusion test sets of different resolutions at inference time (without retraining). All architectures are trained on data that is of resolution $416\times 416$.}
\label{fig:superresolution_piecewise}
\end{figure}

In Figures \ref{fig:superresolution_lognormal} and \ref{fig:superresolution_piecewise} we display the results of applying the neural operator architectures for the lognormal and piecewise input settings, respectively, to test samples of different resolutions than the one used for training. The results demonstrate the zero-shot generalization to different resolutions capability of the transformer neural operator architectures, which does not require retraining of the models. In the lognormal setting the
FNO exhibits invariance to discretization that is more stable to changing resolution
than are the ViT neural operator and the Fourier attention neural operator. 

\subsubsection{Kolmogorov Flow}
\label{subsubsec:NS}
We consider the two-dimensional Navier-Stokes equation for a viscous, incompressible fluid,
\begin{equation}
\label{eq:KF}
\begin{alignedat}{2}
    \frac{\partial u}{\partial t} + u\cdot\nabla u + \nabla p &= \nu \Delta u + \sin(ny)\hat{x}, &\quad& (x,t)\in [0,2\pi]^2\times(0,\infty),\\
    \nabla\cdot u &= 0, &\quad& (x,t)\in [0,2\pi]^2\times[0,\infty),\\
    u(\cdot,0)&=u_0, &\quad& x\in [0,2\pi]^2,
\end{alignedat}
\tag{KF}
\end{equation}
where $u$ denotes the velocity, $p$ the pressure and $\nu$ the kinematic viscosity. The particular choice of forcing function $\sin (ny)$ leads to a particular example of a Kolmogorov flow. We equip the domain $[0,2\pi]$ with periodic boundary conditions. We assume \[u_0 \in \cU \coloneqq \Bigl\{u\in \dot{L}^2_{\textrm{per}}\bigl( [0,2\pi]^2;\R^2\bigr) :\nabla\cdot u = 0\Bigr\}.\]
The dot on $L^2$ denotes the
space of spatially mean-zero functions.
The vorticity is defined as $w = (\nabla \times u)\widehat{z}$ and the stream function $f$ as the solution to the Poisson equation $ - \Delta f = w$. Existence of the semigroup $S_t : \cU \to \cU$ is shown in \citet[Theorem 2.1]{temam2012infinite}.
We generate the data by solving \eqref{eq:KF} in vorticity-streamfunction form by applying the pseudo-spectral split step method from \citet{Chandler_Kerswell_2013}. In our experimental set-up $n=4$ and $\nu=1/70$ are chosen. Random initial conditions are sampled from the Gaussian measure $\cN(0,C)$ where the covariance operator is given by
\[
C = 7^{3}\bigl(-\Delta +49I\bigr)^{-5},
\]
where the Laplacian is equipped with periodic boundary conditions on $[0,2\pi]^2$,
and viewed as acting  between spaces of spatially mean-zero functions. The input output-data pairs for the Kolmogorov flow experiment are given by $$\Bigl\{\omega\bigl(x,T;\omega_0^{(j)}\bigr);\omega\bigl(x,T+\Delta t;\omega_0^{(j)}\bigr)\Bigr\}_{j=1}^J,$$
for $\omega_0^{(j)}\sim \cN(0,C)$ being i.i.d. samples. Indeed, given any $r>0$ we aim to approximate the nonlinear operator $\sG^\dagger: H^r\bigl([0,2\pi]^2;\R \bigr)\to H^r\bigl([0,2\pi]^2;\R \bigr)$ defined by
\[
\sG^\dagger: \omega\bigl(x,T;\omega_0\bigr) \mapsto \omega\bigl(x,T+\Delta t;\omega_0\bigr).
\]

\begin{figure}[h!]
\centering
\includegraphics[width=\linewidth]{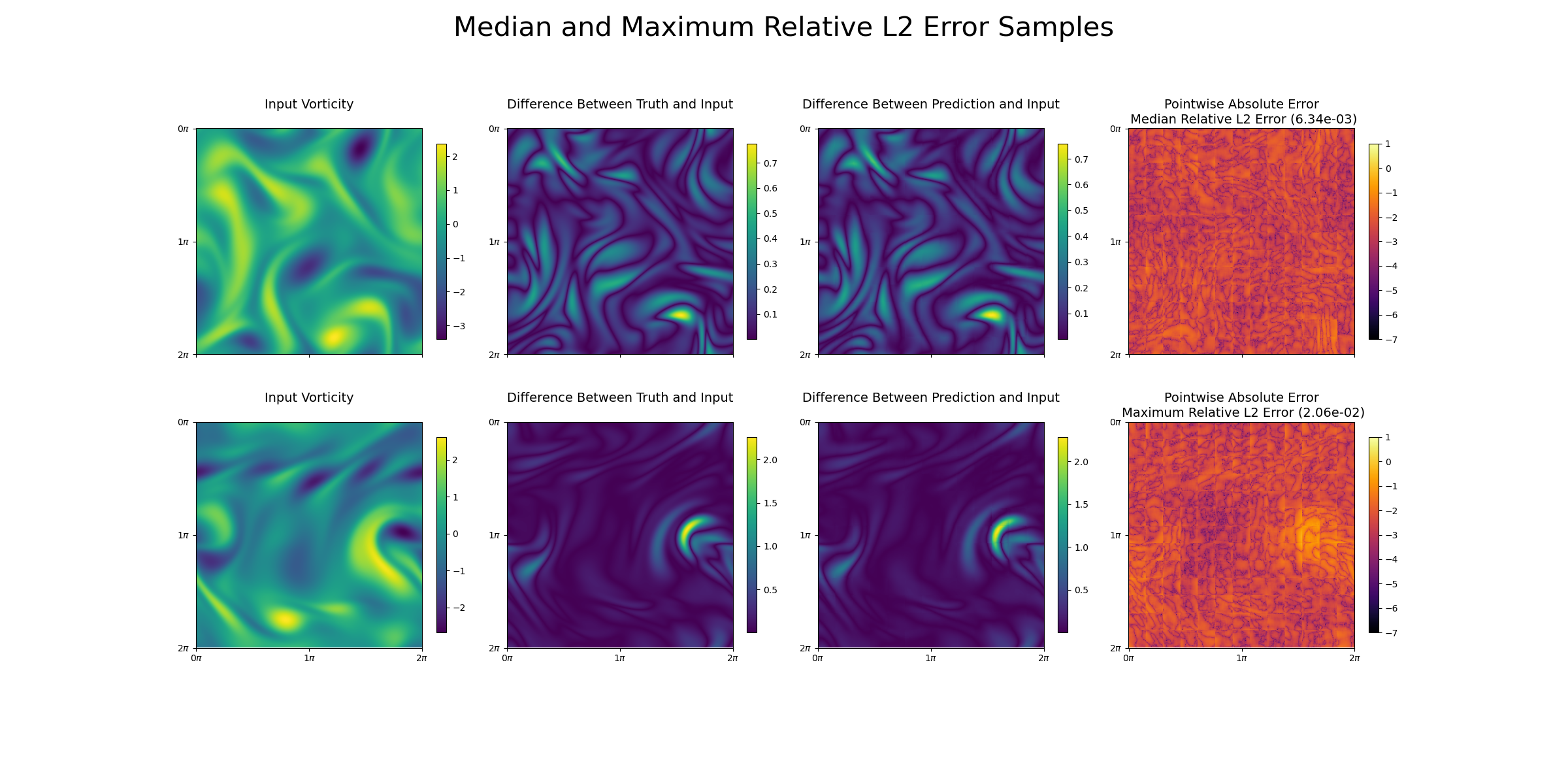}
\caption{The panel displays the result of the application of the Fourier attention neural operator on the Kolmogorov flow experiment for the median and maximum relative $L^2$ error samples. The first row displays the sample from the test set of the same resolution as the training set yielding the median relative $L^2$ error. The second row on the other hand displays the sample yielding the maximum relative $L^2$ error. The first column of the panel displays the input vorticity $w(x,T)$ of the relevant sample. The second column shows the absolute difference between the true vorticity $w(x,T+\Delta t)$ and the input vorticity $w(x,T)$, while the third column displays the absolute difference between the predicted vorticity at time $T+\Delta t$ and the input $w(x,T)$. The last column shows the pointwise absolute error (in log-scale) between the prediction and truth.}
\label{fig:FANO_NS_samples}
\end{figure}

In our experimental setting, we set $T=11$ and $\Delta t =0.1$. We train the AFNO, FNO, ViT neural operator and Fourier attention neural operator architectures on this problem using $9000$ independent samples of resolution $416\times 416$, $500$ samples for validation and $500$ samples for testing. The Fourier neural operator is parametrized by $12$ Fourier modes in each dimension and channel width of $128$; the model is trained using a batch size of $4$ and learning rate of $10^{-3}$. The AFNO is instantiated with $4$ layers of $128$ hidden channels, and does not employ patching; it is trained using a learning rate of $10^{-3}$ and batch size of $1$. Both the ViT neural operator and Fourier attention neural operator are parametrized by $128$ channels in the transformer encoder. The integral operator in the ViT NO is parametrized by $12$ Fourier modes in each dimension while the integral operators appearing in the Fourier attention neural operator are parametrized by $9$ Fourier modes in each dimension. Both architectures employ a final spectral convolution layer (Remark \ref{rm:smoothing}) that is parametrized by $64$ Fourier modes in each dimension. Both of these neural operators are trained using a batch size of $2$ and learning rate of $10^{-3}$. On this problem we train all the neural operators with the relative $H^1$  loss. Figure \ref{fig:FANO_NS_samples} displays the test results using the Fourier attention neural operators. 

\begin{figure}[h!]
\centering
\includegraphics[width=0.8\linewidth]{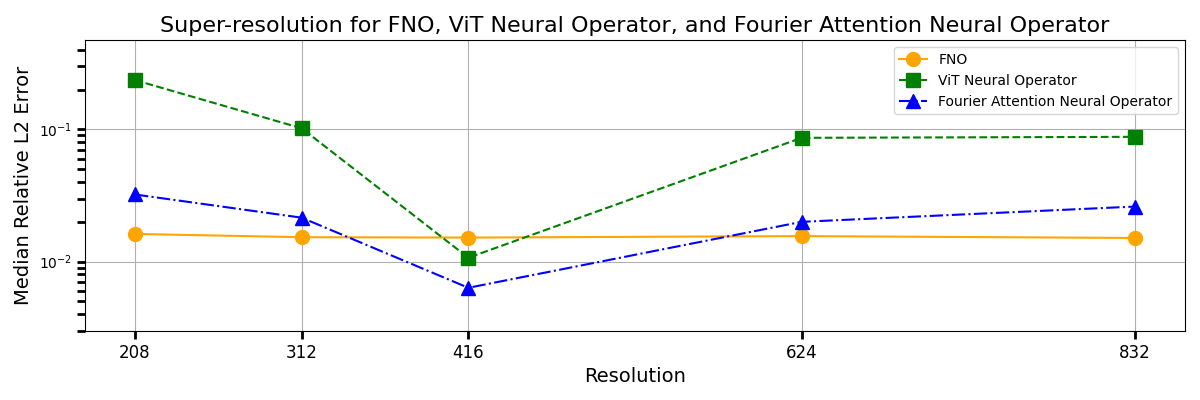}
\caption{The panel displays the performance in relative $L^2$ errors (in log-scale) of the FNO, ViT neural operator and Fourier attention neural operator when applied to Kolomgorov flow test sets of different resolutions at inference time (without retraining). All architectures are trained on data that is of resolution $416\times 416$.}
\label{fig:superresolution_KF}
\end{figure}

In Figure \ref{fig:superresolution_KF} we display the results of applying the three architectures to test samples of resolutions different to the $416\times 416$ training resolution. The results demonstrate the zero-shot generalization capability of the three neural operator architectures to different discretizations, which does not require retraining of the models. Again we note that FNO exhibits invariance to discretization that is more stable to changing resolution than the methods introduced here; but for all the neural
operators it is nonetheless clear that intrinsic properties of the continuum limit are learned and may
be transferred between discretizations. 

A few additional considerations are in order. The mode truncation employed in the Fourier neural operator make the architecture computationally efficient but yields an over-smoothing effect that is not suitable for problems at high resolutions, where high frequency detail is prominent. Attention-based neural operators offer a possible solution to this issue, as demonstrated by the performance of the transformer neural operators proposed when applied to the PDE operator learning problems considered, which involve rough diffusion coefficients and rough initial conditions. On the other hand, patching leads to more efficient attention-based neural operators, but also introduces possible discontinuities at patch-intersections. This issue has been observed to be resolved in the large data regime. Furthermore, it is possible to employ problem specific smoothing operators as the one introduced in Remark \ref{rm:smoothing}. In \Cref{fig:FANO_NS_samples_smoothing} we display the result of applying a smoothing operator as the last layer in the training of the FANO architecture in the context of the Kolmogorov flow problem. The smoothing via the inverse of the negative Laplacian is applied using the Fourier basis, given the periodic boundary conditions. Here, we make the specific choice $\epsilon=10^{-3}$ and $\alpha=1.001$, in the context of Remark \ref{rm:smoothing}. Such an approach is effective at reducing the discontinuities introduced by patching and reducing the evaluation error, but is problem dependent and requires knowledge of boundary conditions. Such discontinuity issues constitute an inherent limitation of patching and highlights the need for more suitable efficient attention mechanisms.

\begin{figure}[h!]
\centering
\includegraphics[width=\linewidth]{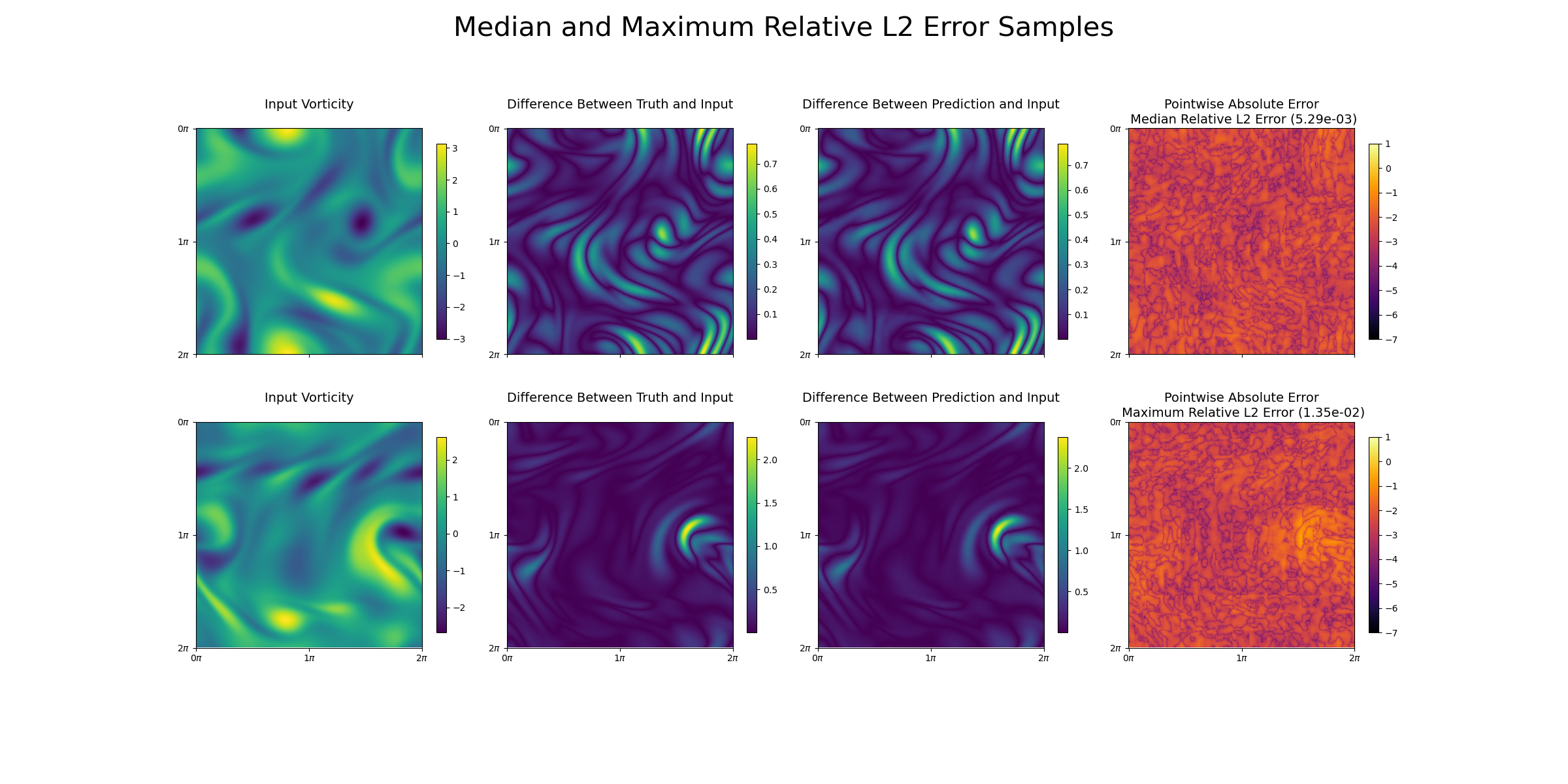}
\caption{The panel displays the result of the application of the Fourier attention neural operator employing a final smoothing layer on the Kolmogorov flow experiment for the median and maximum relative $L^2$ error samples. The first row displays the sample from the test set of the same resolution as the training set yielding the median relative $L^2$ error. The second row on the other hand displays the sample yielding the maximum relative $L^2$ error. The first column of the panel displays the input vorticity $w(x,T)$ of the relevant sample. The second column shows the absolute difference between the true vorticity $w(x,T+\Delta t)$ and the input vorticity $w(x,T)$, while the third column displays the absolute difference between the predicted vorticity at time $T+\Delta t$ and the input $w(x,T)$. The last column shows the pointwise absolute error (in log-scale) between the prediction and truth.}
\label{fig:FANO_NS_samples_smoothing}
\end{figure}

Furthermore, it is observed that in some experimental settings the patching-based neural operators exhibit worse stability to changing resolution than FNO. This is likely due to the effect of patching which may then be further amplified for FANO due to the application of a higher number of FFT(s) to non-periodic domains. In the context of the FANO architecture, such issues can be ameliorated using periodic extensions for Fourier-based parametrizations. In \Cref{fig:superresolution_KF_patching} we demonstrate the discretization invariance of the standard FANO compared to a FANO architecture that involves building a periodic extension of each patch. The Fourier integral operator is applied to each extended patch and the extension is then removed from the output. Such a technique leads to an increase in complexity and runtime but yields more stable discretization invariance. Other remedies could potentially include other kinds of parametrizations for the nonlocal operations, or fine-tuning at different resolutions.

\begin{figure}[h!]
\centering
\includegraphics[width=0.8\linewidth]{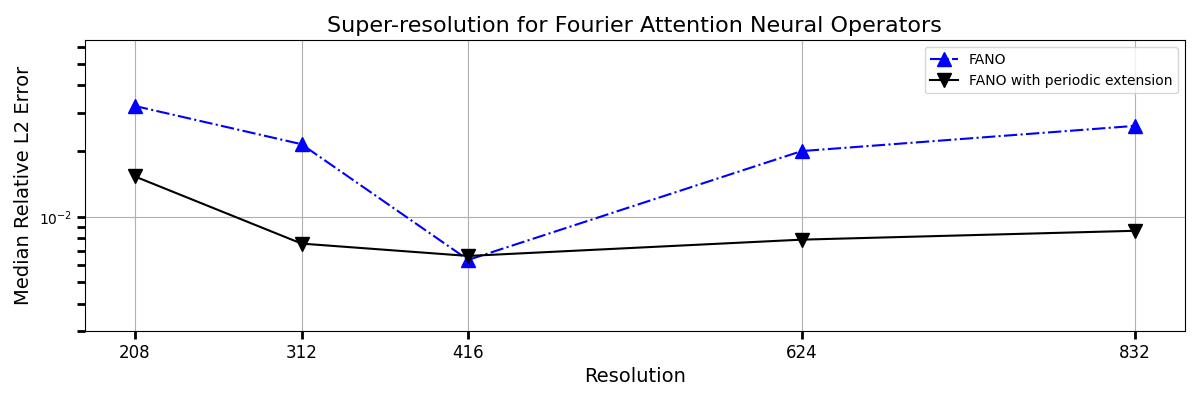}
\caption{The panel displays the performance in relative $L^2$ errors (in log-scale) of the Fourier attention neural operator and Fourier attention neural operator with periodic extensions of patches when applied to Kolomgorov flow test sets of different resolutions at inference time (without retraining). All architectures are trained on data that is of resolution $416\times 416$.}
\label{fig:superresolution_KF_patching}
\end{figure}


\section{Conclusions}

In this work we have introduced a continuum formulation of the attention mechanism from the seminal work \citet{vaswani2017attention}. Our continuum formulation can be used to design transformer neural operators. Indeed, this continuum formulation leads to discretization invariant implementations of attention and schemes that exhibit zero-shot generalization to different resolutions. In the first result of its kind for transformers, with a slight modification to the architecture implemented in practice, the resulting neural operator architecture mapping between infinite-dimensional spaces of functions is shown to be a universal approximator of continuous functions and functions of Sobolev regularity defined over a compact domain. We extend the continuum formulation to patched attention, which makes it possible to design efficient transformer neural operators. To this end, we introduce a neural operator analogue of the vision transformer \citet{dosovitskyi2021image} and a more expressive architecture, the Fourier attention neural operator. Through a cost-accuracy analysis, we demonstrate the power of the methodology in the context of a range of operator learning problems. In the following we highlight potential avenues for further work.

\begin{enumerate}
    \item It is of interest to extend the analysis of \Cref{thm:self_attention_limit,thm:cross_attention_limit} to quantify the norm-dependent rates of convergence of discretized attention to its continuum limit. Furthermore, obtaining the results over i.i.d. samples from arbitrary strictly positive probability distributions would be of interest.

    \item The Fourier attention neural operator architecture uses operator parametrizations for $\sQ,\sK,\sV$. Other design choices for these operators not employing the Fourier basis are possible, for example using the wavelet basis as in \cite{TRIPURA2023115783} or using a multigrid approach as in \cite{he2024mgno}. An investigation of other attention neural operator architectures arising from this consideration would be of interest.

    \item It is of interest to investigate the discontinuity 
    issues arising from patching. In particular, it is desirable to design schemes that impose continuity amongst these throughout the architecture. 
    
    \item Deriving universal approximation theorems for the patch-based transformer neural operators would be of great theoretical interest.
\end{enumerate}


\acks{

AMS is supported by a Department of Defense Vannevar Bush  Faculty Fellowship; this
funding also supports EC and partly supported MEL. In addition, 
AMS and EC are supported by the SciAI Center, funded by the Office of Naval Research 
(ONR), under Grant Number N00014-23-1-2729. NBK is grateful to the Nvidia Corporation for support through full-time employment. MEL is supported by funding from the Eric and Wendy Schmidt Center at the Broad Institute of MIT and Harvard.

We thank Théo Bourdais, Matthieu Darcy, Miguel Liu-Schiaffini, Georgia Gkioxari, Sabera Talukder, Zihui Wu, and Zongyi Li for helpful discussions and feedback.  
}

\appendix
\section{Proofs of Approximation Theorems}
\label{Appendix:A}

\subsection{Proof of Self-Attention Approximation Theorem}
\label{Appendix:self}
We introduce the following auxiliary result that we will apply to prove Theorem \ref{thm:self_attention_limit}.

\begin{lemma}
\label{lemma:expectation_of_l2}
Let \(D \subset \R^d\) be a bounded open set and let \(\{y_j\}_{j=1}^N \sim U(\Bar{D})\)
be an i.i.d. sequence. If \(f \in C(\Bar{D} \times \Bar{D})\)
then, with the expectation taken over the data \(\{y_j\}_{j=1}^N\),
\[ \E \left \| \int_D f(\cdot,y) \: \mathrm{d}y - \frac{|D|}{N} \sum_{j=1}^N f(\cdot,y_j) \right \|_{L^2} \leq |D|^{3/2} \|f\|_{C(\Bar{D} \times \Bar{D})} N^{-1/2}.\]
\end{lemma}
\begin{proof}
Notice that, for any \(x \in \Bar{D}\),
\[\int_D f(x,y) \: \mathrm{d}y  = |D| \E_{y \sim U(\Bar{D})} \big[ f(x,y) \big ].\]
Hence, by expanding the square, we can observe that
\[\E \left ( \int_D f(x,y) \: \mathrm{d}y - \frac{|D|}{N} \sum_{j=1}^N f(x,y_j) \right )^2 \leq \frac{|D|^2}{N} \E_{y \sim U(\Bar{D})} \big [ f(x,y)^2 \big ].\]
Bounding $f$ in $L^\infty$, and noting that it is a continuous function, gives
\begin{align*}
\E \left \| \int_D f(\cdot,y) \: \mathrm{d}y - \frac{|D|}{N} \sum_{j=1}^N f(\cdot,y_j) \right \|^2_{L^2} &\leq \frac{|D|^2}{N} \int_D \E_{y \sim U(\Bar{D})} \big [ f(x,y)^2 \big ] \: \mathrm{d}x \\
&\leq |D|^3 \|f\|^2_{C(\Bar{D} \times \Bar{D})} N^{-1} .
\end{align*}
The result follows by Jensen's inequality. 
\end{proof}
As part of the proof of Theorem \ref{thm:self_attention_limit}, we first show that the the self-attention operator is a mapping of the form $\sA: L^\infty (D;\R^{d_u}) \to L^\infty (D; \R^{d_V})$ and hence $\sA:C (\Bar{D};\R^{d_u}) \to C (\Bar{D}; \R^{d_V})$.
\begin{lemma}
\label{lemma:boundedness}
    Let $u\in L^\infty(D;\R^{d_u})$, then it holds that $\sA(u)\in L^\infty(D;\R^{d_V})$. Furthermore, for $u \in C (\Bar{D};\R^{d_u}) $ it holds that $\sA(u) \in C (\Bar{D};\R^{d_V})$.
\end{lemma}
\begin{proof}
    The result follows readily from the fact that $p$ is a probability density function. For completeness, we show the boundedness of the normalization constant of the function $p$. Indeed we note that,
    \begin{align*}
    \int_D \mathrm{exp} \big ( \langle Q u(x), K u(s) \rangle \big ) \: \mathrm{d}s &\leq \int_D \mathrm{exp} \big ( |Q| |K| |u(x)| |u(y)| \big ) \mathrm{d}s \\
    &\leq \int_D \mathrm{exp} \bigg ( \frac{1}{2} |Q|^2 |K|^2 |u(x)|^2 \bigg ) \mathrm{exp} \bigg ( \frac{1}{2} |u(y)|^2 \bigg ) \mathrm{d}s \\
    &\leq |D| \mathrm{exp} \big ( R \|u\|_{L^\infty}^2 \big ) 
\end{align*}
for some constant \(R  > 0\). From definition, it is clear that 
\[\int_D p(y; u, x) \: \mathrm{d}y = 1,\] 
hence \(p\) is indeed a valid probability density function. It hence follows that
\[\|A(u)\|_{L^\infty} \leq |V|\|u\|_{L^\infty}\]
hence \(A : L^\infty (D;\R^{d_u}) \to L^\infty (D; \R^{d_V})\). Similarly, it follows that
\(A : C (\Bar{D};\R^{d_u}) \to C (\Bar{D}; \R^{d_V})\). We note that the continuity of $u$ is preserved by $\sA$ under the continuity of the inner product in its first argument and the continuity of the exponential.
\end{proof}
We are now ready to establish the full result of Theorem \ref{thm:self_attention_limit} which states that 
\[ \lim_{N \to \infty} \sup_{u \in B} \E \left \| \sA(u) - \frac{\sum_{j=1}^N \mathrm{exp} \big ( \langle Qu(\cdot), Ku(y_j)  \rangle \big ) V u(y_j)}{\sum_{\ell=1}^N \mathrm{exp} \big ( \langle Qu(\cdot), Ku(y_\ell)  \rangle \big )} \right \|_{C(\Bar{D};\R^{d_V})}  = 0,\]
with the expectation taken over i.i.d. sequences \(\{y_j\}_{j=1}^N \sim \unif(\Bar{D})\).
\begin{proof}[Proof of Theorem \ref{thm:self_attention_limit}]
Before commencing the proof, we first establish useful shorthand notation. Letting \(u \in B\) we define,
\[f(x,y;u) \coloneqq \mathrm{exp} \big ( \langle Q u(x), K u(y) \rangle \big ), \quad g_l(x,y;u) \coloneqq \mathrm{exp} \big ( \langle Q u(x), K u(y) \rangle \big ) \big( Vu(y) \big )_l  \]
for all \(x,y \in \Bar{D}\) and \(l \in \range{1}{d_V}\). For $u\in B$ we also define
\[a(x;u) \coloneqq \int_D f(x,y;u) \: \mathrm{d}y, \qquad a^{(N)}(x;u) \coloneqq \frac{|D|}{N} \sum_{j=1}^N f(x,y_j;u),\]
and
\[b_l(x;u) \coloneqq \int_D g_l(x,y;u) \: \mathrm{d}y, \qquad b^{(N)}_l(x;u) \coloneqq \frac{|D|}{N} \sum_{j=1}^N g_l(x,y_j;u)\]
for any $x\in\Bar{D}$ and any \(l \in \range{1}{d_V}\). We set \(\alpha_l = b_l/a\) and \(\alpha_l^{(N)} = b_l^{(N)}/a^{(N)}\). Given this notation, to prove the result we must show
\begin{equation}
\label{eq:aim}
\lim_{N \to \infty} \sup_{u \in B} \E \left \| \alpha_l(\cdot\,;u) - \alpha_l^{(N)}(\cdot\,;u) \right\|_{C(\Bar{D};\R)}=0,
\end{equation}
for any $l\in\range{1}{d_V}$. We divide the proof in the following key steps. We first show that for fixed $u\in B$, it holds that
\begin{equation}
\label{eq:step1}
    \lim_{N \to \infty} \E \|\alpha_l(\cdot\,;u) - \alpha_l^{(N)}(\cdot\,;u)\|_{L^2} = 0,
\end{equation}
for any $l\in\range{1}{d_V}$. By using compactness and applying the Arzel\`a-Ascoli theorem, we then show that
\begin{equation}
\label{eq:step2}
\lim_{N \to \infty} \E \|\alpha_l(\cdot\,;u) - \alpha_l^{(N)}(\cdot\,;u)\|_{C(\Bar{D})} = 0,
\end{equation}
for any $l\in\range{1}{d_V}$. We then use the compactness of $B$ and continuity of the mappings $u\mapsto \alpha_\ell(\cdot\, ;u)$ and $u\mapsto \alpha^{(N)}_\ell(\cdot\, ;u)$ to deduce the result given by \eqref{eq:aim}.
We now focus on establishing \eqref{eq:step1}. Since \(B\) is bounded, there exists a constant \(M > 0\) such that,
\[\sup_{u \in B} \|u\|_{L^\infty} \leq M.\]
Therefore, we have
\[\sup_{u \in B} \max \{ \|f\|, \|g_1\|, \dots, \|g_m\| \} \leq \max \left \{ \text{exp} \big( RM^2 \big ), M |V| \text{exp} \big (RM^2 \big ) \right \} \coloneqq J,\]
where \(\|\cdot\| = \|\cdot\|_{C(\Bar{D} \times \Bar{D})}\).
Clearly,
\[\sup_{u \in B} \max \{ \|a\|_{L^2}, \|a^{(N)}\|_{L^2} \} \leq |D|^{3/2} J\]
Notice that, since \(|\langle Q u(x), K u(s) \rangle| \leq RM^2\), we have \(f(x,y) \geq \text{exp} \big ( - R M^2 \big )\).
It follows that
\[\inf_{u \in B} \min \{ \|a\|_{L^2}, \|a^{(N)}\|_{L^2} \} \geq |D|^{3/2} \text{exp} \big ( - R M^2 \big ) \coloneqq |D|^{3/2} I > 0.\]
Similarly, 
\[\sup_{u \in B} \max_{l \in \range{1}{d_V}} \{ \|b_l\|_{L^2}, \|b_l^{(N)}\|_{L^2} \} \leq |D|^{3/2} J.\]
Applying Lemma~\ref{lemma:expectation_of_l2}, for fixed $u\in B$ we find that
\begin{subequations}
\begin{align}
\label{eq:L2}
\E \left \| \frac{b_l}{a} - \frac{b_l^{(N)}}{a^{(N)}} \right \|_{L^2} &\leq \E \frac{\|a^{(N)}\|_{L^2} \|b_l - b_l^{(N)}\|_{L^2} + \|b_l^{(N)}\|_{L^2} \|a - a^{(N)}\|_{L^2}  }{\|a^{(N)}\|_{L^2} \|a\|_{L^2}} \\
&\leq \frac{2  J^2}{I^2} N^{-1/2}.
\end{align}
\end{subequations}
Setting \(\alpha_l = b_l/a\) and \(\alpha_l^{(N)} = b_l^{(N)}/a^{(N)}\), from \eqref{eq:L2} we deduce that 
\begin{equation}
    \lim_{N \to \infty} \E \|\alpha_l(\cdot\,;u) - \alpha_l^{(N)}(\cdot\,;u)\|_{L^2} = 0,
\end{equation}
for any $l\in\range{1}{d_V}$. We now proceed to the second step of the proof, i.e. showing that \eqref{eq:step2} holds. Using similar reasoning as before, we notice that 
\begin{subequations}
\begin{align}
\inf_{u \in B}& \min \{ \|a\|_{C(\Bar{D})}, \|a^{(N)}\|_{C(\Bar{D})} \} \geq |D|I,\\
\sup_{u \in B} &\max_{l \in \range{1}{d_V}} \{ \|b_l\|_{C(\Bar{D})}, \|b_l^{(N)}\|_{C(\Bar{D})} \} \leq |D| J,
\end{align}
\end{subequations}
hence the sequence \( \{\alpha_l^{(N)}\}\) is uniformly bounded in $N$. Now, for fixed $u\in B$ the sequence \( \{\alpha_l^{(N)}(\cdot\,;u)\}\) is also uniformly equicontinuous with probability 1 over the choice \(\{y_j\}_{j=1}^N \sim U(\Bar{D})\). Indeed, we note that
    \begin{subequations}
        \begin{align}
            \left| \frac{b_l^{(N)}(r)}{a^{(N)}(r)}- \frac{b_l^{(N)}(t)}{a^{(N)}(t)}\right| &= \left| \frac{\sum_{j=1}^N g_l(r,y_j)}{\sum_{k=1}^N f(r,y_k)} - \frac{\sum_{j=1}^N g_l(t,y_j)}{\sum_{k=1}^N f(t,y_k)}\right|\\
            &\leq\frac{1}{N^2I}\left|\sum_{j,n=1}^Ng_l(r,y_j)f(t,y_n)-g_l(t,y_j)f(r,y_n) \right|\\
            &\leq \frac{J}{N^2I}\sum_{j,n=1}^N\bigl(\bigl|g_l(r,y_j)-g_l(t,y_j) \bigr|+\bigl|f(t,y_n)-f(r,y_n)\bigr| \bigr).
        \end{align}
    \end{subequations}
    The result then follows from the uniform continuity of $f$ and $g_l$ in their first arguments, due to the compactness of $\Bar{D}$.
Therefore, since $\Bar{D}\subset\R^d$ is compact, by the Arzel\`a–Ascoli theorem there exists a subsequence of indices $(N_k)_{k\in\N}$ such that $\alpha_l^{(N_k)}$ converges in $C(\Bar{D})$ to some $\alpha_l^{(\infty)}$. Since,
\begin{equation*}
   \E \|\alpha_l^{(\infty)} - \alpha_l^{(N_k)}\|_{L^2} \leq |D|^{1/2}\E \|\alpha_l^{(\infty)} - \alpha_l^{(N_k)}\|_{C(\Bar{D})},
\end{equation*}
by uniqueness of limits, it is readily observed that $\alpha_l^{(\infty)}=\alpha_l$. Therefore, since the limit is independent of the subsequence, it holds that
\[\lim_{N \to \infty} \E \|\alpha_l(\cdot\,;u) - \alpha_l^{(N)}(\cdot\,;u)\|_{C(\Bar{D})} = 0,\]
for any $l\in\range{1}{d_V}$. We now turn our attention to establishing \eqref{eq:aim}. By reasoning as before and by using the continuity of the exponential and the inner product in its first argument, it is straightforward to show that as mappings of the form \(\alpha_l:B \to C(\Bar{D};\R^{d_V})\) and \(\alpha_l^{(N)}:B \to C(\Bar{D};\R^{d_V})\) where $B$ is compact, \(\alpha_l\) and \(\alpha_l^{(N)}\) are continuous and hence the sequence \( \{\alpha_l^{(N)}\}\) is uniformly equicontinuous. Therefore, we can find moduli of continuity $\omega_1, \omega_2$ such that
\[\|\alpha_l(u) - \alpha_l(v)\|_{C(\Bar{D})} \leq \omega_1 \big ( \|u - v\|_{C(\Bar{D})} \big ), \quad \|\alpha_l^{(N)}(u) - \alpha_l^{(N)}(v)\|_{C(\Bar{D})} \leq \omega_2 \big ( \|u - v\|_{C(\Bar{D})} \big )\]
for any $u,v \in B$ and $N \in \mathbb{N}$. Fix $\epsilon > 0$. Since $B$ is compact, we can find a number $L = L (\epsilon) \in \mathbb{N}$ and functions $\phi_1, \dots, \phi_L \in B$ such that, for any $u \in B$, there exists $j \in \range{1}{L}$ such that
\[\|u - \phi_j\|_{C(\Bar{D})} < \epsilon.\]
Furthermore, we can find a number $S = S(\epsilon) \in \mathbb{N}$ such that, for any $j \in \range{1}{L}$, we have, for all $N \geq S$,
\[\E \|\alpha_l(\phi_j) - \alpha_l^{(N)} (\phi_j)\|_{C(\Bar{D})} < \epsilon.\]
It follows by triangle inequality that
\begin{align*}
    \E \|\alpha_l (u) - \alpha_l^{(N)}(u)\|_{C(\Bar{D})} &\leq \E \|\alpha_l (u) - \alpha_l(\phi_j)\|_{C(\Bar{D})} + \E \|\alpha_l(\phi_j) - \alpha_l^{(N)}(u)\|_{C(\Bar{D})} \\
    &\leq |D| \omega_1(\epsilon) + \E \|\alpha_l(\phi_j) - \alpha_l^{(N)}(\phi_j)\|_{C(\Bar{D})} + \E \|\alpha_l^{(N)}(\phi_j) - \alpha_l^{(N)}(u)\|_{C(\Bar{D})} \\
    &\leq |D| \omega_1(\epsilon) + \epsilon + |D| \omega_2(\epsilon).
\end{align*}
Therefore the result follows since the argument is uniform over all $l \in \range{1}{d_V}$.
\end{proof}


\subsection{Proof of Cross-Attention Approximation Theorem}
\label{Appendix:cross}

We make straightforward modifications to Lemmas \ref{lemma:expectation_of_l2}, \ref{lemma:boundedness2} and to the proof of Theorem \ref{thm:self_attention_limit} in order to establish Theorem \ref{thm:cross_attention_limit}.

\begin{lemma}
\label{lemma:expectation_of_l2_cross}
Let \(D \subset \R^d\) and \(E \subset \R^e\) be bounded open sets and \(f \in C(\Bar{D} \times \Bar{E})\).
Then
\[ \E \left \| \int_E f(\cdot,y) \: \mathrm{d}y - \frac{|E|}{N} \sum_{j=1}^N f(\cdot,y_j) \right \|_{L^2} \leq |D|^{1/2} |E|\|f\|_{C(\Bar{D} \times \Bar{E})} N^{-1/2}\]
with the expectation taken over i.i.d. sequences \(\{y_j\}_{j=1}^N \sim U(\Bar{E})\).
\end{lemma}
\begin{proof}
Notice that, for any \(x \in \Bar{D}\),
\[\int_E f(x,y) \: \mathrm{d}y  = |E| \E_{y \sim U(\Bar{E})} \big[ f(x,y) \big ].\]
Hence, by expanding the square, we can observe that
\[\E \left ( \int_E f(x,y) \: \mathrm{d}y - \frac{|E|}{N} \sum_{j=1}^N f(x,y_j) \right )^2 \leq \frac{|E|^2}{N} \E_{y \sim U(\Bar{E})} \big [ f(x,y)^2 \big ].\]
Bounding $f$ in $L^\infty$, and noting that it is a continuous function, gives
\begin{align*}
\E \left \| \int_E f(\cdot,y) \: \mathrm{d}y - \frac{|E|}{N} \sum_{j=1}^N f(\cdot,y_j) \right \|^2_{L^2} &\leq \frac{|E|^2}{N} \int_D \E_{y \sim U(\Bar{E})} \big [ f(x,y)^2 \big ] \: \mathrm{d}x \\
&\leq |D||E|^2 \|f\|^2_{C(\Bar{D} \times \Bar{E})} N^{-1} .
\end{align*}
The result follows by Jensen's inequality. 
\end{proof}
As part of the proof of Theorem \ref{thm:cross_attention_limit}, we first show that the the cross-attention operator is a mapping of the form $\sC:L^\infty (D;\R^{d_u})\times L^\infty (E;\R^{d_v}) \to L^\infty (D; \R^{d_V})$ and hence $\sA: C (\Bar{D};\R^{d_u})\times C (\Bar{E};\R^{d_v}) \to C (\Bar{D}; \R^{d_V})$.
\begin{lemma}
\label{lemma:boundedness2}
    Let $u\in L^\infty (D;\R^{d_u})\times L^\infty (E;\R^{d_v})$, then it holds that $\sA(u)\in L^\infty(D;\R^{d_V})$. Furthermore, for $u \in C (\Bar{D};\R^{d_u})\times C (\Bar{E};\R^{d_v}) $ it holds that $\sA(u) \in C (\Bar{D};\R^{d_V})$.
\end{lemma}
\begin{proof}
    The result follows readily from the fact that $q$ is a probability density function. For completeness, we show the boundedness of the normalization constant of the function $q$. Indeed we note that,
\begin{align*}
    \int_E \mathrm{exp} \big ( \langle Q u(x), K v(s) \rangle \big ) \: \mathrm{d}s &\leq \int_E \mathrm{exp} \big ( |Q| |K| |u(x)| |v(s)| \big ) \mathrm{d}s \\
    &\leq \int_E \mathrm{exp} \bigg ( \frac{1}{2} |Q|^2 |K|^2 |u(x)|^2 \bigg ) \mathrm{exp} \bigg ( \frac{1}{2} |v(s)|^2 \bigg ) \mathrm{d}s \\
    &\leq |E| \mathrm{exp} \big ( R_1 \|u\|_{L^\infty}^2 \big ) \mathrm{exp} \big ( R_2 \|v\|_{L^\infty}^2 \big )\\
    &\leq |E| \mathrm{exp} \Big ( R \bigl(\|u\|_{L^\infty}^2 + \|v\|_{L^\infty}^2 \bigr)\Big ),
\end{align*}
for some constant \(R \coloneqq \max\{R_1,R_2\}\). From definition, it is clear that 
\[\int_E q(y; u,v,x) \: \mathrm{d}y = 1,\] 
hence \(q\) is indeed a valid probability density function. It hence follows that
\[\|\sC(u,v)\|_{L^\infty(D;\R^{d_V})} \leq |V|\|v\|_{L^\infty(E;\R^{d_v})}\]
hence \(\sC : L^\infty (D;\R^{d_u})\times L^\infty (E;\R^{d_v}) \to L^\infty (D; \R^{d_V})\). Similarly, it follows that
\(\sC : C (\Bar{D};\R^{d_u})\times C (\Bar{E};\R^{d_v}) \to C (\Bar{D}; \R^{d_V})\).  We note that the continuity is preserved by $\sC$ under the continuity of the inner product in its first argument and the continuity of the exponential.
\end{proof}
We are now ready to establish the full result of Theorem \ref{thm:cross_attention_limit} which states that 
\[ \lim_{N \to \infty} \sup_{(u,v) \in B} \E \left \| \sC(u,v) - \frac{\sum_{j=1}^N \mathrm{exp} \big ( \langle Qu(\cdot), Kv(y_j)  \rangle \big ) V v(y_j)}{\sum_{\ell=1}^N \mathrm{exp} \big ( \langle Qu(\cdot), Kv(y_\ell)  \rangle \big )} \right \|_{C(\Bar{D};\R^{d_V})}  = 0,\]
    with the expectation taken over i.i.d. sequences \(\{y_j\}_{j=1}^N \sim \unif(\Bar{E})\).
\begin{proof}[Proof of Theorem \ref{thm:cross_attention_limit}]
Before commencing the proof, we first establish useful shorthand notation. Letting \((u,v) \in B\) we define,
\[f(x,y;u,v) \coloneqq \mathrm{exp} \big ( \langle Q u(x), K v(y) \rangle \big ), \quad g_l(x,y;u,v) \coloneqq \mathrm{exp} \big ( \langle Q u(x), K v(y) \rangle \big ) \big( Vv(y) \big )_l  \]
for all \(x\in \Bar{D}, y\in\Bar{E}\) and \(l \in \range{1}{d_V}\). For $(u,v) \in B$ we also define
\[a(x;u,v) \coloneqq \int_E f(x,y;u,v) \: \mathrm{d}y, \qquad a^{(N)}(x;u,v) \coloneqq \frac{|E|}{N} \sum_{j=1}^N f(x,y_j;u,v),\]
and
\[b_l(x;u,v) \coloneqq \int_E g_l(x,y;u,v) \: \mathrm{d}y, \qquad b^{(N)}_l(x;u,v) \coloneqq \frac{|E|}{N} \sum_{j=1}^N g_l(x,y_j;u,v)\]
for any $x\in\Bar{D}$ and any \(l \in \range{1}{d_V}\). We set \(\alpha_l = b_l/a\) and \(\alpha_l^{(N)} = b_l^{(N)}/a^{(N)}\). Given this notation, to prove the result we must show
\begin{equation}
\label{eq:aim2}
\lim_{N \to \infty} \sup_{(u,v) \in B} \E \left \| \alpha_l(\cdot\,;u,v) - \alpha_l^{(N)}(\cdot\,;u,v) \right\|_{C(\Bar{D};\R)}=0,
\end{equation}
for any $l\in\range{1}{d_V}$. We divide the proof in the following key steps. We first show that for fixed $(u,v)\in B$, it holds that
\begin{equation}
\label{eq:step1_2}
    \lim_{N \to \infty} \E \|\alpha_l(\cdot\,;u,v) - \alpha_l^{(N)}(\cdot\,;u,v)\|_{L^2} = 0,
\end{equation}
for any $l\in\range{1}{d_V}$. By using compactness and applying the Arzelà-Ascoli theorem, we then show that for $(u,v)\in B$
\begin{equation}
\label{eq:step2_2}
\lim_{N \to \infty} \E \|\alpha_l(\cdot\,;u,v) - \alpha_l^{(N)}(\cdot\,;u,v)\|_{C(\Bar{D})} = 0,
\end{equation}
for any $l\in\range{1}{d_V}$. We then use the compactness of $B$ and continuity of the mappings $(u,v)\mapsto \alpha_\ell(\cdot\, ;u,v)$ and $(u,v)\mapsto \alpha^{(N)}_\ell(\cdot\, ;u,v)$ to deduce the result given by \eqref{eq:aim2}.
We now focus on establishing \eqref{eq:step1_2}. Since \(B\) is bounded, there exists a constant \(M > 0\) such that,
\[\sup_{(u,v) \in B} \|u\|_{L^\infty}+\|v\|_{L^\infty} \leq M.\]
Therefore, we have
\[\sup_{(u,v) \in B} \max \{ \|f\|, \|g_1\|, \dots, \|g_m\| \} \leq \max \left \{ \text{exp} \big( RM^2 \big ), M |V| \text{exp} \big (RM^2 \big ) \right \} \coloneqq J,\]
where \(\|\cdot\| = \|\cdot\|_{C(\Bar{D} \times \Bar{E})}\).
Clearly, it holds that
\[\sup_{(u,v) \in B} \max \{ \|a\|_{L^2}, \|a^{(N)}\|_{L^2} \} \leq |D|^{1/2}|E| J.\]
Notice that, since \(|\langle Q u(x), K v(s) \rangle| \leq RM^2\), we have \(f(x,y) \geq \text{exp} \big ( - R M^2 \big )\).
It follows that
\[\inf_{(u,v) \in B} \min \{ \|a\|_{L^2}, \|a^{(N)}\|_{L^2} \} \geq |D|^{1/2}|E| \text{exp} \big ( - R M^2 \big ) \coloneqq |D|^{1/2}|E| I > 0.\]
Similarly, 
\[\sup_{(u,v) \in B} \max_{l \in [d_K]} \{ \|b_l\|_{L^2}, \|b_l^{(N)}\|_{L^2} \} \leq |D|^{1/2}|E| J.\]
Applying Lemma~\ref{lemma:expectation_of_l2_cross}, we find
\begin{subequations}
\begin{align}
\label{eq:L2_cross}
\E \left \| \frac{b_l}{a} - \frac{b_l^{(N)}}{a^{(N)}} \right \|_{L^2} &\leq \E \frac{\|a^{(N)}\|_{L^2} \|b_l - b_l^{(N)}\|_{L^2} + \|b_l^{(N)}\|_{L^2} \|a - a^{(N)}\|_{L^2}  }{\|a^{(N)}\|_{L^2} \|a\|_{L^2}} \\
&\leq \frac{2 J^2}{I^2} N^{-1/2}.
\end{align}
\end{subequations}
Setting \(\alpha_l = b_l/a\) and \(\alpha_l^{(N)} = b_l^{(N)}/a^{(N)}\), from \eqref{eq:L2_cross} we deduce that for $(u,v)\in B$ it holds that
\begin{equation}
    \lim_{N \to \infty} \E \|\alpha_l(\cdot\,;u,v) - \alpha_l^{(N)}(\cdot\,;u,v)\|_{L^2} = 0,
\end{equation}
for any $l\in\range{1}{d_V}$. We now proceed to the second step of the proof, i.e. showing that \eqref{eq:step2_2} holds. Using similar reasoning as before, we notice that 
\begin{subequations}
\begin{align}
\inf_{(u,v) \in B}& \min \{ \|a\|_{C(\Bar{D})}, \|a^{(N)}\|_{C(\Bar{D})} \} \geq |E|I,\\
\sup_{(u,v) \in B} &\max_{l \in [m]} \{ \|b_l\|_{C(\Bar{D})}, \|b_l^{(N)}\|_{C(\Bar{D})} \} \leq |E| J,
\end{align}
\end{subequations}
hence the sequence \( \{\alpha_l^{(N)}\}\) is uniformly bounded in $N$. Now, for fixed $(u,v)\in B$ the sequence \( \{\alpha_l^{(N)}(\cdot\,;u,v)\}\) is also uniformly equicontinuous with probability 1 over the choice \(\{y_j\}_{j=1}^N \sim U(\Bar{D})\). Indeed, we note that
    \begin{subequations}
        \begin{align}
            \left| \frac{b_l^{(N)}(r)}{a^{(N)}(r)}- \frac{b_l^{(N)}(t)}{a^{(N)}(t)}\right| &= \left| \frac{\sum_{j=1}^N g_l(r,y_j)}{\sum_{k=1}^N f(r,y_k)} - \frac{\sum_{j=1}^N g_l(t,y_j)}{\sum_{k=1}^N f(t,y_k)}\right|\\
            &\leq\frac{1}{N^2I}\left|\sum_{j,n}g_l(r,y_j)f(t,y_n)-g_l(t,y_j)f(r,y_n) \right|\\
            &\leq \frac{J}{N^2I}\sum_{j,n}\bigl(\bigl|g_l(r,y_j)-g_l(t,y_j) \bigr|+\bigl|f(t,y_n)-f(r,y_n)\bigr| \bigr).
        \end{align}
    \end{subequations}
    The result then follows from the uniform continuity of $f$ and $g_l$ in their first arguments, due to the compactness of $\Bar{D}$.
Therefore, since $\Bar{D}\subset\R^d$ is compact, by the Arzel\`a–Ascoli theorem there exists a subsequence of indices $(N_k)_{k\in\N}$ such that $\alpha_l^{(N_k)}$ converges in $C(\Bar{D})$ to some $\alpha_l^{(\infty)}$. Since,
\begin{equation*}
   \E \|\alpha_l^{(\infty)} - \alpha_l^{(N_k)}\|_{L^2} \leq |D|^{1/2}\E \|\alpha_l^{(\infty)} - \alpha_l^{(N_k)}\|_{C(\Bar{D})},
\end{equation*}
by uniqueness of limits, it is readily observed that $\alpha_l^{(\infty)}=\alpha_l$. Therefore, since the limit is independent of the subsequence, for fixed $(u,v)\in B$ it holds that
\[\lim_{N \to \infty} \E \|\alpha_l(\cdot\,;u,v) - \alpha_l^{(N)}(\cdot\,;u,v)\|_{C(\Bar{D})} = 0,\]
for any $l\in\range{1}{d_V}$. We now turn our attention to establishing \eqref{eq:aim2}. By reasoning as before and by using the continuity of the exponential and of the inner product on the product space, it is straightforward to show that as mappings of the form \(\alpha_l:B \to C(\Bar{D};\R^{d_V})\) and \(\alpha_l^{(N)}:B \to C(\Bar{D};\R^{d_V})\) where $B$ is compact, \(\alpha_l\) and \(\alpha_l^{(N)}\) are continuous and hence the sequence \( \{\alpha_l^{(N)}\}\) is uniformly equicontinuous. 
Therefore, we can find moduli of continuity $\omega_1, \omega_2$ such that
\[
\begin{aligned}
\|\alpha_l(u,v) - \alpha_l(\widetilde{u},\widetilde{v})\|_{C(\Bar{D})} &\leq \omega_1 \big ( \|(u,v) - (\widetilde{u},\widetilde{v})\|_{C(\Bar{D})\times C(\Bar{E})} \big ),\\
\|\alpha_l^{(N)}(u,v) - \alpha_l^{(N)}(\widetilde{u},\widetilde{v})\|_{C(\Bar{D})} &\leq \omega_2 \big ( \|(u,v) - (\widetilde{u},\widetilde{v})\|_{C(\Bar{D})\times C(\Bar{E})} \big )
\end{aligned}
\]
for any $(u,v), (\widetilde{u},\widetilde{v}) \in B$ and $N \in \mathbb{N}$. Fix $\epsilon > 0$. Since $B$ is compact, we can find a number $L = L (\epsilon) \in \mathbb{N}$ and functions $\phi_1, \dots, \phi_L \in B$ such that, for any $(u,v) \in B$, there exists $j \in \range{1}{L}$ such that
\[\|(u,v) - \phi_j\|_{C(\Bar{D})\times C(\Bar{E})} < \epsilon.\]
Furthermore, we can find a number $S = S(\epsilon) \in \mathbb{N}$ such that, for any $j \in \range{1}{L}$, we have, for all $N \geq S$,
\[\E \|\alpha_l(\phi_j) - \alpha_l^{(N)} (\phi_j)\|_{C(\Bar{D})} < \epsilon.\]
It follows by triangle inequality that
\begin{align*}
    \E \|\alpha_l (u,v) - \alpha_l^{(N)}(u,v)\|_{C(\Bar{D})} &\leq \E \|\alpha_l (u,v) - \alpha_l(\phi_j)\|_{C(\Bar{D})} + \E \|\alpha_l(\phi_j) - \alpha_l^{(N)}(u,v)\|_{C(\Bar{D})} \\
    &\leq |D| \omega_1(\epsilon) + \E \|\alpha_l(\phi_j) - \alpha_l^{(N)}(\phi_j)\|_{C(\Bar{D})} + \E \|\alpha_l^{(N)}(\phi_j) - \alpha_l^{(N)}(u,v)\|_{C(\Bar{D})} \\
    &\leq |E| \omega_1(\epsilon) + \epsilon + |E| \omega_2(\epsilon).
\end{align*}
Therefore the result follows since our argument is uniform over all $l \in \range{1}{d_V}$.
\end{proof}

\section{Transformer Neural Operators: The Discrete Setting}

We describe how the transformer neural operators described in Section \ref{sec:transformers_operator} are implemented in the discrete setting. Namely, in practice we will have access to evaluations of the input function $u\in \cU\bigl(D;\R^{d_u} \bigr)$ at a finite set of $N$ discretization points $\{x_i\}_{i=1}^N\subset D$. To highlight how the neural operator methodology is applied in practical settings it is useful, abusing notation, to view the input
to the neural operators as $u\in \R^{N\times d_u}$. In order to distinguish between functions in $\cU\bigl(D;\R^{d_u} \bigr)$, and discrete representations of these functions at a finite set of grid points, in this section we use different fonts for a function $\su\in\cU\bigl(D;\R^{d_u} \bigr)$ and its discrete analogue $u\in \R^{N\times d_u}$. We highlight that in practice, it is of interest to apply approximations of the attention operators $\sA$ and $\sAp$ defined on the continuum, to $u\in \R^{N\times d_u}$. Through an abuse of notation, we will write $\sA_{\placeholder}(u)$ to denote the finite-dimensional approximation of $\sA_{\placeholder}$ acting on matrices, and similarly for $\sE_L, \sF_{\textrm{NN}}$ and $\sF_{\textrm{LN}}.$


In Subsection \ref{subsubsec:vanilla_transformer_implement} we describe the transformer neural operator in the discrete setting, in Subsection \ref{subsubsec:ViT_implementation} the vision transformer neural operator and in Subsection \ref{subsubsec:patch_transformer_implement} the Fourier attention neural operator.

\subsection{Transformer Neural Operator}
\label{subsubsec:vanilla_transformer_implement}

In this section, we describe how the transformer neural operator from Subsection \ref{subsec:vanilla_transformer} is implemented in the discrete setting. The input to the model is therefore a tensor $u\in \R^{N \times d_u}$,
which
is then concatenated with the discretization points $\{x_i\}_{i=1}^N\subset D$, to form the tensor $u_{\textrm{in}}\in \R^{N \times (d_u+d)}$. The learnable linear transformation $W_{\textrm{in}}^\top\in \R^{(d_u+d) \times d_\textrm{model} }$ then acts on $u_{\textrm{in}}$ to yield
\[
v^{(0)}\coloneqq u_{\textrm{in}}W_{\textrm{in}}^\top,
\]
which
is then fed through the transformer encoder.
In practice we solve the recurrence relation in \eqref{eq:recurrence_enc_def} on a finite set of $N$ grid points that arise naturally from the discrete input $u_\textrm{in}$.
The multi-head attention operator $\sA_{\textrm{MultiHead}}$ is defined by the application of the linear transformation $W_{\textrm{MultiHead}}^\top\in \R^{Hd_K\times d_{\textrm{model}} }$ to the concatenation of $H\in\N$ self-attention operations. The concatenation of the outputs of the $H\in\N$ self-attention operations results in a $\R^{N\times Hd_K}$ tensor, so that
\begin{equation}
\label{eq:multihead_implement}
    \sA_{\textrm{MultiHead}}(v) \coloneqq  \bigl(\sA(v;Q_1,K_1,V_1),\ldots, \sA(v;Q_H,K_H,V_H)  \bigr) {W}_{\textrm{MultiHead}}^\top \in \R^{N\times d_{\textrm{model}}},
\end{equation}
for any $v\in \R^{N\times d_{\textrm{model}}}$. We recall that the self-attention operator in question, $\sA$, is defined in Definition \ref{d:4} for functions.
In the discrete setting, where $v\in\R^{N\times d_{\textrm{model}}}$, we compute the expectation and the integral arising from the normalization constant using the following trapezoidal approximations:
\begin{equation}
\label{eq:riemann_expect}
    \sA(v)_j :=   \frac12\sum_{k=2}^N \Bigl( V{v_k} \cdot p(k;v,j) + V{v_{k-1}} \cdot p(k-1;v,j) \Bigr)\cdot \Delta x_k,
    \end{equation}
for any $j=1,\ldots,N$, with
\begin{equation}
\label{eq:riemann_p}
    p(k;v,j) \approx \frac{\exp\Bigl(\bigl\langle Q{v_j}, Kv_k \bigr\rangle_{\R^{\dK}} \Bigr)}{\frac12\sum_{\ell=2}^N \Bigl(\exp\bigl(\bigl\langle Q{v_j}, Kv_\ell \bigr\rangle_{\R^{\dK}} \bigr)+\exp\bigl(\bigl\langle Q{v_j}, Kv_{\ell-1} \bigr\rangle_{\R^{\dK}} \bigr)\Bigr)\Delta x_\ell},
\end{equation}
where by $u_\ell$ we have denoted the $\ell$'th row of the matrix $v\in \R^{N\times d_{\textrm{model}}}$, and similarly for $\sA(v)_\ell$.
%
In equations \eqref{eq:riemann_expect} and \eqref{eq:riemann_p} 
$\Delta x_k$
denotes 
\[
\Delta x_k \coloneqq \prod_{i=1}^d |(x_k)_i-(x_{k-1})_i|,
\]
for 
$k=2,\ldots,N$, where $(x_k)_i$ is the $i$'th component of the vector $x_k$.
This trapezoidal approximation allows for a mathematically consistent implementation of attention for irregularly sampled data, as demonstrated in Section \ref{sec:numerics}. 

The linear transformations $W_1^\top,W_2^\top\in\R^{d_{\textrm{model}}\times d_{\textrm{model}}}$ are applied pointwise. It is also straightforward to apply the definitions for $\sF_\textrm{LayerNorm}$ in \eqref{eq:LN} and $\sF_\textrm{NN}$ in \eqref{eq:NN} to the finite representation of $v$ as both involve pointwise linear transformations.
We observe that the recurrence relation \eqref{eq:recurrence_enc_def} is compatible with varying domain discretization lengths $N$, and will produce an output $v^{(L)}$ of dimension matching the input $v^{(0)}$. The resulting tensor $v^{(L)}\in \R^{N \times  d_{\textrm{model}}}$ is then projected to the output dimension by applying the learnable linear transformation $W_{\textrm{out}}^\top\in \R^{d_{\textrm{model}} \times d_z}$, so that the output of the architecture is defined as 
\[
z\coloneqq v^{(L)}W_{\textrm{out}}^\top \in \R^{d_{\textrm{model}} \times d_z}.
\]

\subsection{Vision Transformer Neural Operator}
\label{subsubsec:ViT_implementation}

In this section, we describe how the transformer neural operator from Subsection \ref{subsec:patch_transformer} is implemented in the discrete setting. In practice, we will have access to a discretization of the domain $D$ containing $N$ points and evaluations of the function $\su\in \cU(D;\R^{d_u})$ at these points; the resolution of the input will be determined by $n_1\times\ldots\times n_d=N$. The input to the model will therefore be a tensor $u\in \R^{n_1\times\ldots\times n_d\times d_u}$. The input $u\in \R^{n_1\times\ldots\times n_d\times d_u}$ is first concatenated with the discretization points $\{x_i\}_{i=1}^N\subset D$, to form the tensor $u_{\textrm{in}}\in \R^{n_1\times\ldots\times n_d\times (d_u+d)}$. Denoting by $p_1\times\ldots\times p_d$ the resolution of each patch, application of the reshaping operator $\sS_{\textrm{reshape}}$ results in
$$\sS_{\textrm{reshape}}\bigl(u_{\textrm{in}} \bigr)\in \R^{P\times p_1 \times\ldots\times p_d \times (d_u+d)}.$$
Following Definition \ref{def:FourierIntegralOperator_final} and the subsequent discussion, the nonlocal linear operator $\sW_{\textrm{in}}$ is implemented as
\begin{equation}
\label{eq:int_discrete}
    \sW_{\textrm{in}}(\theta){{u}} \coloneqq \cF^{-1}\Bigl(R_{\sW_{\textrm{in}}}(\theta) \cdot (\cF{u}) \Bigr),
\end{equation}
for any $u\in \R^{P\times p_1 \times\ldots\times p_d \times (d_u+d)}$, where we have made explicit the dependence on $\theta \in\Theta$ to highlight the learnable components in the implementation. The FFT is computed on the $p_1\times\ldots\times p_d$ dimensions so that the resulting tensor is of the form $\cF{u}\in\C^{P\times p_1\times\ldots\times p_d\times (d_u+d)}$. Since ${u}$ is convolved with a function that possess only  $\bigl|Z_{k_{\textrm{max}}} \bigr|$ Fourier modes, the higher Fourier modes may be truncated so that one may consider $\cF{u}\in\C^{P\times (2k_{\textrm{max},1}+1)\times\ldots\times (2k_{\textrm{max},d}+1)\times (d_u+d)}$. The learnable tensor is $R_{\sW_{\textrm{in}}}\in\C^{(2k_{\textrm{max},1}+1)\times\ldots\times (2k_{\textrm{max},d}+1)\times d_{\textrm{model}}\times (d_u+d)}$ and the operation in Fourier space is defined as
\begin{equation}
\label{eq:frequency_comp}    \Bigl(R\cdot\bigl(\cF{u}\bigr) \Bigr)_{p,i_1,\ldots,i_d, \ell} = \sum_{k=1}^{d_u+d} R_{p,i_1,\ldots,i_d, \ell,k} \bigl(\cF{u} \bigr)_{p,i_1,\ldots,i_d, k}       ,
\end{equation}
for $p\in\range{1}{P}, \,1\leq i_j\leq 2k_{\textrm{max},j}+1,\, 1\leq \ell\leq d_{\textrm{model}} $. 
The inverse FFT is then applied along the dimensions $2k_{\textrm{max},1}+1,\ldots, 2k_{\textrm{max},d}+1$. The resulting tensor defined as ${v}^{(0)}\in \R^{{P\times p_1\times \ldots\times p_d\times d_{\textrm{model}}}}$ is then fed through the transformer encoder $\sE_L:\R^{{P\times p_1\times \ldots\times p_d\times d_{\textrm{model}}}}\to\R^{{P\times p_1\times \ldots\times p_d\times d_{\textrm{model}}}}$.

The multi-head attention is based on the self-attention operator acting on functions from Definition \ref{def:patch_attention}. While the expectation is computed with respect to a probability measure on the patch sequence, hence inducing a discrete probability mass function, the discretization of the domain introduces the need to approximate the $L^2$  inner product appearing in Definition \ref{def:patch_attention_p}. We resort to an approximation of the integral given by 
\begin{subequations}
\label{eq:MC_approx_L2_1}
\begin{align}
    \big\langle \sQ_h {\sv}(j), \sK_h {\sv}(k) \big\rangle_{L^2(D' ,\mathbb{R}^{\dK})} &= \int_{D'} \bigl(\sQ_h {\sv}(j)\bigr)(x) \cdot \bigl(\sK {\sv}(k)\bigr)(x) \,\mathrm{d}x\\
    &\approx \sum_{x_i\in D'} \Delta x_i\cdot\bigl(\sQ_h {\sv}(j)\bigr)(x_i) \cdot \bigl(\sK_h {\sv}(k)\bigr)(x_i)
\end{align}
\end{subequations}
where we have denoted 
by $\Delta x_i$ the volume of the $i$th cell. We recall that in this case the linear operators $\sQ_h,\sK_h,\sV_h$ are pointwise local operators hence they admit finite dimensional representations $Q_h\in \R^{d_K\times d_{\textrm{model}}},K_h\in \R^{d_K\times d_{\textrm{model}}},V_h\in \R^{d_K\times d_{\textrm{model}}}$, for $h=1,\ldots,H$. In the case of uniform patching and uniform discretization of the domain of size $s^d=N$, which is the setting we experiment with, the approximation reduces to the approximation of the integral of the form
\begin{equation}
\label{eq:MC_approx_L2_2}
    \big\langle \sQ_h {\sv}(j), \sK_h {\sv}(k) \big\rangle_{L^2(D' ,\mathbb{R}^{\dK})} \approx \frac{1}{N}  \sum_{i_1,\ldots,i_d=1}^s \sum_{\ell=1}^{d_K} \bigl( {v}Q_h^\top\bigr)_{j\times i_1\times\ldots\times i_d\times \ell}\bigl( {v}K_h^\top\bigr)_{k\times i_1\times\ldots\times i_d\times \ell}.
\end{equation}
The multi-head attention operator is thus defined by the application of the linear transformation $W_{\textrm{MultiHead}}^\top\in \R^{Hd_K\times d_{\textrm{model}} }$ to the concatenation of $H\in\N$ self-attention operations defined by the approximations above. The concatenation of the outputs of the $H\in\N$ self-attention operations results in a $\R^{P\times p_1\times \ldots\times p_d\times Hd_K}$ tensor, so that 
\begin{equation}
\label{eq:multihead_implement_operator}
    \sA_{\textrm{MultiHead}}(v) \coloneqq  \bigl(\sAp(v;{Q}_1,{K}_1,{V}_1),\ldots, \sAp(v;{Q}_H,{K}_H,{V}_H)  \bigr) W_{\textrm{MultiHead}}^\top \in \R^{P\times p_1\times \ldots\times p_d\times d_{\textrm{model}}},
\end{equation}
for any $v\in \R^{P\times p_1\times \ldots\times p_d\times d_{\textrm{model}}}$. The linear transformations $W_1^\top,W_2^\top\in\R^{d_{\textrm{model}}\times d_{\textrm{model}}}$ are applied pointwise. The layer normalization operator is applied as defined in \eqref{eq:LN_patch} to every point in every patch so that $\sF_{\textrm{LayerNorm}}(v)\in \R^{P\times p_1 \times \ldots \times p_d \times d_{\textrm{model}}} $. 

The operator $\sF_{\textrm{NN}}$ is applied so that $\sF_{\textrm{NN}}(v)\in \R^{P\times p_1 \times \ldots \times p_d \times d_{\textrm{model}}}$, for any $v\in \R^{P\times p_1 \times \ldots \times p_d \times d_{\textrm{model}}}$. Indeed, in this finite-dimensional setting $\sF_\textrm{NN}$ is defined as
\begin{equation}
\label{eq:NN_patch}
    \sF_\textrm{NN}(v;W_3,W_4,b_1,b_2) = f\bigl(vW_4^\top+b_1\bigr)W_3^\top + b_2,
\end{equation}
for any $v\in \R^{P\times p_1 \times \ldots \times p_d \times d_{\textrm{model}}}$, where $f$ is a nonlinear activation function $W_3^\top,W_4^\top\in\R^{d_\textrm{model}\times d_\textrm{model}}$ are learnable local linear operators and $b_1,b_2 \in \R^{P\times p_1 \times \ldots \times p_d \times d_{\textrm{model}}}$ are learnable in the last dimension.

After computing the recurrence for $\sE_L$ as described in \Cref{eq:recurrence_enc_def}, a reshaping operator $\sS'_{\textrm{reshape}}$ is applied to the resulting tensor so that 
\begin{equation}
\label{eq:s'r_impl}
    v \mapsto \sS'_{\textrm{reshape}} (v) \in \R^{n_1 \times \ldots \times n_d \times d_{\textrm{model}} },
\end{equation}
for any $v\in \R^{{P\times p_1\times \ldots\times p_d\times d_{\textrm{model}}}}$. A pointwise linear transformation $W_{\textrm{out}}^\top\in \R^{ d_{\textrm{model}}\times d_z}$ is then applied according to 
\begin{equation}
\label{eq:w_out_impl}
v \mapsto vW_{\textrm{out}}^\top,
\end{equation}
for any $v\in\R^{n_1 \times \ldots \times n_d \times d_{\textrm{model}} }$.


\subsection{Fourier Attention Neural Operator}
\label{subsubsec:patch_transformer_implement}
In this section, we describe how the transformer neural operator from Subsection \ref{subsec:fourier_patch_transformer} is implemented in the discrete setting. In practice, we consider a discretization of the domain $D$ containing $N$ points and evaluations of the function $u\in\cU\bigl(D;\R^{d_u} \bigr)$ at these points; the resolution of the input is defined by $n_1\times\ldots\times n_d=N$. The input to the model is hence a tensor $u\in \R^{n_1\times\ldots\times n_d\times d_u}$. The input is concatenated with the discretization points $\{x_i\}_{i=1}^N\subset D$, to form the tensor $u_{\textrm{in}}\in \R^{n_1\times\ldots\times n_d\times (d_u+d)}$. Denoting by $p_1\times\ldots\times p_d$ the resolution of each patch, application of the reshaping operator $\sS_{\textrm{reshape}}$ results in $$\sS_{\textrm{reshape}}\bigl(u_{\textrm{in}} \bigr)\in \R^{P\times p_1 \times\ldots\times p_d \times (d_u+d)}.$$
The learnable linear transformation $W_{\textrm{in}}^\top\in \R^{(d_u+d)\times d_\textrm{model}}$ is then applied as a map
$$\sS_{\textrm{reshape}}\bigl(u_{\textrm{in}} \bigr)\mapsto \sS_{\textrm{reshape}}\bigl(u_{\textrm{in}} \bigr)W_{\textrm{in}}^\top \in \R^{P\times p_1 \times\ldots\times p_d \times d_{\textrm{model}}}.$$
The resulting tensor $\widetilde{u}^{(0)}$ is then fed through the transformer encoder $$\sE_L:\R^{{P\times p_1\times \ldots\times p_d\times d_{\textrm{model}}}}\to\R^{{P\times p_1\times \ldots\times p_d\times d_{\textrm{model}}}}.$$ 
In the recurrence relation that defines the encoder, $\sA_{\textrm{MultiHead}}$ is the multi-head attention operator as defined in equation \eqref{eq:multihead1}. In this setting, letting $H\in \N$ denote the number of attention heads, the multi-head attention operator is parametrized by learnable linear integral operators $\sQ_h:\R^{p_1\times\ldots\times p_d\times d_\textrm{model}}\to \R^{p_1\times\ldots\times p_d\times d_K}$, $\sK_h: \R^{p_1\times\ldots\times p_d\times d_\textrm{model}}\to \R^{ p_1\times\ldots\times p_d\times d_K}$, and $\sV_h: \R^{p_1\times\ldots\times p_d\times d_\textrm{model}}\to \R^{ p_1\times\ldots\times p_d\times d_K}$ for $h=1,\ldots, H$. Following Definition \ref{def:FourierIntegralOperator_final} and the subsequent discussion we implement the linear integral operators as 
\begin{subequations}
\label{eq:QKV_operator_implement}
    \begin{align}
        \sQ_h(\theta){{v}} &= \cF^{-1}\Bigl(R_{\sQ_h}(\theta) \cdot (\cF{v}) \Bigr),\\
        \sK_h(\theta){v} &= \cF^{-1}\Bigl(R_{\sK_h}(\theta) \cdot (\cF{v}) \Bigr),\\
        \sV_h(\theta){v}  &= \cF^{-1}\Bigl(R_{\sV_h}(\theta) \cdot (\cF{v}) \Bigr),
    \end{align}
\end{subequations}
for $h=1,\ldots,H$ and any $v\in\R^{p_1\times\ldots\times p_d\times d_\textrm{model}}$, where we have made explicit the dependence on $\theta \in\Theta$ to highlight the learnable components in the implementation. In \eqref{eq:QKV_operator_implement} the FFT is computed on the $p_1\times\ldots\times p_d$ so that the resulting tensor $\cF{v}\in\C^{ p_1\times\ldots\times p_d\times d_{\textrm{model}}}$. Since ${v}$ is a function that possess only $\bigl|Z_{k_{\textrm{max}}} \bigr|$ Fourier modes, the higher Fourier modes may be truncated so that one may consider $\cF{v}\in\C^{ (2k_{\textrm{max},1}+1)\times\ldots\times (2k_{\textrm{max},d}+1)\times d_{\textrm{model}}}$. The learnable tensors are of dimension $R_{\sQ_h},R_{\sK_h},R_{\sV_h}\in\C^{(2k_{\textrm{max},1}+1)\times\ldots\times (2k_{\textrm{max},d}+1)\times d_{K}\times d_{\textrm{model}}}$ for $h=1,\ldots,H$, and the operation in Fourier space is defined as
\begin{equation}
    \Bigl(R\cdot\bigl(\cF{v}\bigr) \Bigr)_{i_1,\ldots,i_d, \ell} = \sum_{k=1}^{d_{\textrm{model}}} R_{i_1,\ldots,i_d, \ell,k} \bigl(\cF{v} \bigr)_{i_1,\ldots,i_d, k}       ,
\end{equation}
for any $v\in \R^{p_1\times\ldots\times p_d\times d_\textrm{model}}$ and $\,1\leq i_j\leq 2k_{\textrm{max},j}+1,\, 1\leq \ell\leq d_K $. The inverse FFT is computed along the dimensions $2k_{\textrm{max},1}+1,\ldots, 2k_{\textrm{max},d}+1$ so that $\sQ_h{v}, \sK_h{v}, \sV_h{v} \in \R^{ p_1\times \ldots\times p_d\times d_K}$ for $h=1,\ldots,H$. We recall that in this case the self-attention operator $\sA$ is defined according to Definition \ref{def:patch_attention}. While the expectation is computed with respect to a probability measure on the patch sequence, hence inducing a discrete probability mass function, the discretization of the domain introduces the need to approximate the $L^2$  inner product appearing in Definition \ref{def:patch_attention_p}. We resort to an approximation of the integral given by \eqref{eq:MC_approx_L2_1} and \eqref{eq:MC_approx_L2_2}. The multi-head attention operator is thus defined by the application of the linear transformation $W_{\textrm{MultiHead}}^\top\in \R^{Hd_K\times d_{\textrm{model}} }$ to the concatenation of $H\in\N$ self-attention operations. The concatenation of the outputs of the $H\in\N$ self-attention operations results in a $\R^{P\times p_1\times \ldots\times p_d\times Hd_K}$ tensor, so that 
\begin{equation}
\label{eq:multihead_implement_operator1}
    \sA_{\textrm{MultiHead}}(v) \coloneqq  \bigl(\sAp(v;\sQ_1,\sK_1,\sV_1),\ldots, \sAp(v;\sQ_H,\sK_H,\sV_H)  \bigr) W_{\textrm{MultiHead}}^\top \in \R^{P\times p_1\times \ldots\times p_d\times d_{\textrm{model}}},
\end{equation}
for any $v\in \R^{P\times p_1\times \ldots\times p_d\times d_{\textrm{model}}}$. The linear transformations $W_1^\top,W_2^\top\in\R^{d_{\textrm{model}}\times d_{\textrm{model}}}$ are applied pointwise. The layer normalization operator is applied as defined in \eqref{eq:LN} to every point in every patch so that $\sF_{\textrm{LayerNorm}}(v)\in \R^{P\times p_1 \times \ldots \times p_d \times d_{\textrm{model}}} $. Finally, the operator $\sF_{\textrm{NN}}$ is parametrized and applied as in \eqref{eq:NN_patch} so that $\sF_{\textrm{NN}}(v)\in \R^{P\times p_1 \times \ldots \times p_d \times d_{\textrm{model}}}$, for any $v\in \R^{P\times p_1 \times \ldots \times p_d \times d_{\textrm{model}}}$.

A reshaping operator $\sS'_{\textrm{reshape}}$ to the tensor resulting from the transformer encoder iteration as in \Cref{eq:s'r_impl}. A pointwise linear transformation $W_{\textrm{out}}^\top\in \R^{ d_{\textrm{model}}\times d_z}$ is then applied as in \Cref{eq:w_out_impl}.

\section{Complexity of the Transformer Neural Operators}
\label{Appendix:complexity}

We summarize the expressions for the parameter complexity and evaluation cost (measured in floating point operations, FLOPS) of the neural operators from Section \ref{sec:transformers_operator} and the Fourier neural operator (FNO) from \citet{li2021fourier}. These are provided in Tables \ref{tab:parameter_complexity1} and \ref{tab:evaluation_cost1}, respectively. In Subsection \ref{subsec:param_complex} and \ref{subsec:cost_complex} we provide further details for the computations of the parameter and evaluation cost complexities, respectively.

\subsection{Parameter Complexity}
\label{subsec:param_complex}

\begin{table}[htbp]
    \centering
    \begin{tabular}{|l|c|}
        \hline
        \cellcolor{gray!30}\textbf{Architecture} & \cellcolor{gray!30}\textbf{Parameter Complexity}  \\
        \hhline{|==|}
        \textbf{FNO} & $(d_u+d)\cdot d_{\textrm{FNO}}+2d_{\textrm{FNO}}\cdot  d_{\textrm{model}}+d_{\textrm{FNO}}\cdot d_z +4\cdot\bigl(d^2_{\textrm{model}}\cdot k_{\textrm{max}}+d_{\textrm{model}}^2 \bigr)$\\
        \hline
        \textbf{AFNO} & $(d_u+d)\cdot d_{\textrm{model}}+Nd_{\textrm{model}}+d_{\textrm{model}}d_z+4\cdot(8+4/b)\cdot d^2_{\textrm{model}}$\\
        \hline
        \textbf{Transformer NO}&  $(d_u+d)\cdot d_{\textrm{model}}+d_{\textrm{model}}\cdot d_z +6\cdot\bigl(6d^2_{\textrm{model}}+2d_{\textrm{model}} \bigr)$\\
        \hline
        \textbf{ViT NO}& $(d_u+d)\cdot d_{\textrm{model}}\cdot k_{\textrm{max}} + d_{\textrm{model}}\cdot d_z + 6\cdot\bigl( 6d^2_{\textrm{model}}+2d_{\textrm{model}}\bigr)$ \\
        \hline
        \textbf{FANO}& $(d_u+d)\cdot d_{\textrm{model}}+d_{\textrm{model}}\cdot d_z + 6\cdot \bigl(  3d^2_{\textrm{model}}\cdot k_{\textrm{max}}+ 3d^2_{\textrm{model}}+2d_{\textrm{model}}\bigr)$ \\
        \hline
    \end{tabular}
    \caption{Parameter complexity for the three proposed transformer neural operators, as implemented in practice, compared to the Fourier neural operator \citep{li2021fourier}. We note that for our implementations, $d_{\textrm{FNO}}$, which is the latent dimension of the two-layer lifting and projection MLP in FNO, is fixed to be $128$, while the hyperparameter $b$ in the AFNO is fixed to be $8$. Parameters arising from possible bias terms, parameters arising from skip connections in FNO and parameters arising from normalizations in AFNO are omitted for ease of presentation.}
    \label{tab:parameter_complexity1}
\end{table}

\subsection{Evaluation Cost}
\label{subsec:cost_complex}

We note that a matrix vector product $Wu$ for $W\in \R^{r\times r'}$ has cost $2rr'$. We approximate the cost of computing the FFT of a vector in $\R^r$ as $5r\log r$ \citep{hoop2022cost}. We approximate the cost of the two-dimensional real FFT, used in implementation, by $(15/2)\cdot N \log (\sqrt{N})$ where $N$ is the total number of grid points. Assuming that the grid is of size $s\times s=N$ for $s$ an even integer, 1d FFTs are first computed for each row. The real FFT requires $s/2$ column 1d FFTs, hence the $(15/2) \cdot  N \log \sqrt{N}$ complexity.

\begin{table}[htbp]
    \centering
    \begin{tabular}{|l|c|c|}
        \hline
        \cellcolor{gray!30}\textbf{Architecture} & \cellcolor{gray!30}\textbf{Evaluation Cost}  \\
        \hhline{|===|}
        \textbf{FNO} & $\begin{aligned}
            &2N(d_u+d+2d_{\textrm{model}})d_{\textrm{FNO}} +2Nd_{\textrm{FNO}}d_z \\
            &+4\bigl(15d_{\textrm{model}}N\log(\sqrt{N})+k_{\textrm{max}}\cdot(2d^2_{\textrm{model}}-d_{\textrm{model}})\\
            &+2Nd^2_{\textrm{model}}\bigr)
            \end{aligned}$\\
        \hline
        \textbf{Transformer NO}&  $\begin{aligned}
            &2N(d_u+d)d_{\textrm{model}} +2Nd_{\textrm{model}}d_z +6\bigl(8d^2_{\textrm{model}}N \\&+2N^2d_{\textrm{model}}
            +4d^2_{\textrm{model}}N\bigl)
            \end{aligned}$ \\
        \hline
        \textbf{ViT NO}& $ \begin{aligned}
            &(d_u+d+d_{\textrm{model}})\cdot 15N\log (\sqrt{N/P})/2\\
            &+Pk_{\textrm{max}}d_{\textrm{model}}\bigl(2(d_u+d)-1 \bigr)+2Nd_{\textrm{model}}d_z\\
            &+6\Bigl( 8d^2_{\textrm{model}}N+2d_{\textrm{model}}NP+4d^2_{\textrm{model}}N\Bigr)\\
        \end{aligned}$
        \\
        \hline
        \textbf{FANO}& $ \begin{aligned}
            &2N(d_u+d)d_{\textrm{model}} +2Nd_{\textrm{model}}d_z\\
            &+6\Bigl( 45d_{\textrm{model}}N\log(\sqrt{N/P})+3k_{\textrm{max}}P\bigl(2d^2_{\textrm{model}}-d_{\textrm{model}} \bigr)\\
            &+2Nd^2_{\textrm{model}}+2d_{\textrm{model}}NP +4Nd^2_{\textrm{model}}\Bigr)\\
        \end{aligned}$ \\
        \hline
        \textbf{AFNO}& $ \begin{aligned}
            &2N(d_u+d)d_{\textrm{model}} +2Nd_{\textrm{model}}d_z\\
            &+4\bigl(16d_{\textrm{model}}^2 + 15Nd_{\textrm{model}}\log(\sqrt{N}) + 6N(2d_{\textrm{model}}^2/b - d_{\textrm{model}}) \bigr)\\
        \end{aligned}$ \\
        \hline
    \end{tabular}
    \caption{Evaluation cost (in FLOPS) for the proposed transformer neural operators, as implemented in practice, compared to FNO \citep{li2021fourier}. As done in \citet{Kaplan2020ScalingLF} and for ease of presentation, the additional evaluation cost arising from nonlinear activations, the softmax operator and normalizations are omitted. Furthermore, we omit the cost arising from skip connections. Indeed these do not affect the scaling. We further note that for our implementations, $d_{\textrm{FNO}}$, which is the latent dimension of the two-layer lifting and projection MLP in FNO, is fixed to be $128$. In this table we present the expression for the evaluation cost of AFNO without consideration of possible patching, which we do not employ for our comparisons: this in combination with the fact that we do not employ mode truncation is the reason for the $Nd_{\textrm{model}}^2$ term. In our context, the parameter $b$ in AFNO is set to $b=8$.}
    \label{tab:evaluation_cost1}
\end{table}

\bibliography{references}

\end{document}